\documentclass[11pt]{article}
\usepackage[margin=1in]{geometry}
\usepackage{amssymb,amsfonts,amsmath,amsthm,amscd,dsfont,mathrsfs,bbold,pifont}
\usepackage{blkarray}
\usepackage{graphicx,float,psfrag,epsfig,color}

\usepackage{subcaption}

\usepackage{algorithm}
\usepackage[noend]{algpseudocode}

\makeatletter
\def\BState{\State\hskip-\ALG@thistlm}
\makeatother

	\usepackage{microtype}
	\usepackage[pdftex,pagebackref=true,colorlinks]{hyperref}
	\hypersetup{linkcolor=[rgb]{.7,0,0}}
	\hypersetup{citecolor=[rgb]{0,.7,0}}
	\hypersetup{urlcolor=[rgb]{.7,0,.7}}

\newcommand\independent{\protect\mathpalette{\protect\independent}{\perp}} 
\def\independent#1#2{\mathrel{\rlap{$#1#2$}\mkern2mu{#1#2}}}

\newcommand{\F}{\mathcal{F}} 
\newcommand{\mR}{\mathbb{R}}

\newcommand{\pp}{\mathbb{P}}
\newcommand{\E}{\mathbb{E}}
\DeclareMathOperator{\Var}{Var}

\newcommand{\X}{\mathcal{X}}
\newcommand{\Y}{\mathcal{Y}}
\newcommand{\Z}{\mathcal{Z}}

\newcommand{\D}{\mathcal{D}}

\newcommand{\1}{\mathbb{1}}
\newcommand{\B}{\mathbb{B}}

\newcommand{\tX}{\tilde{X}}
\newcommand{\tY}{\tilde{Y}}
\newcommand{\tW}{\tilde{W}}

\newtheorem{theorem}{Theorem}

\newtheorem{lemma}{Lemma}
\newtheorem{corollary}{Corollary}
\newtheorem{definition}{Definition}
\newtheorem{remark}{Remark}
\newtheorem{example}{Example}

\begin{document}



\title{Poly-time universality and limitations of deep learning}

\author{Emmanuel Abbe \\ EPFL
\and  Colin Sandon \\ MIT 
}

\date{}
\maketitle

\begin{abstract}
The goal of this paper is to characterize function distributions that deep learning can or cannot learn in poly-time. A universality result is proved for SGD-based deep learning and a non-universality result is proved for GD-based deep learning; this also gives a separation between SGD-based deep learning and statistical query algorithms:

(1) {\it Deep learning with SGD is efficiently universal.} Any function distribution that can be learned from samples in poly-time can also be learned by a poly-size neural net trained with SGD on a poly-time initialization with poly-steps, poly-rate and possibly poly-noise. 

 Therefore deep learning  provides a universal learning paradigm: it was known that the approximation and estimation errors could be controlled with poly-size neural nets, using ERM that is NP-hard; this new result shows that the optimization error can also be controlled with SGD in poly-time. 
The picture changes for GD with large enough batches:

(2) {\it Result (1) does not hold for GD:} Neural nets of poly-size trained with GD (full gradients or large enough batches) on any initialization with poly-steps, poly-range and at least poly-noise cannot learn any function distribution that has super-polynomial {\it cross-predictability,} where the cross-predictability gives a measure of ``average'' function correlation -- relations and distinctions to the statistical dimension are discussed. In particular, GD with these constraints can learn efficiently monomials of degree $k$ if and only if $k$ is constant. 

Thus (1) and (2) point to an interesting contrast: SGD is universal even with some poly-noise while full GD or SQ algorithms are not (e.g., parities). 
 This thus gives a separation between SGD-based deep learning and SQ algorithms. 
Finally, we complete these by showing that the cross-predictability also impedes SGD once larger amounts of noise are added on the initialization and gradients, or when sufficiently few weight are updating per time step (as in coordinate descent). 

\end{abstract}

\newpage
\tableofcontents

\newpage


\newpage
\section{Introduction}

\subsection{Context and this paper}

It is known that the class of neural networks (NNs) with polynomial network size can express any function that can be implemented in a given polynomial time \cite{parberry,sipser}, and that their sample complexity scales polynomially with the network size \cite{anthony}. Thus NNs have favorable approximation and estimation errors. The main challenge is with the optimization error, as there is no known efficient training algorithm for NNs with provable guarantees, in particular, it is NP-hard to implement the ERM rule  \cite{net_hard,daniely}. The success behind deep learning is to train {\it deep} NNs with  stochastic gradient descent or the like; this gives record performances\footnote{While deep learning operates in an overparametrized regime, and while SGD optimizes a highly non-convex objective function, the training by SGD gives astonishingly low generalization errors for  these types of signals.} 
 in image \cite{imagenet}, speech \cite{speech}, document recognitions \cite{document} and increasingly more applications \cite{deep_nature,deep_book}. This raises the question of whether SGD  complements neural networks to a universal learning paradigm \cite{shaishai}, i.e., capable of learning efficiently any efficiently learnable function distribution.

(i) This paper answers this question in the affirmative. It is shown that training poly-size neural nets with SGD in poly-steps allows one to learn any function distribution that is learnable by some algorithm running in poly-time with poly-many samples. This part is resolved using a specific non-random net initialization that is implemented in poly-time and not dependent on the function to be learned, and that allows to emulate any efficient learning algorithm under SGD training.  

(ii) We further show that this positive result is achieved with some robustness to noise: polynomial noise can be added to the gradients and weights can be of polynomial precision and the result still holds. Therefore, in a computational theoretic sense, deep learning gives a universal learning framework. 

(iii) This positive result is also put in contrast with the following one: the same universality result does not hold when using full gradient descent or large enough batches\footnote{Some of the negative results presented here appeared in a preliminary version of the paper \cite{our_first_arxiv}; a few changes are obtained in the current version, with in particular the dependency in the batch size for the negative  result on GD. This allows to show that as the GD queries become more random (smaller batches), the negative result breaks down.}, due to the existence of efficiently learnable function distribution having low cross-predictability (see definitions below). This also creates a separation between deep learning and statistical query (SQ) algorithms, which cannot afford such noise-robustness on function classes having high statistical (see more below). 

In a practical setting, there may be no obvious reason to use the SGD replacement to a general learning algorithm, but this universality result shows that negative results about deep learning cannot be obtained without further constraints.

To obtain negative results about GD, we show that GD cannot learn in poly-steps and with poly-noise certain function distributions that have a low cross-predictability (a measure of average function correlation defined in Section \ref{init}). This is similar to the type of negative results that SQ algorithms provide, except for the differences that our results apply to statistical noise, to a weaker  learning requirement that focuses on an average-case rather than worst-case guarantee on the function class, and to possibly non-statistical queries as in SGD (with an account given on the batch size dependencies). We refer to Section \ref{sq} for further discussions on SQ algorithms and statistical dimension, as well as to \cite{boix_mds} for further comparisons. Note that the dependency on the batch size is particularly important: with batch-size 1, we show that SGD is universal, and this breaks down as the batch size gets polynomial.  

Therefore, while SGD can be viewed as a surrogate to GD that is  computationally less expensive (but less effective in convex settings), SGD turns out to be universal while GD is not. Note the stochasticity of SGD has already been advocated in different contexts, such as stability, implicit regularization or to avoid bad critical points \cite{faster,rethink,bad_min,bad_min2}.

As mentioned earlier, the amount of noise under which SGD can still learn in our positive result is large enough to break down not only GD, but more generally SQ algorithms. For example, our positive result shows that SGD can learn efficiently parities with some poly-noise, while GD or SQ algorithms break down in such cases. Note that parities were also known to be hard as far back as Minsky and Papert for the perceptron \cite{perceptron}, and our positive result requires indeed more than a single hidden layer to succeed.

Thus deep nets trained with SGD can be more powerful for generalization than deep nets trained with GD or than SQ algorithms. 

To complement the story, we also obtain negative results about SGD under low cross-predictability if additional constraints are added on the number of weights that can be updated per time steps (as in coordinate descent), or when larger amounts of noise are added on the initialization and on the gradients. 

Informal results are discussed in Section \ref{init} and formal definitions and results are given in Section \ref{results}.

\subsection{Problem formulations and learning objectives}
We focus on Boolean functions to simplify the setting.
Since it is known that any Boolean function that can be computed in time $O(T(n))$ can also be expressed by a neural network of size $O(T(n)^2)$ \cite{parberry,sipser}, it is not meaningful to ask whether any such function $f_0$ can be learned with a poly-size NN and a descent algorithm that has degree of freedom on the initialization and knowledge of $f_0$; one can simply pre-set the net to express $f_0$. Two more meaningful questions that one can ask are:
\begin{enumerate}
    \item Can one learn a given function with an agnostic/random\footnote{A random initialization means  i.i.d.\  weights as discussed in Section \ref{init}.} initialization?
    \item Can one learn an unknown function from a class or distribution with some choice of the initialization?
\end{enumerate}
For the second question, one is not give a specific function $f_0$ but a class of functions, or more generally, a distribution on functions. 

We focus here mainly on question 2, which gives a more general framework than restricting the initialization to be random.  Moreover, in the case of symmetric function distributions, such as the parities discussed below, failure at 2 implies failure at 1. Namely, if we cannot learn a parity function for a random selection of the support $S$ (see definitions below), we cannot learn any given parity function on a typical support  $S_0$ with a random initialization of the net, because the latter is symmetrical. Nonetheless, question 1 may also be interesting for applications, as random (or random-like) initializations may be used in practice. We discuss in Section \ref{init} how we expect that our results and the notion of cross-predictability export to the setting of question 1.

We thus have the following setting:
\begin{itemize}
    \item Let $\D=\{+1,-1\}$ and $\X=\D^n$ be the data domain and let $\Y=\{+1,-1\}$ be the label domain. We work with binary vectors and binary labels for convenience (several of the results  extend beyond this setting with appropriate reformulation of  definitions).
    \item Let $P_\X$ be a probability distribution on the data domain $\X$ and $P_\F$ be a probability distribution on $\Y^\X$ (the set of functions from $\X$ to $\Y$). We also assume for convenience that these distributions lead to balanced classes, i.e., that $P(F(X)=1)=1/2+o_n(1)$ when $(X,F) \sim P_\X \times P_\F$ (non-balanced cases require adjustments of the definitions).
    \item Our goal is to learn a function $F$ drawn under $P_\F$ by observing labelled examples $(X,Y)$ with $X \sim P_\X$, $Y=F(X)$. 
    \item In order to learn $F$ we can train our algorithm on labelled examples with a descent algorithm starting with an initialization $f^{(0)}$ and running for a number of steps $T=T(n)$ (other parameters of the algorithm such as the learning rate are also specified). In the case of GD, each step accesses the full distribution of labelled examples, while for SGD, it only accesses a single labelled example per step (see definitions below). In all cases, after the training with $(f^{(0)},T)$, the algorithm produces an estimator $\hat{F}_{f^{(0)},T}$ of $F$. 
    In order to study negative results, we will set what is arguably the least demanding learning requirement: we say that `typical-weak learning' is solvable in $T$ time steps for the considered $(P_\X,P_\F)$, if a net with initialization $f^{(0)}$ can be constructed such that: 
    \begin{align}
       \text{{\it Typical-weak learning:}} \quad P(\hat{F}_{f^{(0)},S}(X)=F(X))=1/2+\Omega_n(1), \label{twl}
    \end{align}
    where the above probability is over $(X,F)\sim (P_\X\times P_\F)$ and any randomness potentially used by the algorithm. 
    In other words, after training the algorithm on some initialization, we can predict the label of a new fresh sample from $P_\X$ with accuracy strictly better than random guessing, and this takes place when the unknown function is drawn under $P_\F$. 
\end{itemize}
Failing at typical-weak learning implies failing at most other learning requirements. For example, failing at typical weak learning for a uniform distribution on a certain class of  functions implies failing at PAC learning that class of functions. 
However, for our positive results with SGD, we will not only show that one can typically weakly learn efficiently any function distribution that is typically weakly learnable, but that we can in fact reproduce whatever accuracy an algorithm can achieve for the considered distribution. To be complete we need to define accuracy and typical weak learning for more general algorithms: 

\begin{definition}
Let $n>0$, $P_{\X}$ be a probability distribution on $\X=\mathcal{D}^n$ for some set $\mathcal{D}$, and $P_\F$ be a probability distribution on the set of functions from $\X$ to $\{+1,-1\}$. Assume that these distributions lead to balanced classes, i.e., $P(F(X)=1)=1/2+o_n(1)$ when $(X,F) \sim P_\X \times P_\F$. 

Consider an algorithm $A$ that, given access to an oracle that uses $P_\X$ and $F \sim P_\F$ (e.g., samples under $P_\X$ labelled by $F$), outputs a function $\hat{F}$.
Then $A$ learns $(P_\F,P_{\X})$ with accuracy $\alpha$ if  $\pp\{\hat{F}(X)= F(X)\} \ge \alpha$, where the previous probability is taken over $(X,F) \sim P_\X \times P_\F$ and any randomness potentially used by $\hat{F}$. In particular, we say that $A$ (typically-weakly) learns $(P_\F,P_{\X})$ if it learns $(P_\F,P_{\X})$ with accuracy $1/2+\Omega_n(1)$.
\end{definition}

From now on we often shorten `typical-weak learning' to simply `learning'. We also talk about learning a `function distribution' or a `distribution' when referring to learning a pair $(P_\X,P_\F)$. 

{\bf Example.} The problem of learning parities corresponds normally to $P_\X$ being uniform on $\{+1,-1\}^n$ and $P_\F$ being uniform on the set of parity functions defined by $\mathcal{P}=\{ p_s:  s \subseteq [n] \}$, where $p_s: \{+1,-1\}^n  \to \{+1,-1\}$ is such that $p_s(x)=\prod_{i  \in s} x_i.$ 
So nature picks $S$ uniformly at random in $2^{[n]}$, and with knowledge of $\mathcal{P}$ but not $S$, the problem is to learn which set $S$ was picked from samples $(X,p_S(X))$.

\subsection{Informal results: Cross-predictability, junk-flow and universality}\label{init}
\begin{definition}
For a positive integer $m$, a probability measure $P_\X$ on the data domain $\X$, and a probability measure $P_{\F}$ on the class of functions $\F$ from $\X$ to $\Y=\{+1,-1\}$, we define the cross-predictability by 
\begin{align}
\mathrm{CP}_m(P_\X, P_{\F}) := \E_{(X^m,F,F')\sim P_\X^m \times P_\F \times P_\F} (\E_{X\sim P_{X^m}} F(X) F'(X))^2, \label{cpm}
\end{align}
where $X^m=(X_1,\dots,X_m)$ has i.i.d.\ components under $P_\X$, $F,F'$ are independent of $X^m$ and i.i.d.\ under $P_{\F}$, and $X$ is drawn independently of $(F,F')$ under the empirical measure of $X^m$, i.e., $P_{X^m}=\frac{1}{m}\sum_{i=1}^m \delta_{X_i}$.
\end{definition}
Note the following equivalent representations:
\begin{align}
\mathrm{CP}_m(P_\X, P_{\F}) = \frac{1}{m}+ \left(1-\frac{1}{m} \right)\mathrm{CP}_\infty(P_\X, P_{\F}),
\end{align}
where
\begin{align}
\mathrm{CP}_\infty(P_\X, P_{\F}) 
&:=\E_{F,F'\sim P_\F} (\E_{X\sim P_\X} F(X) F'(X))^2\\
&=\E_{X,X'\sim P_\X} (\E_{F\sim P_\F} F(X) F(X'))^2 \label{dual}\\
&= \| \E_F \mathcal{F}(F)^{\otimes 2} \|_2^2
\end{align}
and $\mathcal{F}(F)$ is the Fourier-Walsh transform of $F$ with respect to the measure $P_\X$.


This measures how predictable a sampled function is from another one on a typical data point, or equivalently, how predictable a sampled data label is from another one on a typical function. The data point is drawn either from the true distribution or the empirical one depending on whether $m$ is infinity or not, and $m$ will refer to the batch-size in the GD context (i.e., how many samples are used to compute gradients). 
Equivalently, this measures the typical correlation among functions.
For example, if $P_\X$ is a delta function, then $\mathrm{CP}_\infty$ achieves the largest possible value of 1, and for purely random input and purely random functions, $\mathrm{CP}_\infty$ is $2^{-n}$, the lowest possible value.

Our negative results primarily exploit a low cross-predictability (CP). 
We obtain the following lower bound on the generalization error\footnote{Here $\mathrm{gen}$ is $1$ minus the probability of guessing the right label, i.e., the complement of \eqref{twl}.} of the output of GD with noise $\sigma$ and batch-size $m$, 
\begin{align}
    \mathrm{gen} &\ge \frac{1}{2}-\frac{1}{\sigma} \cdot \mathrm{JF} \cdot \left(\frac{1}{m}+ \mathrm{CP}_\infty \right)^{1/4} \label{cocktail}
\end{align}
where JF is the junk flow, a quantity that does not depend on $F$ and $P_\F$ but that depends on the net initialization, and that consists of the accumulation of gradient norm when GD
is run on randomly labelled data (i.e., junk labels; see Definition \ref{}):
\begin{align}
\mathrm{JF}:=\sum_{i=1}^T \gamma_t \| \E_{X_i, Z_i} \nabla L_{W^{(i)}}(X_i,Z_i) \|_2 
\end{align}
In particular, no matter what the initialization is, JF and $1/\sigma$ are polynomial if the neural net, the GD hyper-parameters (including the range of derivatives) and the time steps are all polynomial.    
Thus, if the batch-size is super-polynomial (or a large enough polynomial) and the CP is inverse-super-polynomial (or a low enough polynomial), no matter what the net initialization and architecture are, we do not generalize. This implies that full gradient does not learn, but SGD may still learn as the right hand side of \eqref{cocktail} does no longer tend to $1/2$ when $m=1$. In fact, this is no coincidence as we next show that SGD is indeed universal.

Namely, for any distribution that can be learned by some algorithm in poly-time, with poly-many samples and with accuracy $\alpha$, there exists an initialization (which means a neural net architecture with an initial assignment of the weights) that is constructed in poly-time and agnostic to the function to be learned, such that training this neural net with SGD and possibly poly-noise learns this distribution in poly-steps with accuracy $\alpha-o(1)$. Again, this does not take place once SGD is replaced by full gradient descent (or with large enough poly batches), or once SQ algorithms are used.

{\bf Example.} For random degree-$k$ monomials and uniform inputs, $CP_\infty \asymp {n\choose k}^{-1}$. Thus, GD with the above constraints can learn random degree $k$ monomials if and only if $k=O(1)$. The same outcome takes place for SQ algorithms. Other examples dealing with connectivity of graphs and community detection are discussed in Section \ref{others}.

The main insight for the negative results is that all of the deep learning algorithms that we consider essentially take a neural net, attempt to compute how well the functions computed by the net and slightly perturbed versions of the net correlate with the target function, and adjust the net in the direction of higher correlation. If none of these functions have significant correlation with the target function, this will generally make little or no progress. More precisely, if the target function is randomly drawn from a class with negligible cross-predictability, and if one cannot operate with noiseless GD, then no function is significantly correlated with the target function with nonnegligible probability and a descent algorithm will generally fail to learn the function in a polynomial time horizon.

{\bf Failures for random initializations.}
Consider the function  $f_s(x)=\prod_{i \in s} x_i$ for a specific subset $s$ of $[n]$. One can use our negative result for  function {\it distributions} on any initialization, to obtain a negative result for that specific function $f_s$ on a {\it random} initialization.  For this, construct the `orbit' of $f_s$, $\{f_S: S \subseteq [n]\}$; put a measure on subsets $S$ such that $s$ belongs to the typical set for that measure, i.e., the i.i.d.\ Ber$(p)$ measure such that $np=|s|$. Then, if one cannot learn under this distribution with any initialization, one cannot learn a typical function such as $f_s$ with a random i.i.d.\ initialization due to the symmetry of the model.

We also conjecture that the cross-predictability measure can be used to understand when a given function $h$ cannot be learned in poly-time with GD/SGD on poly-size nets that are randomly initialized, without requiring the stronger negative result for all initializations and the argument of previous paragraph. 

Namely, define the cross-predictability between a target function and a random neural net as 
\begin{align}
\mathrm{Pred}( P_\X, h , \mu_{NN} ) = \E_{G} (\E_{X} h(X) \mathrm{eval}_{G,f}(X))^2, 
\end{align}
where $(G,f)$ is a random neural net under the distribution $\mu_{NN}$, i.e., $f$ is a fixed non-linearity, $G$ is a random graph that consists
of complete bipartite\footnote{One could consider other types of graphs but a certain amount of randomness has to be present in the model.} graphs between consecutive layers of a poly-size NN, with weights i.i.d.\ centered Gaussian of variance equal to one over the width of the previous layer, and $X \sim P_\X$ is independent of $G$. We then conjecture that if such a cross-predictability decays super-polynomially, training such a  random neural net with a polynomial number of steps of GD or SGD will fail at learning {\it even without noise or memory constraints}.
Again, as mentioned above, if the target function is permutation invariant, it cannot be learned with a random initialization and noisy GD with  small random noise. So the claim is that the random initialization gives already enough randomness in one step to cover all the added randomness from noisy GD.

\section{Results}\label{results}

\subsection{Definitions and models}
In this paper we will be using a fairly generic notion of neural nets, simply weighted directed acyclic graphs with a special set of vertices for the inputs, a special vertex for the output, and a non-linearity at the other vertices. The formal definition is as follows. 

\begin{definition}
A neural net is a pair of a function $f:\mathbb{R}\rightarrow \mathbb{R}$ and a weighted directed graph $G$ with some special vertices and the following properties. First of all, $G$ does not contain any cycle. Secondly, there exists $n>0$ such that $G$ has exactly $n+1$ vertices that have no edges ending at them, $v_0$, $v_1$,...,$v_n$. We will refer to $n$ as the input size, $v_0$ as the constant vertex and $v_1$, $v_2$,..., $v_n$ as the input vertices. Finally, there exists a vertex $v_{out}$ such that for any other vertex $v'$, there is a path from $v'$ to $v_{out}$ in $G$.  We also denote by $w(G)$ the weights on the edges of $G$. 
\end{definition}


\begin{definition}
Given a neural net $(f,G)$ with input size $n$, and $x\in\mathbb{R}^n$, the evaluation of $(f,G)$ at $x$, written as $eval_{(f,G)}(x)$ (or $eval_{(G)}(x)$ if $f$ is implicit), is the scalar computed by means of the following procedure:
(1) Define $y\in \mathbb{R}^{|G|}$ where $|G|$ is the number of vertices in $G$, set $y_{v_0}=1$, and set $y_{v_i}=x_i$ for each $i$; (2) Find an ordering $v'_1,...,v'_m$ of the vertices in $G$ other than the constant vertex and input vertices such that for all $j>i$, there is not an edge from $v'_j$ to $v'_i$; (3) For each $1\le i\le m$, set 
$y_{v'_i}=f\left(\sum_{v: (v,v'_i)\in E(G)} w_{v,v'_i} y_v\right)$; (4) Return $y_{v_{out}}$.
\end{definition}
We will also sometimes use a shortcut notation for the $\mathrm{eval}$ function; for a neural net $G$ with a set of weights $W$, we will sometimes use\footnote{There is an abuse of notation between $W(G)$ and $W(X)$ but the type of input in $W()$ makes the interpretation clear.} $W(x)$ for $\mathrm{eval}_G(x)$.

The trademark of deep learning is to do this by defining a loss function in terms of how much the network's outputs differ from the desired outputs, and then using a descent algorithm to try to adjust the weights based on some initialization. More formally, if our loss function is $L$, the function we are trying to learn is $h$, and our net is $(f,G)$, then the net's loss at a given input $x$ is $L(h(x)-eval_{(f,G)}(x))$ (or more generally $L(h(x),eval_{(f,G)}(x))$). Given a probability distribution for the function's inputs, we also define the net's expected loss as $\E[L(h(X)-eval_{(f,G)}(X))]$. 

We will focus in this paper on GD, SGD, and for one part on block-coordinate descent, i.e., updating not all the weights at once but only a subset based on some rule (e.g., steepest descent). 
We will also consider noisy versions of some of these algorithms. This would be the same as the noise-free version, except that in each time step, the algorithm independently draws a noise term for each edge from some probability distribution and adds it to that edge's weight. Adding noise is sometimes advocated to help avoiding getting stuck in local minima or regions where the derivatives are small \cite{perturbed_sgd}, however it can also drown out information provided by the gradient. 

\begin{remark}
As we have defined them, neural nets generally give outputs in $\mathbb{R}$ rather than $\{0,1\}$. As such, when talking about whether training a neural net by some method learns a given Boolean function, we will implicitly be assuming that the output of the net on the final input is thresholded at some predefined value or the like. None of our results depend on exactly how we deal with this part (one could have alternatively worked with the mutual information between the true label and the real-valued output of the net).
\end{remark}


We want to answer the question of whether or not training a neural net with these algorithms is a universal method of learning, in the sense that it can learn anything that is reasonably learnable. We next recall what this means. 

\begin{definition}\label{learn_sample} Let $n>0$, $\epsilon>0$, $P_{\X}$ be a probability distribution on $\{0,1\}^n$, and $P_\F$ be a probability distribution on the set of functions from $\{0,1\}^n$ to $\{0,1\}$. Also, let $X_0, X_1,...$ be independently drawn from $P_{\X}$ and $F\sim P_\F$. An algorithm learns $(P_\F,P_{\X})$ with accuracy $1/2+\epsilon$ in $T$ time steps if the algorithm is given the value of $(X_i, F(X_i))$ for each $i<T$ and, when given the value of $X_T \sim P_\X$ independent of $F$, it returns $Y_T$ such that $P[F(X_T)=Y_T]\ge 1/2+\epsilon$. \end{definition}
Algorithms such as SGD (or Gaussian elimination from samples) fit under this definition. For SGD, the algorithm starts with an initialization $W^{(0)}$ of the neural net weights, and updates it sequentially with each sample $(X_i,F(X_i))$ as 
$W^{(i)}=g(X_i,F(X_i),W^{(i-1)}):=W^{(i-1)}-\gamma \nabla L(\mathrm{eval}_{W^{(i-1)}}(X_i),F(X_i))$, $i \in [T-1]$. It then outputs $Y_T=\mathrm{eval}_{W^{(T-1)}}(X_T)$. 

For GD however, in the idealized case where the gradient is averaged over the entire sample set, or more formally, when one has access to the exact expected gradient under $P_\X$, we are not accessing samples as in the previous definition. We then talk about learning a distribution with an algorithm like GD under the following more general setup.

\begin{definition} Let $n>0$, $\epsilon>0$, $P_{\X}$ be a probability distribution on $\{0,1\}^n$, and $P_\F$ be a probability distribution on the set of functions from $\{0,1\}^n$ to $\{0,1\}$. An algorithm learns $(P_\F,P_{\X})$ with accuracy $1/2+\epsilon$, if 
given the value of $X \sim P_\X$ independent of $F$, 
 it returns $Y$ such that $P[F(X)=Y]\ge 1/2+\epsilon$. \end{definition}
 Obviously the algorithm must access  some information about the function $F$ to be learned.
In particular, GD proceeds successively with the following $(F,P_\X)$-dependent updates 
 $W^{(i)}=\E_{X\sim P_\X} g(X,F(X),W^{(i-1)})$ for $i \in [T-1]$ for the same function $g$ as in SGD. 

Recall also that we talk about  ``learning parities" in the case where $P_\F$ picks a parity function uniformly at random and $P_\X$ is uniform on $\{+1,-1\}^n$, as defined in Section \ref{model}.


\begin{definition}
For each $n>0$, let\footnote{Note that these are formally sequences of distributions.} $P_{\X}$ be a probability distribution on $\{0,1\}^n$, and $P_{\F}$ be a probability distribution on the set of functions from $\{0,1\}^n$ to $\{0,1\}$. We say that $(P_\F,P_{\X})$ is efficiently learnable if there exists $\epsilon>0$, $N>0$, and an algorithm with running time polynomial in $n$ such that for all $n\ge N$, the algorithm learns $(P_{\F},P_{\X})$ with accuracy $1/2+\epsilon$. In the setting of Definition \ref{learn_sample}, we further say that the algorithm takes a polynomial number of samples (or has polynomial sample complexity) if the algorithm learns $(P_\F,P_{\X})$ and $T$ is polynomial in $n$. Note that an algorithm that learns in poly-time using samples as in Definition \ref{learn_sample} must have a polynomial sample complexity as well as polynomial memory.
\end{definition}

\subsection{Positive  results}\label{positive}

We show that if SGD is initialized properly and run with enough resources, it is in fact possible to learn efficiently and with polynomial sample complexity any efficiently learnable distribution that has polynomial sample complexity.  

\begin{theorem}\label{thm_univ}
For each $n>0$, let $P_\X$ be a probability measure on $\{0,1\}^n$, and $P_{\F}$ be a probability measure on the set of functions from $\{0,1\}^n$ to $\{0,1\}$. Also, let $Ber(1/2)$ be the uniform distribution on $\{0,1\}$. Next, define $\alpha=\alpha_n$ such that there is some algorithm that takes a polynomial number of samples $(X_i,F(X_i))$ where the $X_i$ are i.i.d.\ under $P_{\X}$, runs in polynomial time, and learns $(P_\F,P_\X)$ with accuracy $\alpha$. Then there exists $\gamma=o(1)$, a polynomial-sized neural net $(G_n,\phi)$, and a polynomial $T_n$ such that using stochastic gradient descent with learning rate $\gamma$ to train $(G_n,\phi)$ on $T_n$ samples $((X_i,R_i,R'_i), F(X_i))$ where $(X_i,R_i,R'_i)\sim P_\X\times Ber(1/2)^2$ learns $(P_\F,P_\X)$ with accuracy $\alpha-o(1)$.
\end{theorem}

\begin{remark}
 One can construct in polynomial time in $n$ a neural net $(\phi,g)$ that has polynomial size in $n$ such that for a learning rate $\gamma$ that is at most polynomial in $n$ and an integer $T$ that is at most polynomial in $n$, $(\phi,g)$ trained by SGD with learning rate $\gamma$ and $T$ time steps learns parities with accuracy $1-o(1)$. In other words, random bits are not needed for parities, because  parities can be learned from a deterministic algorithms which can use only samples that are labelled 1 without producing bias. 
\end{remark} 

Further, previous result can be extended when sufficiently low amounts of inverse-polynomial noise are added to the weight of each edge in each time step. More formally, we have the following result.

\begin{theorem}\label{thm_univ2}
For each $n>0$, let $P_\X$ be a probability measure on $\{0,1\}^n$, and $P_{\F}$ be a probability measure on the set of functions from $\{0,1\}^n$ to $\{0,1\}$. Also, let $B_{1/2}$ be the uniform distribution on $\{0,1\}$, $t_n$ be polynomial in $n$, and $\delta\in [-1/n^2t_n,1/n^2t_n]^{t_n\times |E(G_n)|}$, $x^{(i)}\in\{0,1\}^{n}$. Next, define $\alpha_n$ such that there is some algorithm that takes $t_n$ samples $(x_i,F(x_i))$ where the $x_i$ are independently drawn from $P_{\X}$ and $F\sim P_{\F}$, runs in polynomial time, and learns $(P_\F,P_\X)$ with accuracy $\alpha$. Then there exists $\gamma=\Theta(1)$, and a polynomial-sized neural net $(G_n,f)$ such that using perturbed stochastic gradient descent with noise $\delta$, learning rate $\gamma$, and loss function $L(x)=x^2$ to train $(G_n,f)$ on $t_n$ samples\footnote{The samples are converted to take values in $\pm 1$ for consistency with other sections.} $((2x_i-1,2r_i-1), 2F(x_i)-1)$ where $(x_i,r_i)\sim P_\X\times B_{1/2}$ learns $(P_\F,P_\X)$ with accuracy $\alpha-o(1)$.
\end{theorem}

While the learning algorithm used does not put a bound on how high the edge weights can get during the learning process, we can do this in such a way that there is a constant that the weights will never exceed. Furthermore, instead of emulating an algorithm chosen for a specific distribution, we could, for any $c>0$, emulate a metaalgorithm that learns any distribution that is learnable by an algorithm working with an upper bound $n^c$ on the number of samples and the time needed per sample. Thus we could have an initialization of the net that is polynomial and agnostic to the specific distribution $(P_\F,P_\X)$ (and not only the actual function drawn from $P_\F$) as long as this one is learnable with the above $n^c$ constraints, and SGD run in poly-time with poly-many samples and possibly inverse-poly noise will succeed in learning. This is further explained in Remark \ref{kolmogorov}.

\subsection{Negative results}
We saw that training neural nets with SGD and polynomial parameters is universal in that it can learn any efficiently learnable distribution. We now show that this universality is broken once full gradient descent is used, or once larger noise on the initialization and gradients are used, or once fewer weights are updated as in coordinate descent. For this purpose, we look for function distributions that are efficiently learnable by some algorithm but not by the considered deep learning algorithms.

\subsubsection{GD with noise}\label{proof2}

\begin{definition}[Noisy GD with batches]
For each $n > 0$, take a neural net of size $|E(n)|$, with any differentiable\footnote{One merely needs to have gradients well-defined.} non-linearity and any initialization of the weights $W^{(0)}$, and train it with gradient descent with learning rate $\gamma_t$, any differentiable loss function, gradients computed at each step from $m$ fresh samples from the  distribution $P_\X$ with labels from $F$, a derivative range\footnote{We call the range or the overflow range of a function to be $A$ if any value of the function potentially exceeding $A$ (or $-A$) is rounded at $A$ (or $-A$).} of $A$, additive Gaussian noise of variance $\sigma^2$, and $T$ steps, i.e.,        
\begin{align}
W^{(t)}= W^{(t-1)} - \gamma_t  \E_{X \sim P_{S^{(t)}}} \left[\nabla L(W^{(t-1)}(X),F(X)) \right]_A + Z^{(t)}, \quad t=1,\dots, T,
\end{align}
where $\{Z^{(t)}\}_{t \in [T]}$ are i.i.d.\ $\mathcal{N}(0,\sigma^2)$ (independent of other random variables) and $\{S^{(t)}\}_{t \in [T]}$ are i.i.d.\ where $S^{(t)}=(X_1^{(t)},\dots, X_m^{(t)})$ has i.i.d.\ components under $P_\X$. \end{definition}

\begin{definition}[Junk Flow] Using the notation in previous definition, define the junk flow of an initialization $W^{(0)}$ with data distribution $P_\X$, $T$ steps and learning rate  $\{\gamma_t\}_{t \in [T]}$ by 
\begin{align}
\mathrm{JF}=\mathrm{JF}(W^{(0)}, P_\X, T, \{\gamma_t\}_{t \in [T]}):=
 \sum_{t=1}^T  \gamma_{t} \| \E_{X,Y} [\nabla L(W_{\star}^{(t-1)}(X),Y)]_A  \|_2.
\end{align}
where $(X,Y)\sim P_\X \times U_\Y$, $W_\star^{(0)}=W^{(0)}$ and $W_\star^{(t)}= W_\star^{(t-1)} - \gamma_t  \E_{X,Y} \left[\nabla L(W_\star^{(t-1)}(X),Y) \right]_A + Z^{(t)}$, $t\in [T]$. That is, the junk flow is the power series over all time steps of the norm of the expected gradient when running noisy GD on junk samples, i.e., $(X,Y)$ where $X$ is a random input under $P_\X$ and $Y$ is a (junk) label that is  independent of $X$ and uniform. 
\end{definition}

\begin{theorem}\label{thm2'}
Let $P_{\X}$ with $\X=\mathcal{D}^n$ for some finite set $\mathcal{D}$ and $P_{\F}$ such that the output distribution is balanced,\footnote{Non-balanced cases can be handled by modifying definitions appropriately.} i.e., $\pp\{F(X)=0\}=\pp\{F(X)=1\}+o_n(1)$ when $(X,F) \sim P_{\X} \times P_{\F}$. Recall the definitions of cross-predictability $\mathrm{CP}_m=CP(m,P_{\X}, P_{\F} )$ and junk-flow $\mathrm{JF}_T=\mathrm{JF}(W^{(0)}, P_\X, S, \{\gamma_t\}_{t \in [T]})$. 
Then,
\begin{align}
\pp\{ W^{(T)}(X)=F(X)\}&\le 1/2+\frac{1}{\sigma} \cdot  \mathrm{JF}_T \cdot  \mathrm{CP}_m^{1/4}\\
&\le 1/2+\frac{1}{\sigma} \cdot  \mathrm{JF}_T \cdot  (1/m+\mathrm{CP}_\infty)^{1/4}
\end{align}
\end{theorem}
\begin{corollary}\label{looseb}
If the derivatives of the gradient have an overflow range of $A$ and if the learning rate is constant at $\gamma$, then 
  $$\mathrm{JF}_T \le \gamma T \sqrt{|E|} A, $$
  and a deep learning system as in previous theorem with $M:=\max(\gamma ,\frac{1}{\sigma}, A , |E|, T)$ polynomial in $n$ cannot learn under $(P_\X,P_{\F})$ if $\mathrm{CP}_m$ decays super-polynomially in $n$ (or more precisely if $\mathrm{CP}_m^{-1/4}$ is a larger polynomial than $M$).
\end{corollary}

\begin{corollary}
A deep learning system as in previous theorem with $\max(\gamma ,\frac{1}{\sigma}, A , |E|, T)$ polynomial in $n$ can learn a random degree-$k$ monomial with full GD if and only if $k=O(1)$. 
\end{corollary}
The positive statement in the previous corollary uses the fact that it is easy to learn random degree-$k$ parities with neural nets and GD when $k$ is finite, see for example \cite{etienne} for a specific implementation.  

\begin{remark} We now argue that in the results above, all constraints are qualitatively needed. Namely, the requirement that the cross predictability is low is necessary because otherwise we could use an easily learnable function. Without bounds on $|E|$ we could build a net with sections designed for every possible value of $F$, and without a bound on $T$ we might be able to simply let the net change haphazardly until it stumbles upon a configuration similar to the target function. If we were allowed to set an arbitrarily large value of $\gamma$ we could use that to offset the small size of the function's effect on the gradient, and if there was no noise we could initialize parts of the net in local maxima so that whatever changes GD caused early on would get amplified over time. Without a bound on $A$ we could design the net so that some edge weights had very large impacts on the net's behavior in order to functionally increase the value of $\gamma$.
\end{remark}

In the following, we apply our proof technique from Theorem \ref{thm2'} to the specific case of parities, with a tighter bound obtained that results in the term $CP^{1/2}$ rather than $CP^{1/4}$. The following follows from this tighter version. 
\begin{theorem}\label{thm2}
For each $n > 0$, let $(f,g)$ be a neural net of polynomial size in $n$. Run gradient descent on $(f,g)$ with less than $2^{n/10}$ time steps, a learning rate of at most $2^{n/10}$, Gaussian noise with variance at least $2^{-n/10}$ and overflow range of at most $2^{n/10}$. For all sufficiently large $n$, this algorithm fails at learning parities with accuracy $1/2 + 2^{-n/10}$.
\end{theorem}
See Section \ref{sq} for more details on how the above compares to \cite{kearns}. In particular, an application of \cite{kearns} would not give the same exponents for the reasons explained in \ref{sq}. More generally, Theorem \ref{thm2'} applies to low cross-predictability functions which do not necessarily have large statistical dimension --- see Section \ref{related} for examples and further details. In the other cases the SQ framework gives the relevant qualitative bounds.


\begin{remark}
Note first that having GD run with a little noise is not equivalent to having noisy labels for which learning parities can be hard irrespective of the algorithm used \cite{parity_blum,regev}. In addition, the amount of noise needed for GD in the above theorem can be exponentially small, and if such amount of noise were added to the sample labels themselves, then the noise would essentially be ineffective (e.g., Gaussian elimination would still work with rounding, or if the noise were Boolean with such variance, no flip would take place with high probability). The  failure is thus due to the nature of the GD algorithm.
\end{remark}

\begin{remark}
Note that the positive results show that we could learn a random parity function using stochastic gradient descent under these conditions. The reason for the difference is that SGD lets us get the details of single samples, while GD averages all possible samples together. In the latter case, the averaging mixes together information provided by different samples in a way that makes it harder to learn about the function.
\end{remark}

\subsubsection{SGD with memory constraint}

\begin{theorem}\label{thm1'}
Let $\epsilon>0$, and $P_{\F}$ be a probability distribution over functions with a cross-predictability of $\mathrm{c_p}=o(1)$. For each $n > 0$, let $(f,g)$ be a neural net of polynomial size in $n$ such that each edge weight is recorded using $O(\log(n))$ bits of memory. Run stochastic gradient descent on $(f,g)$ with at most $\mathrm{c_p}^{-1/24}$ time steps and with $o(|\log(\mathrm{c_p})|/\log(n))$ edge weights updated per time step.  For all sufficiently large $n$, this algorithm fails at learning functions drawn from $P_{\F}$ with accuracy $1/2 + \epsilon$.
\end{theorem}

\begin{corollary}
Block-coordinate descent with a polynomial number of steps and precision and $o(n/\log(n))$ edge updates per step fails at learning parities with non-trivial accuracy.   
\end{corollary}

\begin{remark}
Specializing the previous result to the case of parities, one obtains the following. 
Let $\epsilon>0$. For each $n > 0$, let $(f,g)$ be a neural net of polynomial size in $n$ such that each edge weight is recorded using $O(\log(n))$ bits of memory. Run stochastic gradient descent on $(f,g)$ with at most $2^{n/24}$ time steps and with $o(n/\log(n))$ edge weights updated per time step.  For all sufficiently large $n$, this algorithm fails at learning parities with accuracy $1/2 + \epsilon$.

As discussed in Section \ref{related},
one could obtain the special case of Theorem \ref{thm1'} for parities using \cite{parity_conj} with the following argument. If bounded-memory SGD could learn a random parity function with nontrivial accuracy, then we could run it a large number of times, check to see which iterations learned it reasonably successfully, and combine the outputs in order to compute the parity function with an accuracy that exceeded that allowed by Corollary 4 in \cite{parity_conj}. However, in order to obtain a generalization of this argument to low cross-predictability functions, one would need to address the points made in Section \ref{related} regarding statistical dimension and cross-predictability.
\end{remark}

\begin{remark}In the case of parities, the emulation argument allows us to show that one can learn a random parity function using SGD that updates $O(n)$ edge weights per time step. With some more effort we could have made the memory component encode multiple bits per edge. This would have allowed it to learn parity if it was restricted to updating $O(n/m)$ edges of our choice per step, where $m$ is the maximum number of bits each edge weight is recorded using.
\end{remark}

\subsubsection{SGD with additional randomness}

In the case of full gradient descent and low cross-predictability, the gradients of the losses with respect to different inputs mostly cancel out, so an exponentially small amount of noise is enough to drown out whatever is left. With stochastic gradient descent, that does not happen, and we have the following instead.

\begin{definition}
Let $(f,g)$ be a NN, and recall that $w(g)$ denotes the set of weights on the edges of $g$. Define the $\tau$-neighborhood of $(f,g)$ as 
\begin{align}
N_{\tau}(f,g)=\{(f,g'): E(g')=E(g),  |w_{u,v}(g)-w_{u,v}(g')| \le \tau , \forall (u,v) \in E(g) \}.    
\end{align}
\end{definition}

\begin{theorem}\label{thm3}
For each $n>0$, let $(f,g)$ be a neural net with size $m$ polynomial in $n$, and let $B,\gamma,T>0$. There exist $\sigma=O(m^2\gamma^2 B^2/n^2)$ and $\sigma'=O(m^3 \gamma^3 B^3/n^2)$ such that the following holds. Perturb the weight of every edge in the net by a Gaussian distribution of variance $\sigma$ and then train it with a noisy stochastic gradient descent algorithm with learning rate $\gamma$, $T$ time steps, and Gaussian noise with variance $\sigma'$. Also, let $p$ be the probability that at some point in the algorithm, there is a neural net $(f,g')$ in $N_{\tau}(f,g)$, $\tau=O(m^2\gamma B/n)$, such that at least one of the first three derivatives of the loss function on the current sample with respect to some edge weight(s) of $(f,g')$ has absolute value greater than $B$. Then this algorithm fails to learn parities with an accuracy greater than $1/2+2p+O(Tm^4B^2\gamma^2/n)+O(T[e/4]^{n/4})$.
\end{theorem}

\begin{remark}
Normally, we would expect that if training a neural net by means of SGD works, then the net will improve at a rate proportional to the learning rate, as long as the learning rate is small enough. As such, we would expect that the number of time steps needed to learn a function would be inversely proportional to the learning rate. This theorem shows that if we set $T=c/\gamma$ for any constant $c$ and slowly decrease $\gamma$, then the accuracy will approach $1/2+2p$ or less. If we also let $B$ slowly increase, we would expect that $p$ will go to $0$, so the accuracy will go to $1/2$. It is also worth noting that as $\gamma$ decreases, the typical size of the noise terms will scale as $\gamma^{3/2}$. So, for sufficiently small values of $\gamma$, the noise terms that are added to edge weights will generally be much smaller than the signal terms. 
\end{remark}

\begin{remark}
The bound on the derivatives of the loss function is essentially a requirement that the behavior of the net be stable under small changes to the weights. It is necessary because otherwise one could effectively multiply the learning rate by an arbitrarily large factor simply by ensuring that the derivative is very large. Alternately, excessively large derivatives could cause the probability distribution of the edge weights to change in ways that disrupt our attempts to approximate this probability distribution using Gaussian distributions. For any given initial value of the neural net, any given smooth activation function, and any given $M>0$, there must exists some $B$ such that as long as none of the edge weights become larger than $M$ this will always hold. However, that $B$ could be very large, especially if the net has many layers. 
\end{remark}

\begin{remark}
The positive results show that it is possible to learn a random parity function using a polynomial sized neural net trained by stochastic gradient descent with inverse-polynomial noise for a polynomial number of time steps. Furthermore, this can be done with a constant learning rate, a constant upper bound on all edge weights, a constant $\tau$, and $B$ polynomial in $n$ such that none of the first three derivatives of the loss function of any net within $\tau$ of ours are greater than $B$ at any point. So, this result would not continue to hold for all choices of exponents.
\end{remark}

\subsection{Proof techniques: undistinguishability, emulation and sequential learning algorithms}  
{\bf Negative results.} Our main approach to showing the failure of an algorithm (e.g., noisy GD) using data from a  model (e.g, parities)  for a desired task (e.g., typical weak learning), will be to show that under limited resources (e.g., limited number of time steps), the   output of the algorithm trained on the true model is {\it statistically indistinguishable} from the output of the algorithm trained on a null model, where the null model fails to provide the desired performance for trivial reasons. This forces the true model to fail as well.

The indistinguishability to null condition (INC) is obtained by manipulating information measures, bounding the total variation distance of the two posterior measures between the test and null models. The failure of achieving the desired algorithmic performance on the test model is then a consequence of the INC, either by converse arguments -- if one could achieve the claimed  performance, one would be able to use the  performance gap to distinguish the null and test models and thus contradict the INC -- or directly using the total variation distance between the two probability distributions to bound the difference in the probabilities that the nets drawn from those distributions compute the function correctly (and we know that it fails to do so on the null model).

An example with more details:
\begin{itemize}
\item Let $D_1$ be the  distribution of the data for the parity learning model, i.e., i.i.d.\ samples with labels from the parity model in dimension $n$;
\item Let $R=(R_1,R_2)$ be the resource in question, i.e., the number $R_1$ of edge weights of poly-memory that are updated and the number of steps $R_2$ of the algorithm;
\item Let $A$ be the coordinate descent algorithm used with a constraint $C$ on the resource $R$; 
\item Let $T$ be the task, i.e, achieving an accuracy of $1/2 + \Omega_n(1)$ on a random input.
\end{itemize}

\noindent
Our program then runs as follows:
\begin{enumerate}
    \item Chose $D_0$ as the null distribution that generates i.i.d.\ pure noise labels, such that the task $T$ is obviously not achievable for $D_0$.
\item Find a INC on $R$, i.e., a constraint $C$ on $R$ such that the trace of the algorithm $A$ is indistinguishable under $D_1$ and $D_0$; to show this,
\begin{enumerate}
    \item show that the total variation distance between the posterior distribution of the trace of $A$ under $D_0$ and $D_1$ vanishes if the INC holds; 
    \item to obtain this, it is sufficient to show that any $f$-mutual information between the algorithm's trace and the model hypotheses $D_0$ or $D_1$ (chosen equiprobably) vanishes.
\end{enumerate}
\item Conclude that the INC on $R$ prohibits the achievement of $T$ on the test model $D_0$, either by contradiction  as one could use $T$ to distinguish between $D_1$ and $D_0$ if only the latter fails at $T$ or using the fact that for any event $\mathrm{Success}$ and any random variables $Y(D_i)$ that depend on data drawn from $D_i$ (and represent for example the algorithms outputs), we have $\pp\{Y(D_1) \in \mathrm{Success} \} \le \pp\{Y(D_0) \in \mathrm{Success} \}+TV(D_0,D_1) = 1/2 + TV(D_0,D_1)$.
\end{enumerate}
Most of the work then lies in part 2(a)-(b), which consist in manipulating information measures to obtain the desired conclusion. In particular, the Chi-squared mutual information will be convenient for us, as its ``quadratic'' form will allow us to bring  the cross-predictability as an upper-bound, which is then ``easier'' to evaluate. This is carried out in Section \ref{proof2} in the context of GD and in Section \ref{sla} in the  context of so-called ``sequential learning algorithms''.

In the case of noisy GD (Theorems \ref{thm2} and \ref{thm2'}), the program is more direct from step 2, and runs with the following specifications. When computing the full gradient, the losses with respect to different inputs mostly cancel out, which makes the gradient updates reasonably  small, and a small amount of noise suffices to cover it. We then show a subadditivity property of the TV using the data processing inequality,  bound the one step  total variation distance with the KL distance (Pinsker's inequality), which in the Gaussian case gives the $\ell_2$ distance, and then use a change of measure argument to bring down the cross-predictability (using various generic inequalities).

In the case of the failure of SGD under noisy initialization and updates (Theorem \ref{thm3}), we rely on a more sophisticated version of the above program. We use again a step used for GD that consists in showing that the average value of any function on samples generated by a random parity function will be approximately the same as the average value of the function on true random samples.\footnote{This gives also a variant of a result in \cite{ohad} applying to the special case of 1-Lipschitz loss function.} This is essentially a consequence of the low cross-predictability. Most of the work then is using this to show that if we draw a set of weights in $\mathbb{R}^m$ from a sufficiently noisy probability distribution and then perturb it slightly in a manner dependent on a sample generated by a random parity function, the probability distribution of the result is essentially indistinguishable from what it would be if the samples were truly random. Then, we argue that if we do this repeatedly and add in some extra noise after each step, the probability distribution stays noisy enough that the previous result continues to apply. 
After that, we show that the probability distribution of the weights in a neural net trained by noisy stochastic gradient descent on a random parity function is indistinguishable from the the probability distribution of the weights in a neural net trained by noisy SGD on random samples, which represents most of the work.

\noindent
{\bf Sequential Learning algorithms.} Our negative results exploit the sequential nature of descent algorithms such as gradient, stochastic gradient or coordinate descent. That is, the fact that these algorithms proceed by querying some function on some samples (typically the gradient function), then update the memory structure according to some rule (typically the neural net weights using a descent algorithm step), and then forget about these samples. We next formalize this class of algorithms using the notion of sequential learning algorithms (SLA).

\begin{definition}   
A sequential learning algorithm $A$ on $(\mathcal{Z}, \mathcal{W})$ is an algorithm that for an input of the form $(Z,(W_1,...,W_{t-1}))$ in $\mathcal{Z} \times \mathcal{W}^{t-1}$ produces an output $A(Z,(W_1,...,W_{t-1}))$ valued in $\mathcal{W}$. 
Given a probability distribution $D$ on $\mathcal{Z}$, a sequential learning algorithm $A$ on $(\mathcal{Z}, \mathcal{W})$, and $T\ge 1$, a $T$-trace of $A$ for $D$ is a series of pairs $((Z_1, W_1), ...,(Z_T, W_T))$ such that for each $i \in [T]$, $Z_i\sim D$ independently of $(Z_1,Z_2,...,Z_{i-1})$ and $W_i=A(Z_i,(W_1,W_2,...,W_{i-1}))$.
\end{definition}
Note that $Z$ may represent a single sample with its label (and $D$ the corresponding distribution) as for SGD, or a collection of $m$ i.i.d.\ samples as for mini-batch GD.  
Our negative result for SGD in Theorem \ref{thm3} will apply more generally to such algorithms, with constraints added on the number of weights that can be updated per time step. For Theorems \ref{thm2'} and \ref{thm3}, we use further assumption on how the memory (weights) are updated, i.e., via the subtraction of gradients. These correspond to special cases of SLAs where the following memory update rules are used:    
\begin{align}
W^{(t)}= W^{(t-1)} -  \E_{X \sim \hat{P}_{S_m^{(t)}}} G_{t-1}(W^{(t-1)}(X),F(X)) + Z^{(t)}, \quad t=1,\dots,T \label{sda}
\end{align}
where $G_t$ is some function valued in some bounded range (like the query function in statistical query algorithms) and  $\hat{P}_{S_m^{(t)}}=\frac{1}{m}\sum_{i=1}^m \delta_{X_i^{(t)}}$ is the empirical distribution of $m$ samples (with $m=1$ for SGD and larger $m$ for GD).

{\bf Positive result.} For the positive result, we emulate any learning algorithm using poly-many samples and running in poly-time with poly-size neural nets trained by poly-step SGD. This requires emulating any poly-size circuit implementation with free access to reading and writing in memory using a particular computational model that computes, reads and writes memory solely via SGD steps on a fixed neural net. In particular, this requires designing subnets that perform arbitrary efficient computations in such a way that SGD does not alter them and subnet structures that cause SGD to change specific edge weights in a manner that we can control. One difficulty encountered with such an SGD implementation is that no update of the weights will take place when given a sample that is correctly predicted by the net. If one does not mitigate this, the net may end up being trained on a sample distribution that is mismatched to the original one, which can have unexpected consequences. A randomization mechanism is thus used to circumvent this issue.\footnote{This mechanism is not necessary for cases like parities.} See Section \ref{universality} for further details.

\section{Related literature}\label{related}

\subsection{Minsky and Papert}
The difficulty of learning  functions like parities with NNs is not new. Together with the connectivity case, the difficulty with parities was in fact one of the central focus in the perceptron book of Minksy and Papert \cite{perceptron}, which resulted in one of the main cause of skepticism regarding neural networks in the 70s \cite{bottou_perso}. The sensitivity of parities is also well-studied in the theoretical computer science literature, with the relation to circuit complexity, in particular the computational limitations of small-depth circuit   \cite{hastad,allender}. The seminal paper of Kearns on statistical query learning algorithms \cite{kearns} brings up the difficulties in learning parities with such algorithms, as discussed next.

\subsection{Statistical querry algorithms}\label{sq} The lack of correlations between two parity functions and its implication in learning parities is extensively studied in the context of statistical query learning algorithms \cite{kearns}. These algorithms have access to an oracle that gives estimates on the expected value of some query function over the underlying data distribution. The main result of \cite{kearns,query} gives a tradeoff for learning a  function class in terms of (i) the statistical dimension (SD) that captures the largest possible number of functions in the class that are weakly correlated, (ii) the precision range $\tau$, that controls the error added by the oracle to each query valued in the range of $[-1,1]$, (iii) the number of queries made to the oracle. In particular, parities have exponential SD and thus for a polynomial error $\tau$, an exponential number of queries are needed to learn them.      
Gradient-based algorithms with approximate oracle access are realizable as statistical query algorithms, since the gradient takes an expectation of some function (the derivative of the loss). In particular, \cite{kearns} implies that the class of parity functions cannot be learned by such algorithms, which  
implies a result similar in nature to our Theorem \ref{thm2} as further discussed below. 
The result from \cite{kearns} and its generalization in \cite{parity_blum} have however a few differences from those presented here. First these papers define successful learning for {\it all} function in a class of functions, whereas we work here with {\it typical} functions from a function distribution, i.e., succeeding with non-trivial probability according to some  function distribution that may not be a uniform distribution.  
Second these papers require the noise to be adversarial, while we use here statistical noise, i.e., a less powerful adversary. We also focus on guessing the label with a better chance than random guessing; this can also obtained for the SQ algorithms but the classical definition of SD is typically not designed for this case. Finally the proof techniques are different, mainly based on Fourier analysis in \cite{parity_blum} and on hypothesis testing and information theory here. 

Nonetheless, our Theorem \ref{thm2} admits a quantitative counter-part in the SQ framework \cite{kearns}. 
Technically \cite{kearns} only says that a SQ algorithm with a polynomial number of queries and inverse polynomial noise cannot learn a parity function, but the proof would still work with appropriately chosen exponential parameters. To further convert this to the setting with statistical noise, one could use an argument saying that the Gaussian noise is large enough to mostly drown out the adversarial noise if the latter is small enough, but the resulting bounds would be slightly looser than ours because that would force one to make trade offs between making the amount of adversarial noise in the SQ result low and minimizing the probability that one of the queries does provide meaningful information. Alternately, one could probably rewrite their proof  using Gaussian noise instead of bounded adversarial noise and bound sums of $L_1$ differences between the probability distributions corresponding to different functions instead of arguing that with high probability the bound on the noise quantity is high enough to allow the adversary to give a generic response to the query. 

To see how Theorem \ref{thm2'} departs from the setting of \cite{parity_blum} beyond the statistical noise discussed above, note that the cross-predictability captures the expected inner product $\langle F_1,F_2 \rangle_{P_\X}$ over two i.i.d.\ functions $F_1,F_2$ under $P_\F$, whereas the statistical dimension defined in \cite{parity_blum} is the largest number $d$ of functions $f_i \in \F$ that are nearly orthogonal, i.e, $|\langle f_i,f_j \rangle_{P_\X}| \le 1/d^3$, $1 \le i< j \le d$. Therefore, while the cross-predictability and statistical dimension tend to be negatively correlated, one can construct a family $\F$ that contains many almost orthogonal functions, yet with little mass under $P_\F$ on these so that the distribution has a high cross-predictability. For example, take a class containing two types of functions, hard and easy, such as parities on sets of components and almost-dictatorships which agree with the first input bit on all but $n$ of the inputs. The parity functions are orthogonal, so the union contains a set of size  $2^n$ that is pairwise orthogonal. However, there are about $2^n$ of the former and $2^{n^2}$ of the latter, so if one picks a function uniformly at random on the union, it will belong to the latter group with high probability, and the cross-predictability will be $1-o(1)$.
 So one can build examples of function classes where it is possible to learn with a moderate cross-predictability while the statistical dimension is large and learning fails in the sense of \cite{parity_blum}.

There have been many follow-up works and extensions of the statistical dimension and SQ models. We refer to \cite{boix_mds} for a more in-depth discussion and comparison between these and the results in this paper. In particular, \cite{vempala2} allows for a probability measure on the functions as well. The statistical dimension as defined in Definition 2.6 of \cite{vempala2} measures the maximum probability subdistribution with a sufficiently high correlation among its members (note that this defined in view of studying exact rather than weak learning). As a result, any probability distribution with a low cross predictability must have a high statistical dimension in that sense. However, a distribution of functions that are all moderately correlated with each other could have an arbitrarily high statistical dimension despite having a reasonably high cross-predictability. For example, using definition 2.6 of \cite{vempala2} with constant $\bar{\gamma}$ on the collection of functions from $\{0,1\}^n \to \{0,1\}$ that are either 1 on $1/2+\sqrt{\gamma}/4$ of the possible inputs or 1 on $1/2-\sqrt{\gamma}/4$ of the inputs, gives a statistical dimension with average correlation $\gamma$ that is doubly exponential in $n$. However, this has a cross predictability of $\gamma^2/16$.

In addition, queries in the SQ framework typically output the exact expected value with some error, but do not provide the tradeoff that occur by taking a number of samples and using these to estimate the expectation, as provided with the variable $m$ in Theorem \ref{thm2'}. In particular, as $m$ gets low, one can no longer obtain negative results as shown with Theorem \ref{thm_univ}.  

Regarding Theorem \ref{thm1'}, one could imagine a way to obtain it using prior SQ works by proving the following: (a) generalize the paper of \cite{parity_conj} that establishes a result similar to our Theorem \ref{thm1'} for the special case of parities to the class of low cross-predictability functions, (b) show that this class has the right notion of statistical dimension that is high. However, the distinction between low cross-predictability and high statistical dimension would kick in at this point. If we take the example mentioned in the previous paragraph, the version of SGD used in Theorem \ref{thm1'} could learn to compute a function drawn from this distribution with expected accuracy $1/2+\sqrt{\gamma}/8$ given $O(1/\gamma)$ samples, so the statistical dimension of the distribution is not limiting its learnability by such algorithms in an obvious way. One might be able to argue that a low cross-predictability implies a high statistical dimension with a value of $\gamma$ that vanishes sufficiently quickly and then work from there. However, it is not clear exactly how one would do that, or why it would give a preferred approach.

Paper \cite{query2} also shows that gradient-based algorithms with approximate oracle access are realizable as statistical query algorithms, however, \cite{query2} makes a convexity assumption that is not satisfied by non-trivial neural nets. SQ lower bounds for learning with data generated by neural networks is also investigated in \cite{song} and for neural network models with one hidden nonlinear activation layer in \cite{wilmes}.

Finally, the current SQ framework does not apply to noisy SGD (even for adversarial noise). 
One may consider instead 1-STAT oracles, that provide a query from random sample, but we did not find results comparable to our Theorem \ref{thm3} in the literature. 

In fact, we show that it is possible to learn parities with better noise-tolerance and complexity than any SQ algorithm will do (see Section \ref{positive}), so the variance in the random queries of SGD is crucial to make it a universal algorithm as opposed to GD or any SQ algorithm.

\subsection{Memory-sample trade-offs} In \cite{ran_memory}, it is shown that one needs either quadratic memory or an exponential number of samples in order to learn parities, settling a conjecture from \cite{parity_conj}. This gives a  non-trivial lower bound on the number of samples needed for a learning problem and a complete negative result in this context, with applications to bounded-storage cryptography  \cite{ran_memory}. 
Other works have extended the results of \cite{ran_memory}; in particular \cite{ran_sparse} applies to k-sparse sources, \cite{ran17} to other functions than parities, and \cite{grt} exploits properties of two-source extractors to obtain comparable memory v.s.\ sample complexity trade-offs, with similar results obtained in \cite{bogy}. The cross-predictability has also similarity with notions of almost orthogonal matrices used in $L_2$-extractors for two independent sources \cite{two_source,grt}.

    In contrast to this line of works (i.e., \cite{ran_memory} and follow-up papers), our Theorem \ref{thm1'} (when specialized to the case of parities) shows that one needs exponentially many samples to learn parities if less than $n/24$ pre-assigned bits of memory are used {\it per sample}. These are thus different models and results. Our result does not say anything interesting about our ability to learn parities with an algorithm that has free access to memory, while the result of \cite{ran_memory} says that it would need to have $\Omega(n^2)$ total memory or an exponential number of samples. On the flip side, our result shows that an algorithm with unlimited amounts of memory will still be unable to learn a random parity function from a subexponential number of samples if there are sufficiently tight limits on how much it can edit the memory while looking at each sample, which cannot be concluded from \cite{ran_memory}. The latter is relevant to study SGD with a bounded number of weight updates per time step as discussed in this paper. 
    
    Note also that for the special case of parities, one could aim for Theorem \ref{thm1'} using \cite{parity_conj} with the following argument. If bounded-memory SGD could learn a random parity function with nontrivial accuracy, then we could run it a large number of times, check to see which iterations learned it reasonably successfully, and combine the outputs in order to compute the parity function with an accuracy that exceeded that allowed by Corollary 4 in \cite{parity_conj}. However, in order to obtain a generalization of this argument to low cross-predictability functions, one would need to address the points made previously regarding Theorem \ref{thm1'} and \cite{parity_conj} (namely points (a) and (b) in the previous subsection).

\subsection{Gradient concentration} Finally, \cite{ohad}, with an earlier version in \cite{ohad2} from the first author, also give strong support to the impossibility of learning parities. 
In particular the latter discusses whether specific assumptions on the ``niceness'' of the input distribution or the target function (for example based on notions of smoothness, non-degeneracy, incoherence or random choice of parameters), are sufficient to guarantee learnability using gradient-based methods, and evidences are provided that neither class of assumptions alone is sufficient.

\cite{ohad} gives further theoretical insights  and practical experiments on the failure of learning parities in such context. More specifically, it proves that the gradient of the loss function of a neural network will be essentially independent of the parity function used. This is achieved by a variant of our  Lemma \ref{new-pred} below with the requirement in \cite{ohad} that the loss function is 1-Lipschitz\footnote{The proofs are both simple but slightly different, in particular our Lemma \ref{new-pred} does not make  regularity assumptions.}. This provides a strong intuition of why one should not be able to learn a random parity function using gradient descent or one of its variants, and this is backed up with theoretical and experimental evidence. However, it is not proved that one cannot learn parity using SGD, batch-SGD or the like. The implication is far from trivial, as with the right algorithm, it is indeed possible to reconstruct the parity function from the gradients of the loss function on a list of random inputs. In fact, we show here that it is possible to learn parities in polynomial time by SGD with small enough batches and a careful poly-time initialization of the net (that is agnostic to the parity function).Thus, obtaining formal negative results requires more specific assumptions and elaborate proofs, already for GD and particularly for SGD.

\section{Some challenging functions}\label{others}
\subsection{Parities}\label{model}
The problem of learning parities corresponds to $P_\X$ being uniform on $\{+1,-1\}^n$ and $P_\F$ being uniform on the set of parity functions defined by $\mathcal{P}=\{ p_s:  s \subseteq [n] \}$, where $p_s: \{+1,-1\}^n  \to \{+1,-1\}$ is such that $$p_s(x)=\prod_{i  \in s} x_i.$$ 
So nature picks $S$ uniformly at random in $2^{[n]}$, and with access to $\mathcal{P}$ but not to $S$, the problem is to learn which set $S$ was chosen from samples $(X,p_S(X))$ as defined in previous section. 

Note that without noise, this is {\it not} a hard problem.
Even exact learning of the set $S$ (with high probability) can be achieved if we do not restrict ourselves to using a NN trained with a descent algorithm. One can simply take an algorithm that builds a basis from enough samples (e.g., $n + \Omega(\log(n))$) and solves the resulting system of linear equations to reconstruct $S$.

This seems however far from how deep learning proceeds. For instance, descent algorithms are ``memoryless'' in that they update the weights of the NN at each step but do not a priori explicitly remember the previous steps. Since each sample (say for SGD) gives very little information about the true $S$, it thus seems unlikely for SGD to make any progress on a polynomial time horizon. However, it is far from trivial to argue this formally if we allow the NN to be arbitrarily large and with arbitrary initialization (albeit of polynomial complexity), and in particular inspecting the gradient will typically not suffice.

In fact, we will show that this is wrong, and SGD {\it can} learn the parity function with a proper initialization --- See Sections \ref{positive} and \ref{universality}. 
We will then show that using GD with small amounts of noise, as sometimes advocated in different forms \cite{perturbed_sgd,langevin1,langevin2}, or using  (block-)coordinate descent or more generally bounded-memory update rules, it is in fact not possible to learn parities with deep learning in poly-time steps. Parities corresponds in fact to an extreme instance of a distribution with low cross-predictability, to which failures apply, and which is related to statistical dimension in statistical query algorithms; see Section \ref{related}.

An important point is is that that the amount of noise that we will add is  smaller than the amount of noise\footnote{Note also that having GD run with little noise is not exactly equivalent to having noisy labels.} needed to make parities
hard to learn \cite{parity_blum,regev}. The amount of noise needed for GD to fail can be exponentially small, which would effectively represent no noise if that noise was added on the labels itself as in learning with errors (LWE); e.g., Gaussian elimination would still work in such regimes.

As discussed in Section \ref{init}, in the case of parities, our negative result for any initialization can be converted into a negative result for random initialization. 
We believe however that the randomness in a random initalization would actually be enough to account for any small randomness added subsequently in the algorithm steps. Namely, that one cannot learn parities with GD/SGD in poly-time with a random initialization.

To illustrate the phenomenon, we consider the following data set and numerical experiment in PyTorch \cite{paszke2017automatic}. The elements in $\X$ are images with a white background and either an even or odd number of black dots, with the parity of the dots determining the label  --- see Figure \ref{images}. 
The dots are drawn by building a $k \times k$ grid with white background and activating each square with probability $1/2$.

We then train a neural network to learn the parity label of these images with a random initalization. The architecture is a 3 hidden linear layer perceptron with 128 units and ReLU non linearities trained using binary cross entropy. The training\footnote{We pick  samples from a pre-set training set v.s.\ sampling  fresh samples; these are not expected to behave differently.} and testing dataset are composed of 1000 images of grid-size $k=13$. We used PyTorch implementation of SGD with step size 0.1 and i.i.d.\ rescaled uniform weight initialization \cite{imagenet2}.


\begin{figure}[H]
\centering
  \includegraphics[width=0.5\linewidth]{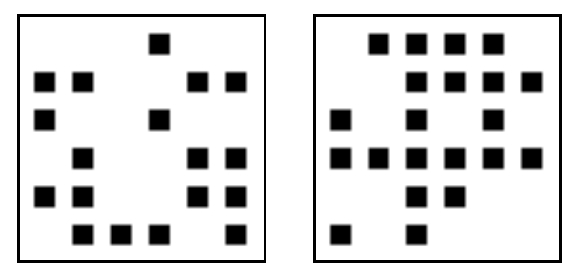}
\caption{Two images of $13^2=169$ squares colored black with probability $1/2$. The left (right) image has an even (odd) number of black squares. The experiment illustrates the incapability of deep learning to learn the parity.}
\label{images}
\end{figure}

Figure \ref{accuracy} show the evolution of the training loss, testing and training errors. As can be seen, the net can learn the training set but does not  generalize better than random guessing.  

\begin{figure}[H]
\centering
  \includegraphics[width=.9\linewidth]{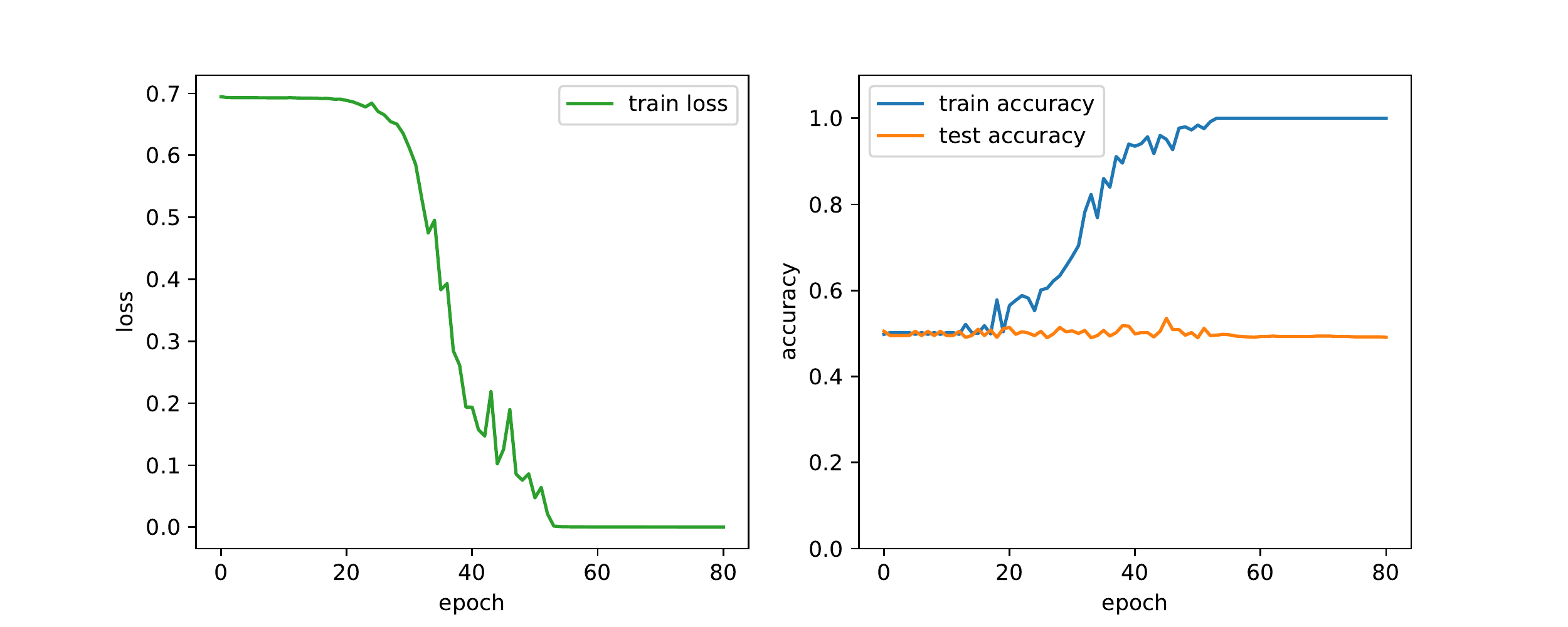}
\caption{Training loss (left) and training/testing errors (right) for up to 80 SGD epochs.}
\label{accuracy}
\end{figure}

\subsection{Community detection and connectivity}

Parities are not the most common type of functions used to generate real signals, but they are central to the construction of good codes (in particular the most important class of codes, i.e., linear codes, that rely heavily on parities). We mention now a few specific examples of functions that we believe would be also difficult to learn with deep learning. Connectivity is another notorious example discussed in the Perceptron book of Minsky-Papert \cite{perceptron}. In that vain, we provide here a different and concrete question related to connectivity and community detection. We then give another example of low cross-predictability distribution in arithmetic learning.   \\

Consider the problem of determining whether or not some graphs are connected. This could be difficult because it is a global property of the graph, and there is not necessarily any function of a small number of edges that is correlated with it. Of course, that depends on how the graphs are generated. In order to make it difficult, we define the following probability distribution for random graphs.

\begin{definition}
Given $n,m,r>0$, let $AER(n,m,r)$ be the probability distribution of $n$-vertex graphs generated by the following procedure. First of all, independently add an edge between each pair of vertices with probability $m/n$ (i.e., start with an Erd\H{o}s-R\'enyi random graph). Then, randomly select a cycle of length less than $r$ and delete one of its edges at random. Repeat this until there are no longer any cycles of length less than $r$.
\end{definition}

Now, we believe that deep learning with a random initialization will not be able to learn to distinguish a graph drawn from $AER(n,10\ln(n),\sqrt{\ln(n)})$ from a pair of graphs drawn from $AER(n/2,10\ln(n),\sqrt{\ln(n)})$, provided the vertices are randomly relabeled in the latter case. That is, deep learning will not distinguish between a patching of two such random graphs (on half of the vertices) versus a single such graph (on all vertices). Note that a simple depth-first search algorithm would learn the function in poly-time. More generally, we believe that deep learning would not solve community detection on such variants of random graph models\footnote{It would be interesting to investigate the approach of \cite{bruna_cd} on such models.} (with edges allowed between the clusters as in a stochastic block model with similar loop pruning), as connectivity v.s.\ disconnectivity is an extreme case of community detection.

The key issue is that no subgraph induced by fewer than $\sqrt{\ln(n)}$ vertices provides significant information on which of these cases apply. Generally, the function computed by a node in the net can be expressed as a linear combination of some expressions in small numbers of inputs and an expression that is independent of all small sets of inputs. The former cannot possibly be significantly correlated with the desired output, while the later will tend to be uncorrelated with any specified function with high probability. As such, we believe that the neural net would fail to have any nodes that were meaningfully correlated with the output, or any edges that would significantly alter its accuracy if their weights were changed. Thus, the net would have no clear way to improve.

\subsection{Arithmetic learning} Consider trying to teach a neural net arithmetic. More precisely, consider trying to teach it the following function. The function takes as input a list of $n$ numbers that are written in base $n$ and are $n$ digits long, combined with a number that is $n+1$ digits long and has all but one digit replaced by question marks, where the remaining digit is not the first. Then, it returns whether or not the sum of the first $n$ numbers matches the remaining digit of the final number. So, it would essentially take expressions like the following, and check whether there is a way to replace the question marks with digits such that the expression is true.
\begin{align*}
&120\\
+&112\\
+&121\\
=?&?0?
\end{align*}

Here, we can define a class of functions by defining a separate function for every possible ordering of the digits. If we select inputs randomly and map the outputs to $\mathbb{R}$ in such a way that the average correct output is $0$, then this class will have a low cross predictability. Obviously, we could still initialize a neural net to encode the function with the correct ordering of digits. However, if the net is initialized in a way that does not encode the digit's meanings, then deep learning will have difficulties learning this function comparable to its problems learning parity. Note that one can sort out which digit is which by taking enough samples where the expression is correct and the last digit of the sum is left, using them to derive linear equation in the digits $\pmod{n}$, and solving for the digits.

We believe that if the input contained the entire alleged sum, then deep learning with a random initialization would also be unable to learn to determine whether or not the sum was correct. However, in order to train it, one would have to give it correct expressions far more often than would arise if it was given random inputs drawn from a probability distribution that was independent of the digits' meanings. As such, our notion of cross predictability does not apply in this case, and the techniques we use in this paper do not work for the version where the entire alleged sum is provided. The techniques instead apply to the above version.

\subsection{Beyond low cross-predictability}
We showed in this paper that SGD can learn efficiently any efficiently learnable distribution despite some polyn-noise. One may wonder when this takes place for GD. 

In the case of random degree $k$ monomials, i.e., parity functions on a uniform subset $S$ of size $k$ with uniform inputs, we showed that GD fails at learning under memory or noise constraints as soon as $k=\omega(1)$. This is because the cross-predictability scales as ${n \choose k}^{-1}$, which is already super-polynomial when $k=\omega(1)$.  

On the flip side, if $k$ is constant, it is not hard to show that GD can learn this function distribution by inputting all the ${n \choose k}$ monomials in the first layer (and for example the cosine non-linearity to compute the parity in one hidden layer). Further, one can run this in a robust-to-noise fashion, say with exponentially low noise, by implementing AND or OR gates properly \cite{etienne}.  Therefore, for random degree $k$ monomials, deep learning can learn efficiently and robustly {\it if and only if} $k=O(1)$.  Thus one can only learn low-degree functions in that class.

We believe that small cross-predictability does not take place for typical labelling functions concerned with images or sounds, where many of the functions we would want to learn are correlated both with each other and with functions a random neural net is reasonably likely to compute. For instance, the objects in an image will correlate with whether the image is outside, which will in turn correlate with whether the top left pixel is sky blue. A randomly initialized neural net is likely to compute a function that is nontrivially correlated with the last of these, and some perturbations of it will correlate with it more, which means the network is in position to start learning the functions in question.

Intuitively, this is due to the fact that images and image classes have more compositional structures (i.e., their labels are well explained by combining `local' features). Instead, parity functions of large support size, i.e., not constant size but growing size, are not well explained by the composition of local features of the vectors, and require more global operations on the input. As a result, using more samples for the gradients may never hurt in such cases.

Another question is whether or not GD with noise can successfully learn a random function drawn from any distribution with a cross-predictability that is at least the inverse of a polynomial.

The first obstacle to learning such a function is that some functions cannot be computed to a reasonable approximation by any neural net of polynomial size. A probability distribution that always yields the same function has a cross-predictability of $1$, but if that function cannot be computed with nontrivial accuracy by any polynomial-sized neural net, then any method of training such a net will fail to learn it. 

Now, assume that every function drawn from $P_{\F}$ can be accurately computed by a neural net with polynomial size. If $P_{\F}$ has an inverse-polynomial cross-predictability, then two random functions drawn from the distribution will have an inverse-polynomial correlation on average. In particular, there exists a function $f_0$ and a constant $c$ such that if $F\sim P_{\F}$ then $\E_{F} (\E_{X} F(X) f_0(X))^2=\Omega(n^{-c})$. Now, consider a neural net $(G,\phi)$ that computes $f_0$. Next, let $(G',\phi)$ be the neural net formed by starting with $(G,\phi)$ then adding a new output vertex $v$ and an intermediate vertex $v'$. Also, add an edge of very low weight from the original output vertex to $v'$ and an edge of very high weight from $v'$ to $v$. This ensures that changing the weight of the edge to $v'$ will have a very large effect on the behavior of the net, and thus that SGD will tend to primarily alter its weight. That would result in a net that computes some multiple of $f_0$. If we set the loss function equal to the square of the difference between the actual output and the desired output, then the multiple of $f_0$ that has the lowest expected loss when trying to compute $F$ is $\E_{X}[f_0(X)F(X)] f_0$, with an expected loss of $1-\E^2_{X}(f_0(X)F(X))$. We would expect that training $(G',\phi)$ on $F$ would do at least this well, and thus have an expected loss over all $F$ and $X$ of at most $1-\E_{F} (\E_{X} F(X) f_0(X))^2=1-\Omega(n^{-c})$. That means that it will compute the desired function with an average accuracy of $1/2+\Omega(n^{-c})$. Therefore, if the cross-predictability is polynomial, one can indeed learn with at least a polynomial accuracy.  

However, we cannot do much better than this. To demonstrate that, consider a probability distribution over functions that returns the function that always outputs $1$ with probability $1/\ln(n)$, the function that always outputs $-1$ with probability $1/\ln(n)$, and a random function otherwise. This distribution has a cross-predictability of $\theta(1/\ln^2(n))$. However, a function drawn from this distribution is only efficiently learnable if it is one of the constant functions. As such, any method of attempting to learn a function drawn from this distribution that uses a subexponential number of samples will fail with probability $1-O(1/\ln(n))$. In particular, this type of example demonstrates that for any $g=o(1)$, there exists a probability distribution of functions with a cross-predictability of at least $g(n)$ such that no efficient algorithm can learn this distribution with an accuracy of $1/2+\Omega(1)$.

However, one can likely prove that a neural net trained by noisy GD or noisy SGD can learn $P_{\F}$ if it satisfies the following property. Let $m$ be polynomial in $n$, and assume that there exists a set of functions $g_1,...,g_m$ such that each of these functions is computable by a polynomial-sized neural net and the projection of a random function drawn from $P_{\F}$ onto the vector space spanned by $g_1,...,g_m$ has an average magnitude of $\Omega(1)$. In order to learn $P_{\F}$, we start with a neural net that has a component that computes $g_i$ for each $i$, and edges linking the outputs of all of these components to its output. Then, the training process can determine how to combine the information provided by these components to compute the function with an advantage that is within a constant factor of the magnitude of its projection onto the subspace they define. That yields an average accuracy of $1/2+\Omega(1)$. However, we do not think that this is a necessary condition to be able to learn a distribution using a neural net trained by noisy SGD or the like.

\section{Proofs of negative results}

\subsection{Proof of Theorem \ref{thm2'}}
Consider SGD with mini-batch of size $m$, i.e., 
for a sample set $S_m^{(t)}=\{X_1^{(t)},\dots,X_m^{(t)}\}$
define 
\begin{align}
    \hat{P}_{S_m^{(t)}}=\frac{1}{m}\sum_{i=1}^m \delta_{X_i^{(t)}}
\end{align}
and
\begin{align}
W^{(t)}= W^{(t-1)} -  \E_{X \sim \hat{P}_{S_m^{(t)}}} G_{t-1}(W^{(t-1)}(X),F(X)) + Z^{(t)}, \quad t=1,\dots,T \label{sda}
\end{align}
where $G_t=\gamma_t [\nabla L]_A$.

Theorem \ref{thm2'} holds for any sequential algorithm that edits its memory using \eqref{sda} for some function $G_t$ that is valued in $[-\gamma_t A, \gamma_t A]$. In particular, if one has access to a statistical query algorithm as in \cite{kearns} with a tolerance of $\tau$, one can `emulate' such an algorithm with a constant $\gamma$ by using $m=\infty$ and $\sigma / (\gamma A) =\tau$; this is however for a worst-case rather than  statistical noise model.

\begin{proof}[Proof of Theorem \ref{thm2'}]
Consider the same algorithm run on either true data labelled with $F$ or junk data labelled with random labels, i.e., 
\begin{align}
W_H^{(t)}= W_H^{(t-1)} -   \E_{(X,Y) \sim D^{(t)}_{H,m}} G_{t-1}(W^{(t-1)}(X),Y) + Z^{(t)}, \quad t=1,\dots, T,
\end{align}
where
\begin{align}
D^{(t)}_{H,m}(x,y)=
\begin{cases}
P_{S_m^{(t)}}(x) (1/2) & \text{ if } H= \star,\\
P_{S_m^{(t)}}(x) \delta_{F(X)}(y) & \text{ if } H= F.
\end{cases}
\end{align}
Denote by $Q^{(t)}_{H}$ the probability distribution of $W_H^{(t)}$
and let $S_m^t:=(S_m^{(1)},\dots, S_m^{(t)})$.
We then have the following.

\begin{align}
\pp\{W_F^{(T)}(X) = F(X)  \} &\le 
\pp\{ W_\star^{(T)}(X)= F(X)  \} + \E_{F,S_m^T} d( Q^{(T)}_{F} , Q^{(T)}_{\star} |F,S_m^T)_{TV}\\
&\le 1/2 +   \E_{F,S_m^T} d( Q^{(T)}_{F} , Q^{(T)}_{\star} |F,S_m^T)_{TV}.
\end{align}

For $t \in [T+1]$ $H,h \in \{F,\star\}$, define 
\begin{align}
W_{H,h}^{(t-1)}= W_H^{(t-1)} - \left(\E_{(X,Y) \sim D^{(t)}_{h,m}} G_{t-1}(W_H^{(t-1)}(X),Y) \right) + Z^{(t)},       
\end{align}
and denote by $Q^{(t-1)}_{H,h}$ the distribution of $W_{H,h}^{(t-1)}$. 

Using the triangular and Data-Processing inequalities, we have 
\begin{align}
&  d( Q^{(t)}_{F} , Q^{(t)}_{\star} | F,S_m^t )_{TV}\\
& \le d( Q^{(t-1)}_{F,F},Q^{(t-1)}_{\star, F} | F,S_m^t )_{TV} + d(Q^{(t-1)}_{\star, F}, Q^{(t-1)}_{\star,\star}  | F,S_m^t )_{TV}\\
& \le  d( Q^{(t-1)}_{F},Q^{(t-1)}_{\star} | F,S_m^{t-1} )_{TV} +  d(Q^{(t-1)}_{\star, F}, Q^{(t-1)}_{\star,\star}  | F,S_m^t )_{TV}\\
& =  d( Q^{(t-1)}_{F},Q^{(t-1)}_{\star}| F,S_m^{t-1}) _{TV} \\ 
& +  TV( \E_{(X,Y) \sim D^{(t)}_{m,F}} G_{t-1}(W_\star^{(t-1)}(X),Y) +Z^{(t)}, \E_{(X,Y) \sim D^{(t)}_{m,\star}} G_{t-1}(W_\star^{(t-1)}(X),Y)+Z^{(t)} | F,S_m^{t}).
\end{align}

Let $t$ fixed, $Z=(X,Y)$,  $g(Z):=G_{t-1}(W_\star^{(t-1)}(X),Y)$, $D_\cdot=D_\cdot^{(t)}$. 
By Pinsker's inequality\footnote{One can get an additional $1/\pi$ factor by exploiting the Gaussian distribution more tightly.},
\begin{align}
    TV(\E_{Z \sim D_{m,F}} g(Z) +Z^{(t)}, \E_{Z \sim D_{m,\star}} g(Z)  +Z^{(t)} | F,S_m^{t}) \le \frac{1}{2 \sigma}  \|\E_{Z \sim D_{m,F}} g(Z) - \E_{Z \sim D_{m,\star}} g(Z) \|_2
\end{align}
and by Cauchy-Schwarz,
\begin{align}
 &\E_{F}     TV(\E_{Z \sim D_{m,F}} g(Z)+Z^{(t-1)}, \E_{Z \sim D_{m,\star}} g(Z) +Z^{(t-1)} | F,S_m^{t})\\
 & \le \frac{1}{2 \sigma}  ( \E_{F}  \|\E_{Z \sim D_{m,F}} g(Z) - \E_{Z \sim D_{m,\star}} g(Z) \|_2^2)^{1/2}.\label{break}
\end{align}
We now investigate a single component $e\in E(G)$ appearing in the norm, 
\begin{align}
   & \E_{F}   (\E_{Z \sim D_{m,F}} g_e(Z) - \E_{Z \sim D_{m,\star}} g_e(Z) )^2 =     \E_{F}   ( \E_{Z \sim D_{m,\star}} g_e(Z) (1-D_{m,F}(Z)/D_{m,\star}(Z)) )^2 \label{break2} \\
   & = \E_{F}  \langle g_e, (1-D_{m,F}/D_{m,\star} ) \rangle^2_{D_{m,\star}}\\
   &= \E_{F}  \langle g_e^{\otimes 2}, (1-D_{m,F}/D_{m,\star} )^{\otimes 2} \rangle_{D_{m,\star}^2} \label{lift}\\
   &=   \langle g_e^{\otimes 2},\E_F  (1-D_{m,F}/D_{m,\star} )^{\otimes 2} \rangle_{D_{m,\star}^2}\\
    & \le [(\E_{Z \sim D_{m,\star}}  g_e(Z)^2) \|\E_F  (1-D_{m,F}/D_{m,\star} )^{\otimes 2} \|_{D_{m,\star}^2}]  \\
   & = (\E_{Z \sim D_{m,\star}}  g_e(Z)^2)( \E_{F,F'} [\E_{Z \sim D_{m,\star}}   (1-D_{m,F}(Z)/D_{m,\star}(Z)) (1-D_{F',m}(Z)/D_{m,\star}(Z)) ]^2 )^{1/2} \label{replic} \\
    & = (\E_{Z \sim D_{m,\star}}  g_e(Z)^2) CP(m,t)^{1/2} \label{cp12}
\end{align}
where \eqref{lift} uses a tensor lifting to bring the expectation over $F$ on the second component before using the Cauchy-Schwarz inequality, and where \eqref{replic} uses replicates, i.e., $(EZ)^2=\E Z_1 Z_2$ for $Z,Z_1,Z_2$ i.i.d., with 
\begin{align}
  CP(m,t)&:= \E_{F,F'} [\E_{Z \sim D^{(t)}_{m,\star}}  (1-2\delta_{F(X)}(Y)) (1-2\delta_{F'(X)}(Y)) ]^2\\
&=  \E_{F,F'} [\E_{X \sim P_{S_m^{(t)}}}  F(X)F'(X) ]^2
\end{align}
Therefore,
\begin{align}
&\E_F TV( \E_{(X,Y) \sim D^{(t)}_{m,F}} G_{t-1}(W_\star^{(t-1)}(X),Y) +Z^{(t)}, \E_{(X,Y) \sim D^{(t)}_{m,\star}} G_{t-1}(W_\star^{(t-1)}(X),Y)+Z^{(t)} | F,S_m^{t})\\
 &\le \frac{1}{2 \sigma} (\E_{Z \sim D^{(t)}_{m,\star}} \| G(W_\star^{(t-1)}(X),Y) \|_2) CP(m,t)^{1/4}
\end{align}
and
\begin{align}
&\E_{F,S_m^{t}} TV( \E_{(X,Y) \sim D^{(t)}_{m,F}} G_{t-1}(W_\star^{(t-1)}(X),Y) +Z^{(t)}, \E_{(X,Y) \sim D^{(t)}_{m,\star}} G_{t-1}(W_\star^{(t-1)}(X),Y)+Z^{(t)} | F,S_m^{t})\\
 &\le \frac{1}{2 \sigma} \E_{S_m^{t}} (\E_{Z \sim D^{(t)}_{m,\star}} \| G_{t-1}(W_\star^{(t-1)}(X),Y) \|_2) CP(m,t)^{1/4}.
\end{align}

Defining the gradient norm as 
\begin{align}
    GN(m,t):= \E_{Z \sim D^{(t)}_{m,\star}} \| G_{t-1}(W_\star^{(t-1)}(X),Y) \|_2.
\end{align}
we get
\begin{align}
 \E_{F,S_m^T}   d( Q^{(T)}_{F} , Q^{(T)}_{\star} |F,S_m^T)_{TV}
& \le \frac{1}{ \sigma} \cdot    \sum_{t=1}^{T} \E_{S_m^t} (GN(m,t) \cdot CP(m,t)^{1/4})  \\
& \le \frac{1}{ \sigma} \cdot    \sum_{t=1}^{T} (\E_{S_m^t} GN(m,t)^2)^{1/2} \cdot (\E_{S_m^t} CP(m,t)^{1/2})^{1/2} \label{csjunk}\\
& = \frac{1}{ \sigma} \cdot    \sum_{t=1}^{T} (\E_{S_m^t} GN(m,t)^2)^{1/2} \cdot (\E_{S_m} CP(m,1)^{1/2})^{1/2}
\end{align}
and thus
\begin{align}
 \E_{F,S_m^T} d( Q^{(T)}_{F} , Q^{(T)}_{\star}|F,S_m^T )_{TV}
& \le \frac{1}{ \sigma} \cdot    \sum_{t=1}^{T} (\E_{S_m^t} (\E_{Z \sim D^{(t)}_{m,\star}} \| G_{t-1}(W_\star^{(t)}(X),Y) \|_2)^2)^{1/2} \cdot CP_m^{1/4}\\
& \le \frac{1}{ \sigma} \cdot    \sum_{t=1}^{T} (\E_{S_m^t} \E_{Z \sim D^{(t)}_{m,\star}} \| G_{t-1}(W_\star^{(t)}(X),Y) \|_2^2)^{1/2} \cdot CP_m^{1/4}\\
& = \frac{1}{ \sigma} \cdot    \sum_{t=1}^{T} (\E_{X,Y\sim P_\X(1/2) }\| G_{t-1}(W_\star^{(t)}(X),Y) \|_2^2)^{1/2} \cdot CP_m^{1/4}
\end{align}

Finally note that
\begin{align}
&\E_{X,Y\sim P_\X(1/2) }\| G_{t-1}(W_\star^{(t)}(X),Y) \|_2^2=\|\E_{X,Y\sim P_\X(1/2) } G_{t-1}(W_\star^{(t)}(X),Y) \|_2^2,\\
&CP_m=\E_{F,F'}\E_{S_m}  [\E_{X\sim P_{S_m}}  F(X)F'(X) ]^2 =1/m+\left(1-1/m\right)CP_\infty.
\end{align}
\end{proof}

\begin{proof}[Proof of Corollary \ref{looseb}]
GN is trivially bounded by $A E^{1/2}$, so 
\begin{align}
 \E_{F,S_m^T} d( Q^{(T)}_{F} , Q^{(T)}_{\star}|F,S_m^T )_{TV}
& \le  \frac{A}{ \sigma} E^{1/2} T (1/m+\left(1-1/m\right)CP_{\infty})^{1/4}.\label{bias}
\end{align}
\end{proof}

\subsubsection{Proof of Theorem \ref{thm2}}

We first need the following basic inequalities.

\begin{lemma}\label{new-pred}
Let $n>0$ and $f:\B^{n+1}\rightarrow \mathbb{R}$. Also, let $X$ be a random element of $\B^n$ and $Y$ be a random element of $\B$ independent of $X$. Then
\[\sum_{s\subseteq[n]} (\E f(X,Y)-\E f(X,p_s(X)) )^2\le \E f^2(X,Y) \]
\end{lemma}

\begin{proof}
For each $x\in\B^n$, let $g(x)=f(x,1)-f(x,0)$.

\begin{align}
&\sum_{s\subseteq[n]} (\E[f(X,Y)]-\E[f(X,p_s(X))])^2\\
&=\sum_{s\subseteq[n]} \left(2^{-n-1}\sum_{x\in\B^n} (f(x,0)+f(x,1)-2f(x,p_s(x)))\right)^2\\
&=\sum_{s\subseteq[n]} \left(2^{-n-1}\sum_{x\in\B^n} g(x)(-1)^{p_s(x)}\right)^2   \label{pars1} \\
&=2^{-2n-2}\sum_{x_1,x_2\in\B^n,s\subseteq[n]}g(x_1)(-1)^{p_s(x_1)}\cdot g(x_2)(-1)^{p_s(x_2)}\\
&=2^{-2n-2}\sum_{x_1,x_2\in\B^n}g(x_1)g(x_2) \sum_{s\subseteq[n]}(-1)^{p_s(x_1)}(-1)^{p_s(x_2)}\\
&=2^{-2n-2}\sum_{x\in\B^n}2^n g^2(x) \label{pars2} \\
&= 2^{-n-2}\sum_{x\in\B^n} [f(x,1)-f(x,0)]^2  \\
&\le 2^{-n-1}\sum_{x\in\B^n} f^2(x,1)+f^2(x,0)\\
&=\E[f^2(X,Y)]
\end{align}
where we note that the equality from \eqref{pars1} to \eqref{pars2} is Parserval's identity for the Fourier-Walsh basis (here we used Boolean outputs for the parity functions). \end{proof}
Note that by the triangular inequality the above implies
\begin{align}
\Var_F \E_{X}f(X,F(X))\le 2^{-n} \E_{X,Y}f^2(X,Y).
\end{align}
As mentioned earlier, this is similar to Theorem 1 in \cite{ohad} that requires in addition the function to be the gradient of a 1-Lipschitz loss function.

We also mention the following corollary of Lemma \ref{new-pred} that results from Cauchy-Schwarz.

\begin{corollary}\label{parityAverage}
Let $n>0$ and $f:\B^{n+1}\rightarrow \mathbb{R}$. Also, let $X$ be a random element of $\B^n$ and $Y$ be a random element of $\B$ independent of $X$. Then
\[\sum_{s\subseteq[n]} |E[f((X,Y))]-E[f((X,p_s(X)))]|\le 2^{n/2}\sqrt{E[f^2((X,Y))]}.\]
\end{corollary}
In other words, the expected value of any function on an input generated by a random parity function is approximately the same as the expected value of the function on a true random input. 

\begin{proof}[Proof of Theorem \ref{thm2}]
We follow the proof of Theorem \ref{thm2'} until \eqref{break2}, 
where we use instead Lemma \ref{new-pred}, to write (for $m=\infty$)
\begin{align}
   & \E_{F}   (\E_{Z \sim D_{m,F}} g_e(Z) - \E_{Z \sim D_{m,\star}} g_e(Z) )^2 \le 2^{-n} \E_{Z \sim D_{m,\star}} g_e^2(Z)  
\end{align}
where $2^{-n}$ is the CP for parities. Thus in the case of parities, we can remove a factor of $1/2$ on the exponent of the CP. Further, the Cauchy-Schwartz inequality in \eqref{csjunk} is no longer needed, and the junk flow can be defined in terms of the sum of gradient norms, rather than taking norms squared and having a root on the sum; this does not however change the scaling of the junk flow. The theorem follows by choosing the .      
\end{proof}

\subsection{Proof of Theorem \ref{thm1'}}\label{sla}

\subsubsection{Learning from a bit}
We now consider the following setup:
\begin{align}
&(X,F) \sim P_\X \times P_\F  \label{r1} \\
&Y=F(X) \text{ (denote by $P_\Y$ the marginal of $Y$)}  \label{r2} \\
&W=g(X,Y) \text{ where $g: \X \times \Y \to \B$}  \label{r3} \\
&(\tX,\tY) \sim P_\X \times U_\Y \text{ (independent of $(X,F)$)}
\end{align}
That is, a random input $X$ and a random hypothesis $F$ are drawn from the working model, leading to an output label $Y$. 
We store a bit $W$ after observing the labelled pair $(X,Y)$. 
We are interested in estimating how much information can this bit contain about $F$, no matter how ``good'' the function $g$ is. 
We start by measuring the information using the variance of the MSE or  Chi-squared mutual information\footnote{The Chi-squared mutual information should normalize this expression with respect to the variance of $W$ for non equiprobable random variables.}, i.e., 
\begin{align}
I_2(W;F)=\Var \E (W|F) 
\end{align}
which gives a measure on how random $W$ is given $F$.
We provide below a bound in terms of the cross-predictability of $P_\F$ with respect to $P_\X$, and the marginal probability that $g$ takes value 1 on two independent inputs, which is a ``inherent bias'' of $g$.

The Chi-squared is convenient to analyze and is stronger than the classical mutual information, which is itself stronger than the squared total-variation distance by Pinsker's inequality. More precisely\footnote{See for example \cite{boix} for details on these inequalities.}, for an equiprobable $W$, 
\begin{align}
TV(W;F) \lesssim I(W;F)^{1/2} \le I_2(W;F)^{1/2}. 
\end{align}
Here we will need to obtain such inequalities for arbitrary marginal distributions of $W$ and in a self-contain series of lemmas. We then bound the latter with the cross-predictability which allows us to bound the error probability of the hypothesis test deciding whether $W$ is dependent on $F$ or not, which we later use in a more general framework where $W$ relates to the updated weights of the descent algorithm. We will next derive the bounds that are  needed.\footnote{These bounds could be slightly tightened but are largely sufficient for our purpose.}

\begin{lemma}\label{lemma_pred}
\begin{align}
 &\Var \E (g(X,Y)|F) \le 
 \E_F (\pp_X(g(X,F(X))=1)-\pp_{\tX,\tY}(g(\tX,\tY)=1))^2 \\ 
 &\le
 \min_{i \in \{0,1\}}   \pp\{ g(\tX,\tY) =i\} \sqrt{\mathrm{Pred}( P_\X, P_{\F} )}
\end{align}
\end{lemma}

\begin{proof}
Note that
\begin{align}
\Var \E (W|F) &=  \E_F (   \pp\{W=1|F\} - \pp\{W=1\}  )^2\\
&\le  \E_F (   \pp\{W=1|F\} -c  )^2
\end{align}
for any $c  \in \mR$.
Moreover,
\begin{align}
\pp\{W=1|F=f \} &= \sum_{x} \pp\{W=1|F=f, X=x \} P_\X(x)\\
&=\sum_{x,y} \pp\{W=1|X=x, Y=y \} P_\X(x) \1(f(x)=y) .
\end{align}
Pick now
\begin{align}
c:=\sum_{x,y} \pp\{W=1|X=x,Y=y \} P_\X(x) U_{\Y}(y) \label{py}
\end{align}

Therefore, 
\begin{align}
\pp\{W=1|F=f \} - c &=\sum_{x,y} A_g(x,y) B_f(x,y) =: \langle A_g,B_f \rangle
\end{align}
where 
\begin{align}
A_g(x,y):&=\pp\{W=1|X=x, Y=y \} \sqrt{P_{\X}(x)U_{\Y}(y)} \\ &= \pp\{g(X,Y)=1|X=x, Y=y \} \sqrt{P_{\X}(x)U_{\Y}(y)} \\
B_f(x,y):&=\frac{\1(f(x)=y) - U_{\Y}(y)}{U_{\Y}(y)} \sqrt{P_{\X}(x)U_{\Y}(y)}.
\end{align}
We have
\begin{align}
 \langle A_g,B_F \rangle^2 =  \langle A_g,B_F \rangle \langle B_F,A_g \rangle =   \langle A_g^{\otimes 2}, B_F^{\otimes 2} \rangle
\end{align}
and therefore 
\begin{align}
\E_F \langle A_g,B_F \rangle^2  &=   \langle A_g^{\otimes 2}, \E_F  B_F^{\otimes 2} \rangle \\
&\le \| A_g^{\otimes 2} \|_2 \| \E_F  B_F^{\otimes 2} \|_2.
\end{align}
Moreover, 
\begin{align}
 \| A_g^{\otimes 2} \|_2 &=  \| A_g \|_2^2 \\
&= \sum_{x,y} \pp\{W=1|X=x, Y=y \}^2 P_{\X}(x)U_{\Y}(y)\\
&\le \sum_{x,y} \pp\{W=1|X=x, Y=y \}  P_{\X}(x)U_{\Y}(y)\\
&= \pp\{ W(\tX,\tY)=1 \} 
\end{align}
and
\begin{align}
\| \E_F  B_F^{\otimes 2} \|_2 &= \left(\sum_{x,y,x',y'}  (\sum_{f} B_f(x,y) B_f(x',y') P_\F(f))^2 \right)^{1/2}\\
&=  \left( \E_{F,F'}  \langle B_F, B_{F'} \rangle^2  \right)^{1/2}.
\end{align}
Moreover, 
\begin{align}
\langle B_f, B_{f'} \rangle &= \sum_{x,y}  \frac{\1(f(x)=y) - U_{\Y}(y)}{U_{\Y}(y)} \frac{\1(f'(x)=y) -U_{\Y}(y)}{U_{\Y}(y)}  P_\X(x) U_{\Y}(y) \\
& =(1/2) \sum_{x,y}  (2 \1(f(x)=y) - 1)  (2 \1(f'(x)=y) - 1)  P_\X(x) \\
&=\E_X f(X)f'(X) .
\end{align}
Therefore, 
\begin{align}
 \Var \pp \{W=1| F\} \le \pp\{ \tW=1 \} \sqrt{\mathrm{Pred}( P_\X, P_{\F} )}.
\end{align}
The same expansion holds with $ \Var \pp \{W=1| F\}=\Var \pp \{W=0| F\} \le \pp\{ \tW=0 \} \sqrt{\mathrm{Pred}( P_\X, P_{\F} )}$.

\end{proof}

Consider now the new setup where $g$ is valued in $[m]$ instead of $\{0,1\}$:
\begin{align}
&(X,F) \sim P_\X \times P_\F \label{s1} \\
&Y=F(X)  \label{s2} \\
&W=g(X,Y) \text{ where $g: \B^n \times \Y \to [m]$}. \label{s3}
\end{align}
We have the following theorem. 
\begin{theorem} \label{corol_unif2}
\[E_F\|P_{W|F}-P_W\|_2^2\le\sqrt{\mathrm{Pred}( P_\X, P_{\F} )}\]
\end{theorem}

\begin{proof}
From Lemma \ref{lemma_pred}, for any $i \in [m]$, 
\begin{align}
\Var \pp \{W=i|F\} \le \pp\{ g(\tX,\tY)=i\} \sqrt{\mathrm{Pred}( P_\X, P_{\F} )},
\end{align}
therefore, 
\begin{align}
E_F\|P_{W|F}-P_W\|_2^2  &= \sum_{i \in [m]} \sum_{f \in F} \pp\{F=f\} (\pp\{W=i|F=f\} -\pp\{W=i\})^2 \\
&\le\sum_{i \in [m]} \pp\{ g(\tX,\tY)=i\} \sqrt{\mathrm{Pred}( P_\X, P_{\F} )}\\\
& =\sqrt{\mathrm{Pred}( P_\X, P_{\F} )}.
\end{align}
\end{proof}

\begin{corollary}\label{thm_pred}
\begin{align}
\| P_{W,F} - P_W P_F \|_2^2  &\le \| P_\F \|_\infty   \sqrt{\mathrm{Pred}( P_\X, P_{\F} )}.
\end{align}
\end{corollary}
We next specialize the bound in Theorem \ref{thm_pred} to the case of uniform parity functions on uniform inputs, adding a bound on the $L_1$ norm due to Cauchy-Schwarz.

\begin{corollary}\label{}
Let $m,n >0$. If we consider the setup of \eqref{s1},\eqref{s2},\eqref{s3} for the case where $P_\F=P_n$, the uniform probability measure on parity functions, and $P_\X=U_n$, the uniform probability measure on $\B^n$, then 
\begin{align}
& \| P_{W,F} - P_W P_F \|_2^2 \le 2^{-(3/2)n},\\
& \| P_{W,F} - P_W P_F \|_1 \le \sqrt{m}2^{-n/4}.
\end{align}
\end{corollary}
In short, the value of $W$ will not provide significant amounts of information on $F$ unless its number of possible values $m$ is exponentially large.

\begin{corollary}\label{corol_unif}
Consider the same setup as in previous corollary, with in addition $(\tX,\tY)$ independent of $(X,F)$ such that $(\tX,\tY) \sim P_\X \times U_\Y$ where $U_\Y$ is the uniform distribution on $\Y$, and $\tW=g(\tX,\tY)$. Then,
\[\sum_{i\in[m]} \sum_{s\subseteq [n]} (P[W=i|f=p_s]-P[\tW=i])^2\le 2^{n/2}.\]
\end{corollary}

\begin{proof}
In the case where $P_{\F}=P_n$, taking the previous corollary and multiplying both sides by $2^{2n}$ yields
\[\sum_{i\in[m]} \sum_{s\subseteq [n]} (P[W=i|f=p_s]-P[W=i])^2\le 2^{n/2}.\]
Furthermore, the probability distribution of $(X,Y)$ and the probability distribution of $(\tX,\tY)$ are both $U_{n+1}$ so $P[\tW=i]=P[W=i]$ for all $i$. Thus,
\begin{align}\sum_{i\in[m]} \sum_{s\subseteq [n]} (P[W=i|f=p_s]-P[\tW=i])^2\le 2^{n/2}. \label{lastw} \end{align}
\end{proof}
Notice that for fixed values of $P_{\X}$ and $g$, changing the value of $P_{\F}$ does not change the value of $P[W=i|f=p_s]$ for any $i$ and $s$. Therefore, inequality \eqref{lastw} holds for any choice of $P_{\F}$, and we also have the  following.

\begin{corollary}
Consider the general setup of \eqref{s1},\eqref{s2},\eqref{s3} with $P_\X=U_n$, and $(\tX,\tY)$ independent of $(X,F)$ such that $(\tX,\tY) \sim P_\X \times U_\Y$, $\tW=g(\tX,\tY)$.
Then,
\[\sum_{i\in[m]} \sum_{s\subseteq [n]} (P[W=i|f=p_s]-P[\tW=i])^2\le 2^{n/2}.\]
\end{corollary}

\subsubsection{Distinguishing with SLAs}
Next, we would like to analyze the effectiveness of an algorithm that repeatedly receives an ordered pair, $(X,F(X))$, records some amount of information about that pair, and then forgets it. We recall the definition of an SLA that formalizes this.

\begin{definition}   
A sequential learning algorithm $A$ on $(\mathcal{Z}, \mathcal{W})$ is an algorithm that for an input of the form $(Z,(W_1,...,W_{t-1}))$ in $\mathcal{Z} \times \mathcal{W}^{t-1}$ produces an output $A(Z,(W_1,...,W_{t-1}))$ valued in $\mathcal{W}$. Given a probability distribution $D$ on $\mathcal{Z}$, a sequential learning algorithm $A$ on $(\mathcal{Z}, \mathcal{W})$, and $T\ge 1$, a $T$-trace of $A$ for $D$ is a series of pairs $((Z_1, W_1), ...,(Z_T, W_T))$ such that for each $i \in [T]$, $Z_i\sim D$ independently of $(Z_1,Z_2,...,Z_{i-1})$ and $W_i=A(Z_i,(W_1,W_2,...,W_{i-1}))$.
\end{definition}

If $|\mathcal{W}|$ is sufficiently small relative to $\mathrm{Pred}( P_\X, P_{\F} )$, then a sequential learning algorithm that outputs elements of $\mathcal{W}$ will be unable to effectively distinguish between a random function from $P_{\F}$ and a true random function in the following sense.

\begin{theorem} \label{SLAfail} 
Let $n>0$, $A$ be a sequential learning algorithm on $(\B^{n+1},\mathcal{W})$, $P_{\X}$ be the uniform distribution on $\B^n$, and $P_{\F}$ be a probability distribution on functions from $\B^n$ to $\B$. Let $\star$ be the probability distribution of $(X,F(X))$ when $F\sim P_{\F}$ and $X\sim P_{\X}$. Also, for each $f:\B^n\to\B$, let let $\rho_f$ be the probability distribution of $(X,f(X))$ when $X\sim P_{\X}$. Next, let $P_{\mathcal{Z}}$ be a probability distribution on $\B^{n+1}$ that is chosen by means of the following procedure: with probability $1/2$, set $P_{\mathcal{Z}}=\star$, otherwise draw $F\sim P_{\F}$ and set $P_{\mathcal{Z}}=\rho_F$. If $|\mathcal{W}|\le 1/\sqrt[24]{\mathrm{Pred}( P_\X, P_{\F} )}$, $m$ is a positive integer with $m<1/\sqrt[24]{\mathrm{Pred}( P_\X, P_{\F} )}$, and $((Z_1,W_1),...,(Z_m,W_m))$ is a $m$-trace of $A$ for $P_{\mathcal{Z}}$, then
\begin{align}
\| P_{W^m|P_{\mathcal{Z}}=\star} - P_{W^m|P_{\mathcal{Z}}\ne \star} \|_{1} =  O(\sqrt[24]{\mathrm{Pred}( P_\X, P_{\F} )}).
\end{align}
\end{theorem}

\begin{proof} 
First of all, let $q=\sqrt[24]{\mathrm{Pred}( P_\X, P_{\F} )}$ and $F'\sim P_{\F}$. Note that by the triangular inequality,
\begin{align*}
&\| P_{W^m|P_{\mathcal{Z}}=\star} - P_{W^m|P_{\mathcal{Z}}\ne \star} \|_{1} \\
&=\sum_{w_1,...,w_m\in\mathcal{W}} |P[W^m=w^m|P_{\mathcal{Z}}  \ne \star]-P[W^m=w^m|P_{\mathcal{Z}}=\star]| \\
&\le \sum_{f:\B^n\to\B}P[F=f]\sum_{w^m\in\mathcal{W}^m} |P[W^m=w^m|P_{\mathcal{Z}}=\rho_s]-P[W^m=w^m|P_{\mathcal{Z}}=\star]|
\end{align*}
and we will bound the last term by $O(q)$.

We need to prove that $P[W^m=w^m|P_{\mathcal{Z}}=\rho_f]\approx P[W^m=w^m|P_{\mathcal{Z}}=\star]$ most of the time. In order to do that, we will use the fact that
\[\frac{P[W^m=w^m|P_{\mathcal{Z}}=\rho_f]}{P[W^m=w^m|P_{\mathcal{Z}}=\star]}=\prod_{i=1}^m \frac{P[W_i=w_i|W^{i-1}=w^{i-1},P_{\mathcal{Z}}=\rho_f]}{P[W_i=w_i|W^{i-1}=w^{i-1},P_{\mathcal{Z}}=\star]}\]
So, as long as $P[W_i=w_i|W^{i-1}=w^{i-1},P_{\mathcal{Z}}=\rho_f]\approx P[W_i=w_i|W^{i-1}=w^{i-1},P_{\mathcal{Z}}=\star]$ and $P[W_i=w_i|W^{i-1}=w^{i-1},P_{\mathcal{Z}}=\star]$ is reasonably large for all $i$, this must hold for the values of $w^m$ and $f$ in question. As such, we plan to define a good value for $(w^m,f)$ to be one for which this holds, and then prove that the set of good values has high probability measure.

First, call a sequence $w^m\in \mathcal{W}^m$ {\it typical} if for each $1\le i\le m$, we have that 
\[t(w^i):= P[W_i=w_i|W^{i-1}=w^{i-1},P_{\mathcal{Z}}=\star]\ge q^3,\]
and denote by $\mathcal{T}$ the set of typical sequences
\begin{align}
\mathcal{T} &:= \{ w^m: \forall i \in [m], t(w^i)\ge q^3 \}.
\end{align}

We have
\begin{align}
1 &=  \pp\{ W^m \in \mathcal{T} | P_{\mathcal{Z}}=\star \}  +   \pp\{ W^m \notin \mathcal{T} | P_{\mathcal{Z}}=\star\}\\
&\le \pp\{ W^m \in \mathcal{T}| P_{\mathcal{Z}}=\star \}  + \sum_{i=1}^m  \pp\{  t(W^i) <  q^3| P_{\mathcal{Z}}=\star \} \\
&\le \pp\{ W^m \in \mathcal{T}| P_{\mathcal{Z}}=\star \}  +  m q^3 |\mathcal{W}|.
\end{align}
Thus
\begin{align}
\pp\{ W^m \in \mathcal{T}| P_{\mathcal{Z}}=\star \}&\ge 1-  m q^3 |\mathcal{W}| \ge 1-q.
\end{align}

Next, call an ordered pair of a sequence $w^m\in \mathcal{W}^m$ and an $f:\B^n\to\B$ {\it good} if $w^m$ is typical and
\begin{align}
\left| \frac{P[W_i=w_i|W^{i-1}=w^{i-1},P_{\mathcal{Z}}=\rho_f]}{P[W_i=w_i|W^{i-1}=w^{i-1},P_{\mathcal{Z}}=\star]}-1\right| 
\le q^2  , \quad  \forall i  \in [m],
\end{align}
and denote by $\mathcal{G}$ the set of good pairs. A pair which is not good is called bad.

Note that for any $i$ and any $w_1,...,w_{i-1}\in \mathcal{W}$, there exists a function $g_{w_1,...,w_{i-1}}$ such that $W_i=g_{w_1,...,w_{i-1}}(Z_i)$. So, theorem \ref{corol_unif2} implies that
\begin{align}
&\sum_{w_i\in\mathcal{W}}\sum_{f:\B^n\to\B} P[F'=f](P[W_i=w_i|W^{i-1}=w^{i-1},P_{\mathcal{Z}}=\rho_f]-P[W_i=w_i|W^{i-1}=w^{i-1},P_{\mathcal{Z}}=\star])^2  \\
&=\sum_{w_i\in\mathcal{W}}\sum_{f:\B^n\to\B} P[F'=f] (P[g_{w_1,...,w_{i-1}}(Z_i)=w_i|P_{\mathcal{Z}}=\rho_f]-P[g_{w_1,...,w_{i-1}}(Z_i)=w_i|P_{\mathcal{Z}}=\star])^2  \\
&\le q^{12}
\end{align}

Also, given any $w^m$ and $f:\B^{n}\to\B$ such that $w^m$ is typical but $w^m$ and $f$ are not good, there must exist $1\le i\le m$ such that 
\begin{align}
r(w^i,f)&:=|P[W_i=w_i|W^{i-1}=w^{i-1},P_{\mathcal{Z}}=\rho_f]-P[W_i=w_i|W^{i-1}=w^{i-1},P_{\mathcal{Z}}=\star]|\\
&\ge q^5.
\end{align}
Thus, for $w^m \in \mathcal{T}$
\begin{align}
\sum_{f: (w^m,f) \notin  \mathcal{G} } P[F'=f] &=\pp\{ (w^m,F') \notin  \mathcal{G} \} \\
&\le \pp\{ \exists i \in [m]: r(w^i,F') \ge q^5 \}\\
&\le \sum_{i=1}^m \sum_{f: r(w^i,f) \ge q^5} P[F'=f]\\
&\le  q^{-10} \sum_{i=1}^m \sum_{f:\B^n\to\B} P[F'=f]\cdot r(w^i,f)^2   \\
&\le q^{-10} \sum_{i=1}^m \sum_{w_i' \in \mathcal{W}}  \sum_{f:\B^n\to \B} P[F'=f]\cdot r((w_i',w^{i-1}),f)^2   \\
&\le q^{-10} m \cdot q^{12}  \\
\end{align}
This means that for a given typical $w^m$, the probability that $w^m$ and $F'$ are not good is at most $m q^2\le q$.

Therefore, if $P_{\mathcal{Z}}=\star$, the probability that $W^m$ is typical but $W^m$ and $F'$ is not good is at most $q$; in fact:
\begin{align}
& \pp\{ W^m \in \mathcal{T} , (W^m,F') \notin  \mathcal{G} | P_{\mathcal{Z}}=\star \}\\
& = \sum_{f,  w^m \in \mathcal{T} : (w^m,s) \notin  \mathcal{G} } \pp\{F'=f\}\cdot \pp\{W^m=w^m | P_{\mathcal{Z}}=\star \}  \\
& =\sum_{ w^m \in \mathcal{T}} \pp\{W^m=w^m | P_{\mathcal{Z}}=\star \} \sum_{f : (w^m,s) \notin  \mathcal{G} }  \pp\{F'=f\}  \\
& \le q \sum_{ w^m \in \mathcal{T}} \pp\{W^m=w^m | P_{\mathcal{Z}}=\star \} \\
& \le q.
\end{align}

We already knew that $W^m$ is typical with probability $1-q$ under these circumstances, so $W^m$ and $S$ is good with probability at least $1-2 q$ since
\begin{align}
&1-q  \le \pp\{ W^m \in \mathcal{T}  | P_{\mathcal{Z}}=\star \}\\
& = \pp\{ W^m \in \mathcal{T} , (W^m,F') \in  \mathcal{G} | P_{\mathcal{Z}}=\star \}  +   \pp\{ W^m \in \mathcal{T} , (W^m,F') \notin  \mathcal{G} | P_{\mathcal{Z}}=\star \}  \\
&\le \pp\{ (W^m,F') \in  \mathcal{G} | P_{\mathcal{Z}}=\star \} + q.
\end{align}

Next, recall that 
\[\frac{P[W^m=w^m|P_{\mathcal{Z}}=\rho_f]}{P[W^m=w^m|P_{\mathcal{Z}}=\star]}=\prod_{i=1}^m \frac{P[W_i=w_i|W^{i-1}=w^{i-1},P_{\mathcal{Z}}=\rho_f]}{P[W_i=w_i|W^{i-1}=w^{i-1},P_{\mathcal{Z}}=\star]}\]
So, if $w^m$ and $f$ is good (and thus each term in the above product is within $q^2$ of 1), we have 
\begin{align}
&\left| \frac{P[W^m=w^m|P_{\mathcal{Z}}=\rho_f]}{P[W^m=w^m|P_{\mathcal{Z}}=\star]}-1\right|  \le e^{q}-1 =O(q).
\end{align}
That implies that
\begin{align*}
&\sum_{(w^m,f) \in \mathcal{G}}P[F'=f]\cdot  |P[W^m=w^m|P_{\mathcal{Z}}=\rho_f]-P[W^m=w^m|P_{\mathcal{Z}}=\star]|\\
&\le \sum_{(w^m,f) \in \mathcal{G}} P[F'=f]\cdot O(q)\cdot P[W^m=w^m|P_{\mathcal{Z}}=\star]\\
&\le \sum_{w^m} O(q)\cdot P[W^m=w^m|P_{\mathcal{Z}}=\star]\\
&=O(q).
\end{align*}

Also,
\begin{align*}
&\sum_{(w^m,f) \notin \mathcal{G}} P[F'=f]\cdot (P[W^m=w^m|P_{\mathcal{Z}}=\rho_f]-P[W^m=w^m|P_{\mathcal{Z}}=\star])\\
&= P[(W^m,F') \notin \mathcal{G}|P_{\mathcal{Z}}\ne\star]-P[(W^m,F') \notin \mathcal{G}|P_{\mathcal{Z}}=\star]\\
&= P[(W^m,F') \in \mathcal{G}|P_{\mathcal{Z}}=\star]-P[W^m,F') \in \mathcal{G}|P_{\mathcal{Z}}\ne\star]\\
&= \sum_{(w^m,f) \in \mathcal{G}} P[F'=f]\cdot (P[W^m=w^m|P_{\mathcal{Z}}=\star]-P[W^m=w^m|P_{\mathcal{Z}}=\rho_f]) \\
&\le \sum_{(w^m,f) \in \mathcal{G}} P[F'=f]\cdot |P[W^m=w^m|P_{\mathcal{Z}}=\star]-P[W^m=w^m|P_{\mathcal{Z}}=\rho_f]| \\
&=O(q).
\end{align*}
That means that
\begin{align*}
&\sum_{(w^m,f) \notin \mathcal{G}} P[F'=f]\cdot  |P[W^m=w^m|P_{\mathcal{Z}}=\rho_f]-P[W^m=w^m|P_{\mathcal{Z}}=\star]|\\
&\le \sum_{(w^m,f) \notin \mathcal{G}} P[F'=f]\cdot  (P[W^m=w^m|P_{\mathcal{Z}}=\rho_f]+P[W^m=w^m|P_{\mathcal{Z}}=\star])\\
&= \sum_{(w^m,f) \notin \mathcal{G}} P[F'=f]\cdot 2P[W^m=w^m|P_{\mathcal{Z}}=\star]\\
&\qquad\qquad +\sum_{(w^m,f) \notin \mathcal{G}} P[F'=f]\cdot (P[W^m=w^m|P_{\mathcal{Z}}=\rho_f]-P[W^m=w^m|P_{\mathcal{Z}}=\star])\\
&=O(q).
\end{align*}
Therefore,
\begin{align}&\sum_{f:\B^n\to\B}\sum_{w^m\in\mathcal{W}} P[F'=f]\cdot |P[W^m=w^m|P_{\mathcal{Z}}=\rho_s]-P[W^m=w^m|P_{\mathcal{Z}}=\star]|=O(q),\end{align}
which gives the desired bound. 
\end{proof}

\begin{corollary} $\label{memLimit}$
Consider a data structure with a polynomial amount of memory that is divided into variables that are each $O(\log n)$ bits long, and define $m$, $\mathcal{Z}$, $\star$, and $P_{\mathcal{Z}}$ the same way as in Theorem \ref{SLAfail}. Also, let $A$ be an algorithm that takes the data structure's current value and an element of $\B^{n+1}$ as inputs and changes the values of at most $o(-\log(\mathrm{Pred}( P_\X, P_{\F} ))/\log(n))$ of the variables. If we draw $Z_1,...,Z_m$ independently from $P_{\mathcal{Z}}$ and then run the algorithm on each of them in sequence, then no matter how the data structure is initialized, it is impossible to determine whether or not $P_{\mathcal{Z}}=\star$ from the data structure's final value with accuracy greater than $1/2+O(\sqrt[24]{\mathrm{Pred}( P_\X, P_{\F} )})$.
\end{corollary}

\begin{proof}
Let $q=1/\sqrt[24]{\mathrm{Pred}( P_\X, P_{\F} )}$. Let $W_0$ be the initial state of the data structure's memory, and let $W_i=A(W_{i-1}, Z_i)$ for each $0<i\le m$. Next, for each such $i$, let $W'_i$ be the list of all variables that have different values in $W_i$ than in $W_{i-1}$, and their values in $W_i$. There are only polynomially many variables in memory, so it takes $O(\log(n))$ bits to specify one and $O(\log(n))$ bits to specify a value for that variable. $A$ only changes the values of $o(\log(q)/\log(n))$ variables at each timestep, so $W'_i$ will only ever list $o(\log(q)/\log(n))$ variables. That means that $W'_i$ can be specified with $o(\log(q))$ bits, and in particular that there exists some set $\mathcal{W}$ such that $W'_i$ will always be in $\mathcal{W}$ and $|\mathcal{W}|=2^{o(\log(q))}$. Also, note that we can determine the value of $W_i$ from the values of $W_{i-1}$ and $W'_i$, so we can reconstruct the value of $W_i$ from the values of $W'_1,W'_2,...,W'_i$. 

Now, let $A'$ be the algorithm that takes $(Z_t, (W'_1,...,W'_{t-1}))$ as input and does the following. First, it reconstructs $W_{t-1}$ from $(W'_1,...,W'_{t-1})$. Then, it computes $W_t$ by running $A$ on $W_{t-1}$ and $Z_t$. Finally, it determines the value of $W'_t$ by comparing $W_t$ to $W_{t-1}$ and returns it. This is an SLA, and $((Z_1,W'_1),...,(Z_m,W'_m))$ is an $m$-trace of $A'$ for $P_{\mathcal{Z}}$. So, by the theorem

\begin{align}&  \sum_{w_1,...,w_m} |P[W'_1=w_1,...,W'_m=w_m|P_{\mathcal{Z}}\ne \star ]-P[W'_1=w_1,...,W'_m=w_m|P_{\mathcal{Z}}=\star]|\\
&=O(1/q)\end{align}

Furthermore, since $W_m$ can be reconstructed from $(W'_1,...,W'_m)$, this implies that
\begin{align}
\sum_{w} |P[W_m=w|P_{\mathcal{Z}}\ne \star]-P[W_m=w|P_{\mathcal{Z}}=\star]|=O(1/q).
\end{align}
Finally, the probability of deciding correctly between the hypothesis $P_\Z=\star$ and $P_\Z\ne \star$ given the observation $W_m$ is at most 
\begin{align}
&1-\frac{1}{2}\sum_{w \in \mathcal{W} }  P[W_m=w|P_{\mathcal{Z}}= \star] \wedge  P[W_m=w|P_{\mathcal{Z}}\ne \star] \\
& = \frac{1}{2} + \frac{1}{4} \sum_{w \in \mathcal{W} } |P[W_m=w|P_{\mathcal{Z}}\ne \star]-P[W_m=w|P_{\mathcal{Z}}=\star]|\\
& = \frac{1}{2}  + O(1/q),
\end{align}
which implies the conclusion. 

\end{proof}

\begin{remark}
The theorem and its second corollary state that the algorithm can not determine whether or not $P_{\mathcal{Z}}=\star$. However, one could easily transform them into results showing that the algorithm can not effectively learn to compute $f$. More precisely, after running on $q/2$ pairs $(x,p_f(x))$, the algorithm will not be able to compute $p_s(x)$ with accuracy $1/2+\omega(1/\sqrt{q})$ with a probability of $\omega(1/q)$. If it could, then we could just train it on the first $m/2$ of the $Z_i$ and count how many of the next $m/2$ $Z_i$ it predicts the last bit of correctly. If $P_{\mathcal{Z}}=\star$, each of those predictions will be independently correct with probability $1/2$, so the total number it is right on will  differ from $m/4$ by $O(\sqrt{m})$ with high probability. However, if $P_{\mathcal{Z}}=\rho_f$ and the algorithm learns to compute $\rho_f$ with accuracy $1/2+\omega(1/\sqrt{q})$, then it will predict $m/4+\omega(\sqrt{m})$ of the last $m/2$ correctly with high probability. So, we could determine whether or not $P_{\mathcal{Z}}=\star$ with greater accuracy than the theorem allows by tracking the accuracy of the algorithm's predictions.
\end{remark}

\subsubsection{Application to SGD}
 
One possible variant of this is to only adjust a few weights at each time step, such as the $k$ that would change the most or a random subset. However, any such algorithm cannot learn a random parity function in the following sense.

\begin{theorem}
Let $n>0$, $k=o(n/\log(n))$, and $(f,g)$ be a neural net of size polynomial in $n$ in which each edge weight is recorded using $O(\log n)$ bits. Also, let $\star$ be the uniform distribution on $\B^{n+1}$, and for each $s\subseteq [n]$, let $\rho_s$ be the probability distribution of $(X,p_s(X))$ when $X$ is chosen randomly from $\B^n$. Next, let $P_{\mathcal{Z}}$ be a probability distribution on $\B^{n+1}$ that is chosen by means of the following procedure. First, with probability $1/2$, set $P_{\mathcal{Z}}=\star$. Otherwise, select a random $S\subseteq[n]$ and set $P_{\mathcal{Z}}=\rho_S$. Then, let $A$ be an algorithm that draws a random element from $P_{\mathcal{Z}}$ in each time step and changes at most $k$ of the weights of $g$ in response to the sample and its current values. If $A$ is run for less than $2^{n/24}$ time steps, then it is impossible to determine whether or not $P_{\mathcal{Z}}=\star$ from the resulting neural net with accuracy greater than $1/2+O(2^{-n/24})$.
\end{theorem}

\begin{proof}
This follows immediately from corollary \ref{memLimit}.
\end{proof}

\begin{remark}
The theorem state that one cannot determine whether or not $P_{\mathcal{Z}}=\star$ from the final network. However, if we used a variant of corollary \ref{memLimit} we could get a result showing that the final network will not compute $p_s$ accurately. More precisely, after training the network on $2^{n/24-1}$ pairs $(x,p_s(x))$, the network will not be able to compute $p_s(x)$ with accuracy $1/2+\omega(2^{-n/48})$ with a probability of $\omega(2^{-n/24})$. 
\end{remark}

We can also use this reasoning to prove theorem \ref{thm1'}, which is restated below.

\begin{theorem}
Let $\epsilon>0$, and $P_{\F}$ be a probability distribution over functions with a cross-predictability of $\mathrm{c_p}=o(1)$. For each $n > 0$, let $(f,g)$ be a neural net of polynomial size in $n$ such that each edge weight is recorded using $O(\log(n))$ bits of memory. Run stochastic gradient descent on $(f,g)$ with at most $\mathrm{c_p}^{-1/24}$ time steps and with $o(|\log(\mathrm{c_p})|/\log(n))$ edge weights updated per time step.  For all sufficiently large $n$, this algorithm fails at learning functions drawn from $P_{\F}$ with accuracy $1/2 + \epsilon$.
\end{theorem}

\begin{proof}
Consider a data structure that consists of a neural net $(f,g')$ and a boolean value $b$. Now, consider training $(f,g)$ with any such coordinate descent algorithm while using the data structure to store the current value of the net. Also, in each time step, set $b$ to $True$ if the net computed the output corresponding to the sampled input correctly and $False$ otherwise. This constitutes a data structure with a polynomial amount of memory that is divided into variables that are $O(\log n)$ bits long, such that $o(|\log(\mathrm{c_p})|/\log(n))$ variables change value in each time step. As such, by corollary \ref{memLimit}, one cannot determine whether the samples are actually generated by a random parity function or whether they are simply random elements of $\B^{n+1}$ from the data structure's final value with accuracy $1/2+\omega(\mathrm{c_p}^{1/24})$. In particular, one cannot determine which case holds from the final value of $b$. If the samples were generated randomly, it would compute the final output correctly with probability $1/2$, so $b$ would be equally likely to be $True$ or $False$. So, when it is trained on a random parity function, the probability that $b$ ends up being $True$ must be at most $1/2+O(\mathrm{c_p}^{1/24})$. Therefore, it must compute the final output correctly with probability $1/2+O(\mathrm{c_p}^{1/24})$.
\end{proof}

\subsection{Proof of Theorem \ref{thm3}}

Our next goal is to make a similar argument for stochastic gradient descent. We argue that if we use noisy SGD to train a neural net on a random parity function, the probability distribution of the resulting net is similar to the probability distribution of the net we would get if we trained it on random values in $\B^{n+1}$. This will be significantly harder to prove than in the case of noisy gradient descent, because while the difference in the expected gradients is exponentially small, the gradient at a given sample may not be. As such, drowning out the signal will require much more noise. However, before we get into the details, we will need to formally define a noisy version of SGD, which is as follows.

\vspace{1 cm}
\noindent
{\em NoisySampleGradientDescentStep(f, G, $Y$, $X$, L, $\gamma$, B, $\delta$)}:
\begin{enumerate}
\item For each $(v,v')\in E(G)$: \begin{enumerate}
\item Set 
\[w'_{v,v'}=w_{v,v'}-\gamma \frac{\partial L(eval_{(f,G)}(X)-Y)}{\partial w_{v,v'}}+\delta_{v,v'}\]

\item If $w'_{v,v'}<-B$, set $w'_{v,v'}=-B$.

\item If $w'_{v,v'}>B$, set $w'_{v,v'}=B$.

\end{enumerate}

\item Return the graph that is identical to $G$ except that its edge weight are given by the $w'$.
\end{enumerate}

\vspace{1 cm}
\noindent
{\em NoisyStochasticGradientDescentAlgorithm(f, G, $P_\mathcal{Z}$, L, $\gamma$, B, $\Delta$, t)}:
\begin{enumerate}
\item Set $G_0=G$.

\item If any of the edge weights in $G_0$ are less than $-B$, set all such weights to $-B$.

\item If any of the edge weights in $G_0$ are greater than $B$, set all such weights to $B$.

\item For each $0\le i<t$:
\begin{enumerate}

\item Draw $(X_i, Y_i)\sim P_\mathcal{Z}$, independently of all previous values.

\item Generate $\delta^{(i)}$ by independently drawing $\delta^{(i)}_{v,v'}$ from $\Delta$ for each $(v,v')\in E(G)$.

\item Set $G_{i+1}=NoisySampleGradientDescentStep(f, G_i, Y_i, X_i, L, \gamma, B, \delta^{(i)})$
\end{enumerate}

\item Return $G_t$.
\end{enumerate}

\vspace{1 cm}
\noindent
{\em PerturbedStochasticGradientDescentAlgorithm(f, G, $P_\mathcal{Z}$, L, $\gamma$, $\delta$, t)}:
\begin{enumerate}
\item Set $G_0=G$.

\item For each $0\le i<t$:
\begin{enumerate}

\item Draw $(X_i, Y_i)\sim P_\mathcal{Z}$, independently of all previous values.

\item Set $G_{i+1}=NoisySampleGradientDescentStep(f, G_i, Y_i, X_i, L, \gamma, \infty, \delta_i)$
\end{enumerate}

\item Return $G_t$.
\end{enumerate}

\subsubsection{Uniform noise and SLAs}

The simplest way to add noise in order to impede learning a parity function would be to add noise drawn from a uniform distribution in order to drown out the information provided by the changes in edge weights. More precisely, consider setting $\Delta$ equal to the uniform distribution on $[-C,C]$. If the change in each edge weight prior to including the noise always has an absolute value less than $D$ for some $D<C$, then with probability $\frac{C-D}{C}$, the change in a given edge weight including noise will be in $[-(C-D),C-D]$. Furthermore, any value in this range is equally likely to occur regardless of what the change in weight was prior to the noise term, which means that the edge's new weight provides no information on the sample used in that step. If $D/C=o(nE(G)/\ln(n))$ then this will result in there being $o(n/\log(n))$ changes in weight that provide any relevant information in each timestep. So, the resulting algorithm will not be able to learn the parity function by an extension of corollary $\ref{memLimit}$. This leads to the following result:

\begin{theorem}
Let $n>0$, $\gamma>0$, $D>0$, $t=2^{o(n)}$, $(f,G)$ be a normal\footnote{We say that $(f,G)$ is normal if $f$ is a smooth function, the derivative of $f$ is positive everywhere, the derivative of $f$ is bounded, $\lim_{x\to -\infty} f(x)=0$, $\lim_{x\to \infty} f(x)=1$, and $G$ has an edge from the constant vertex to every other vertex except the input vertices.} neural net of size polynomial in $n$, and $L:\mathbb{R}\rightarrow\mathbb{R}$ be a smooth, convex, symmetric function with $L(0)=0$. Also, let $\Delta$ be the uniform probability distribution on $[-D|E(G)|,D|E(G)|]$. Now, let $S$ be a random subset of $[n]$ and $P_\mathcal{Z}$ be the probability distribution $(X,p_S(X))$ when $X$ is drawn randomly from $\B^n$. Then when NoisyStochasticGradientDescentAlgorithm(f, G, $P_\mathcal{Z}$, L, $\gamma$, $\infty$, $\Delta$, t) is run on a computer that uses $O(\log(n))$ bits to store each edge's weight, with probability $1-o(1)$ either there is at least one step when the adjustment to one of the weights prior to the noise term has absolute value greater than $D$ or the resulting neural net fails to compute $p_S$ with nontrivial accuracy.
\end{theorem}

This is a side result and we provide a concise proof. 
\begin{proof}

Consider the following attempt to simulate NoisyStochasticGradientDescentAlgorithm(f, G, $P_\mathcal{Z}$, L, $\gamma$, $\infty$, $\Delta$, t) with a sequential learning algorithm. First, independently draw $b^{t'}_{v,v'}$ from the uniform probability distribution on $[-D|E(G)|+D,D|E(G)|-D]$ for each $(v,v')\in E(G)$ and $t'\le t$. Next, simulate NoisyStochasticGradientDescentAlgorithm(f, G, $P_\mathcal{Z}$, L, $\gamma$, $\infty$, $\Delta$, t) with the following modifications. If there is ever a step where one of the adjustments to the weights before the noise term is added in is greater than $D$, record ``failure" and give up. If there is ever a step where more than $n/\ln^2(n)$ of the weights change by more than $D|E(G)|-D$ after including the noise record "failure" and give up. Otherwise, record a list of which weights changed by more than $D|E(G)|-D$ and exactly what they changed by. In all subsequent steps, assume that $W_{v,v'}$ increased by $b^{t'}_{v,v'}$ in step $t'$ unless the amount it changed by in that step is recorded.

First, note that if the values of $b$ are computed in advance, the rest of this algorithm is a sequential learning algorithm that records $O(n/\log(n))$ bits of information per step and runs for a subexponential number of steps. As such, any attempt to compute $p_S(X)$ based on the information provided by its records will have accuracy $1/2+o(1)$ with probability $1-o(1)$. Next, observe that in  a given step in which all of the adjustments to weights before the noise is added in are at most $D$, each weight has a probability of changing by more than $D|E(G)|-D$ of at most $1/|E(G)|$ and these probabilities are independent. As such, with probability $1-o(1)$, the algorithm will not record "failure" as a result of more than $n/\ln^2(n)$ of the weights changing by more than $D|E(G)|-D$. Furthermore, the probability distribution of the change in the weight of a given vertex conditioned on the assumption that said change is at most $D|E(G)|-D$ and a fixed value of said change prior to the inclusion of the noise term that has an absolute value of at most $D$ is the uniform probability distribution on $[-D|E(G)|+D,D|E(G)|-D]$. As such, substituting the values of $b^{t'}_{v,v'}$ for the actual changes in weights that change by less than $D|E(G)|-D$ has no effect on the probability distribution of the resulting graph. As such, the probability distribution of the network resulting from NoisyStochasticGradientDescentAlgorithm(f, G, $P_\mathcal{Z}$, L, $\gamma$, $\infty$, $\Delta$, t) if none of the weights change by more than $D$ before noise is factored in differs from the probabiliy distribution of the network generated by this algorithm if it suceeds by $o(1)$. Thus, the fact that the SLA cannot generate a network that computed $p_S$ with nontrivial accuracy implies that NoisyStochasticGradientDescentAlgorithm(f, G, $P_\mathcal{Z}$, L, $\gamma$, $\infty$, $\Delta$, t) also fails to generate a network that computes $p_S$ with nontrivial accuracy.
\end{proof}

\begin{remark}
At first glance, the amount of noise required by this theorem is ridiculously large, as it will almost always be the dominant contribution to the change in any weight in any given step. However, since the noise is random it will tend to largely cancel out over a longer period of time. As such, the result of this noisy version of stochastic gradient descent will tend to be similar to the result of regular stochastic gradient descent if the learning rate is small enough. In particular, this form of noisy gradient descent will be able to learn to compute most reasonable functions with nontrivial accuracy for most sets of starting weights, and it will be able to learn to compute some functions with nearly optimal accuracy. Admittedly, it still requires a learning rate that is smaller than anything people are likely to use in practice.
\end{remark}
We next move to handling lower levels of noise. 

\subsubsection{Gaussian noise, noise accumulation, and blurring}
While the previous result works, it requires more noise than we would really like. The biggest problem with it is that it ultimately argues that even given a complete list of the changes in all edge weights at each time step, there is no way to determine the parity function with nontrivial accuracy, and this requires a lot of noise. However, in order to prove that a neural net optimized by noisy SGD (NSGD) cannot learn to compute the parity function, it suffices to prove that one cannot determine the parity function from the edge weights at a single time step. Furthermore, in order to prove this, we can use the fact that noise accumulates over multiple time steps and argue that the amount of accumulated noise is large enough to drown out the information on the function provided by each input.

More formally, we plan to do the following. First of all, we will be running NSGD with a small amount of Gaussian noise added to each weight in each time step, and a larger amount of Gaussian noise added to the initial weights. Under these circumstances, the probability distribution of the edge weights resulting from running NSGD on truly random input for a given number of steps will be approximately equal to the convolution of a multivariable Gaussian distribution with something else. As such, it would be possible to construct an oracle approximating the edge weights such that the probability distribution of the edge weights given the oracle's output is essentially a multivariable Gaussian distribution. Next, we show that given any function on $\B^{n+1}$, the expected value of the function on an input generated by a random parity function is approximately equal to its expected value on a true random input. Then, we use that to show that given a slight perturbation of a Gaussian distribution for each $z\in\B^{n+1}$, the distribution resulting from averaging togetherthe perturbed distributions generated by a random parity function is approximately the same as the distribution resulting from averaging together all of the perturbed distributions. Finally, we conclude that the probability distribution of the edge weights after this time step is essentially the same when the input is generated by a random parity function is it is when the input is truly random.

Our first order of business is to establish that the probability distribution of the weights will be approximately equal to the convolution of a multivariable Gaussian distribution with something else, and to do that we will need the following definition.

\begin{definition}
For $\sigma,\epsilon\ge 0$ and a probability distribution $\widehat{P}$, a probability distribution $P$ over $\mathbb{R}^m$ is a $(\sigma,\epsilon)$-blurring of $\widehat{P}$ if
\[||P-\widehat{P} * \mathcal{N}(0,\sigma I)||_1\le 2\epsilon\]
In this situation we also say that $P$ is a $(\sigma,\epsilon)$-blurring. If $\sigma\le 0$ we consider every probability distribution as being a $(\sigma,\epsilon)$-blurring for all $\epsilon$.
\end{definition}

The following are obvious consequences of this definition:
\begin{lemma}
Let $\mathcal{P}$ be a collection of $(\sigma,\epsilon)$-blurrings for some given $\sigma$ and $\epsilon$. Now, select $P\sim \mathcal{P}$ according to some probability distribution, and then randomly select $x\sim P$. The probability distribution of $x$ is also a $(\sigma,\epsilon)$-blurring.
\end{lemma}

\begin{lemma} \label{addBlur}
Let $P$ be a $(\sigma,\epsilon)$-blurring and $\sigma'>0$. Then $P*\mathcal{N}(0,\sigma' I)$ is a $(\sigma+\sigma',\epsilon)$-blurring
\end{lemma}

We want to prove that if the probability distribution of the weights at one time step is a blurring, then the probability distribution of the weights at the next time step is also a blurring. In order to do that, we need to prove that a slight distortion of a blurring is still a bluring. The first step towards that proof is the following lemma:

\begin{lemma}
Let $\sigma, B>0$, $m$ be a positive integer, $m\sqrt{2\sigma/\pi}<r\le1/(mB)$, and $f:\mathbb{R}^m\rightarrow\mathbb{R}^m$ such that $f(0)=0$, $|\frac{\partial f_i}{\partial x_j}(0)|=0$ for all $i$ and $j$, and $|\frac{\partial^2 f_i}{\partial x_j\partial x_{j'}}(x)|\le B$ for all $i$, $j$, $j'$, and all $x$ with $||x||_1<r$. Next, let $P$ be the probability distribution of $X+f(X)$ when $X\sim \mathcal{N}(0,\sigma I)$. Then $P$ is a $(\sigma,\epsilon)$-blurring for $\epsilon=\frac{4(m+2)m^2B\sqrt{2\sigma/\pi}+3m^5B^2\sigma}{8}+(1-Bmr)e^{-(r/2\sqrt{\sigma}-m/\sqrt{2\pi})^2/m}$.
\end{lemma}

\begin{proof}
First, note that for any $x$ with $||x||_1<r$ and any $i$ and $j$, it must be the case that $|\frac{\partial f_i}{\partial x_j}(x)|\le B||x||_1< Br$. That in turn means that for any $x,x'$ with $|x||_1,||x'||_1<r$ and any $i$, it must be the case that $|f(x)_i-f(x')_i|\le Br||x-x'||_1$ with equality only if $x=x'$. In particular, this means that for any such $x,x'$, it must be the case that $||f(x)-f(x')||_1\le mBr||x-x'||_1\le ||x-x'||_1$ with equality only if $x=x'$. Thus, $x+f(x)\ne x'+f(x')$ unless $x=x'$. Also, note that the bound on the second derivatives of $f$ implies that $|f_i(x)|\le B||x||_1^2/2$ for all $||x||_1<r$ and all $i$. This means that

\begin{align*}
&||P-\mathcal{N}(0,\sigma I)||_1\\
&\le 2-2\int_{x:||x||_1<r} \min\left((2\pi \sigma)^{-m/2}e^{-||x||_2^2 /2\sigma},(2\pi \sigma)^{-m/2}e^{-||x+f(x)||_2^2 /2\sigma} |I+[\nabla f^{T}](x)|\right) dx\\
&\le 2-2\int_{x:||x||_1<r}(2\pi \sigma)^{-m/2}e^{-(||x||_2^2+B||x||_1^3/2+mB^2||x||_1^4/4)/2\sigma} (1-Bm||x||_1) dx\\
&= 2(2\pi \sigma)^{-m/2}\int_{x:||x||_1<r}e^{-||x||_2^2/2\sigma}-e^{-(||x||_2^2+B||x||_1^3/2+mB^2||x||_1^4/4)/2\sigma} (1-Bm||x||_1) dx\\
&\indent\indent +2(2\pi \sigma)^{-m/2}\int_{x:||x||_1\ge r} e^{-||x||_2^2/2\sigma} dx\\
&= 2(2\pi \sigma)^{-m/2}\int_{x:||x||_1<r}e^{-||x||_2^2/2\sigma}-e^{-(||x||_2^2+B||x||_1^3/2+mB^2||x||_1^4/4)/2\sigma} dx\\
&\indent\indent +2(2\pi \sigma)^{-m/2}\int_{x:||x||_1< r} Bm||x||_1e^{-(||x||_2^2+B||x||_1^3/2+mB^2||x||_1^4/4)/2\sigma} dx\\
&\indent\indent +2(2\pi \sigma)^{-m/2}\int_{x:||x||_1\ge r} e^{-||x||_2^2/2\sigma} dx\\
&\le 2(2\pi \sigma)^{-m/2}\int_{x:||x||_1<r}\frac{2B||x||_1^3+mB^2||x||_1^4}{8\sigma} e^{-||x||_2^2/2\sigma} dx\\
&\indent\indent +2(2\pi \sigma)^{-m/2}\int_{x:||x||_1< r} Bm||x||_1e^{-||x||_2^2/2\sigma} dx+2(2\pi \sigma)^{-m/2}\int_{x:||x||_1\ge r} e^{-||x||_2^2/2\sigma} dx\\
&\le 2(2\pi \sigma)^{-m/2}\int_{x\in\mathbb{R}^m}\frac{2B||x||_1^3+mB^2||x||_1^4}{8\sigma} e^{-||x||_2^2/2\sigma} dx\\
& \indent\indent+2(2\pi \sigma)^{-m/2}\int_{x\in\mathbb{R}^m} Bm||x||_1e^{-||x||_2^2/2\sigma} dx\\
&\indent\indent +2(2\pi \sigma)^{-m/2}(1-Bmr)\int_{x:||x||_1\ge r} e^{-||x||_2^2/2\sigma} dx\\
&\le \frac{m^3B}{2\sigma}\sqrt{8\sigma^3/\pi}+\frac{m^5B^2}{4\sigma}\cdot 3\sigma^2+2m^2B\sqrt{2\sigma/\pi}+2(2\pi \sigma)^{-m/2}(1-Bmr)\int_{x:||x||_1\ge r} e^{-||x||_2^2/2\sigma} dx\\
&= m^3B\sqrt{2\sigma/\pi}+\frac{3m^5B^2\sigma}{4}+2m^2B\sqrt{2\sigma/\pi}+2(2\pi \sigma)^{-m/2}(1-Bmr)\int_{x:||x||_1\ge r} e^{-||x||_2^2/2\sigma} dx\\
&=\frac{4(m+2)m^2B\sqrt{2\sigma/\pi}+3m^5B^2\sigma}{4}+2(2\pi \sigma)^{-m/2}(1-Bmr)\int_{x:||x||_1\ge r} e^{-||x||_2^2/2\sigma} dx
\end{align*}

Next, observe that for any $\lambda\ge 0$, it must be the case that
\begin{align*}
&(2\pi \sigma)^{-m/2}\int_{x:||x||_1\ge r} e^{-||x||_2^2/2\sigma} dx\\
&\le (2\pi \sigma)^{-m/2} e^{-\lambda r/\sigma} \int_{x:||x||_1\ge r} e^{\lambda||x||_1/\sigma} e^{-||x||_2^2/2\sigma} dx\\
&\le (2\pi \sigma)^{-m/2} e^{-\lambda r/\sigma} \int_{x\in\mathbb{R}^m} e^{\lambda||x||_1/\sigma} e^{-||x||_2^2/2\sigma} dx\\
&=e^{-\lambda r/\sigma}\left[ (2\pi \sigma)^{-1/2} \int_{x_1\in\mathbb{R}} e^{\lambda |x_1|/\sigma} e^{-x_1^2/2\sigma}
 dx_1\right]^m\\
&=e^{-\lambda r/\sigma}\left[ 2(2\pi \sigma)^{-1/2} \int_0^\infty e^{\lambda x_1/\sigma} e^{-x_1^2/2\sigma} dx_1\right]^m\\
&=e^{-\lambda r/\sigma}\left[ 2(2\pi \sigma)^{-1/2} \int_0^\infty e^{\lambda ^2/2\sigma} e^{-(x_1-\lambda)^2/2\sigma} dx_1\right]^m\\
&=e^{-\lambda r/\sigma}\left[ 2e^{\lambda ^2/2\sigma} (2\pi \sigma)^{-1/2} \int_{-\lambda}^\infty e^{-x_1^2/2\sigma} dx_1\right]^m\\
&\le e^{-\lambda r/\sigma}\left[ e^{\lambda^2/2\sigma}(1+2\lambda/\sqrt{2\pi\sigma})\right]^m\\
&\le e^{-\lambda r/\sigma+m\lambda^2/2\sigma+2m\lambda/\sqrt{2\pi\sigma}}
\end{align*}

In particular, if we set $\lambda=r/m-\sqrt{2\sigma/\pi}$, this shows that $(2\pi \sigma)^{-m/2}\int_{x:||x||_1\ge r} e^{-||x||_2^2/2\sigma} dx\le e^{-(r/2\sqrt{\sigma}-m/\sqrt{2\pi})^2/m}$. The desired conclusion follows.
\end{proof}

\begin{lemma}
Let $\sigma, B_1, B_2>0$, $m$ be a positive integer with $m<1/B_1$, $m\sqrt{2\sigma/\pi}<r\le(1-mB_1)/(mB_2)$, and $f:\mathbb{R}^m\rightarrow\mathbb{R}^m$ such that $|\frac{\partial f_i}{\partial x_j}(0)|\le B_1$ for all $i$ and $j$, and $|\frac{\partial^2 f_i}{\partial x_j\partial x_{j'}}(x)|\le B_2$ for all $i$, $j$, $j'$, and all $x$ with $||x||_1<r$. Next, let $P$ be the probability distribution of $X+f(X)$ when $X\sim \mathcal{N}(0,\sigma I)$. Then $P$ is a $((1-mB_1)^2\sigma,\epsilon)$-blurring for $\epsilon=\frac{4(m+2)m^2B_2\sqrt{2\sigma/\pi}/(1-mB_1)+3m^5B_2^2\sigma/(1-mB_1)^2}{8}+(1-(1+mB_1)B_2mr)e^{-(r/2\sqrt{\sigma}-m/\sqrt{2\pi})^2/m}$.
\end{lemma}

\begin{proof}
First, define $h:\mathbb{R}^m\rightarrow\mathbb{R}^m$ such that $h(x)=f(0)+x+[\nabla f^{(t)}]^T(0) x$ for all $x$. Every eigenvalue of $[\nabla f](0)$ has a magnitude of at most $mB_1$, so $h$ is invertible. Next, define $f^\star:\mathbb{R}^m\rightarrow\mathbb{R}^m$ such that $f^\star(x)=h^{-1}(x+f(x))-x$ for all $x$. Clearly, $f^\star(0)=0$, and $\frac{\partial f^\star_i}{\partial x_j}(0)=0$ for all $i$ and $j$. Furthermore, for any given $x$ it must be the case that
$\max_{i,j,j'} |\frac{\partial^2 f_i}{\partial x_j\partial x_{j'}}|\ge (1-mB_1) \max_{i,j,j'} |\frac{\partial^2 f^\star_i}{\partial x_j\partial x_{j'}}|$. So, $|\frac{\partial^2 f_i}{\partial x_j\partial x_{j'}}|\le B_2/(1-mB_1)$ for all $i$, $j$, $j'$, and all $x$ with $||x||_1<r$. Now, let $P^\star$ be the probability distribution of $x+f^\star(x)$ when  $x\sim \mathcal{N}(0,\sigma I)$. By the previous lemma, $P^\star$ is a $(\sigma,\epsilon)$-blurring for $\epsilon=\frac{4(m+2)m^2B_2\sqrt{2\sigma/\pi}/(1-mB_1)+3m^5B_2^2\sigma/(1-mB_1)^2}{8}+(1-(1+mB_1)B_2mr)e^{-(r/2\sqrt{\sigma}-m/\sqrt{2\pi})^2/m}$.

Now, let $\widehat{P^\star}$ be a probability distribution such that $P^\star$ is a $(\sigma,\epsilon)$-blurring of $\widehat{P^\star}$. Next, let $\widehat{P}$ be the probability distribution of $h(x)$ when $x$ is drawn from $\widehat{P^\star}$. Also, let $M= (I+[\nabla f^T]^T(0))(I+[\nabla f^T](0))$. The fact that $||P^\star-\widehat{P^\star} * \mathcal{N}(0,\sigma I)||_1\le 2\epsilon$ implies that 
\[||P-\widehat{P} * \mathcal{N}(0,\sigma M)||_1\le 2\epsilon\]
For any $x\in\mathbb{R}^m$, it must be the case that 
\begin{align*}
x\cdot M x&\ge ||x||_2^2-2B_1||x||_1^2-mB_1^2||x||_1^2\\
&\ge ||x||_2^2-2mB_1||x||_2^2-m^2B_1^2||x||_2^2=(1-mB_1)^2||x||_2^2\\
\end{align*}
That in turn means that $\sigma M-\sigma(1-mB_1)^2I$ is positive semidefinite. So, $\widehat{P} * \mathcal{N}(0,\sigma M)=\widehat{P} * \mathcal{N}(0,\sigma M-\sigma(1-mB_1)^2I) * \mathcal{N}(0,\sigma(1-mB_1)^2I)$, which proves that $P$ is a $((1-mB_1)^2\sigma,\epsilon)$-blurring of $\widehat{P} * \mathcal{N}(0,\sigma M-\sigma(1-mB_1)^2I)$. 
\end{proof}

Any blurring is approximately equal to a linear combination of Gaussian distributions, so this should imply a similar result for $X$ drawn from a $(\sigma,\epsilon)$ blurring. However, we are likely to use functions that have derivatives that are large in some places. Not all of the Gaussian distributions that the blurring combines will necessarily have centers that are far enough from the high derivative regions. As such, we need to add an assumption that the centers of the distributions are in regions where the derivatives are small. We formalize the concept of being in a region where the derivatives are small as follows.

\begin{definition}
Let $f:\mathbb{R}^m\rightarrow\mathbb{R}^m$, $x\in\mathbb{R}^m$, and $r, B_1, B_2>0$. Then $f$ is $(r, B_1, B_2)$-stable at $x$ if $|\frac{\partial f_i}{\partial x_j}(0)|\le B_1$ for all $i$ and $j$ and all $x'$ with $||x'-x||_1<r$, and $|\frac{\partial^2 f_i}{\partial x_j\partial x_{j'}}|\le B_2$ for all $i$, $j$, $j'$, and all $x'$ with $||x'-x||_1<2r$. Otherwise, $f$ is $(r, B_1, B_2)$-unstable at $x$.
\end{definition}

This allows us to state the following variant of the previous lemma.

\begin{lemma}\label{stableBlur}
Let $\sigma, B_1, B_2>0$, $m$ be a positive integer with $m<1/B_1$, $m\sqrt{2\sigma/\pi}<r\le(1-mB_1)/(mB_2)$, and $f:\mathbb{R}^m\rightarrow\mathbb{R}^m$ such that there exists $x$ with $||w||_1<r$ such that $f$ is $(r, B_1, B_2)$-stable at $x$. Next, let $P$ be the probability distribution of $X+f(X)$ when $X\sim \mathcal{N}(0,\sigma I)$. Then $P$ is a $((1-mB_1)^2\sigma,\epsilon)$-blurring for $\epsilon=\frac{4(m+2)m^2B_2\sqrt{2\sigma/\pi}/(1-mB_1)+3m^5B_2^2\sigma/(1-mB_1)^2}{8}+(1-(1+mB_1)B_2mr)e^{-(r/2\sqrt{\sigma}-m/\sqrt{2\pi})^2/m}$.
\end{lemma}

\begin{proof}
$|\frac{\partial f_i}{\partial x_j}(0)|\le B_1$ for all $i$ and $j$, and $|\frac{\partial^2 f_i}{\partial x_j\partial x_{j'}}(x')|\le B_2$ for all $i$, $j$, $j'$, and all $x'$ with $||x'||_1<r$. Then, the desired conclusion follows by the previous lemma.
\end{proof}

This lemma could be relatively easily used to prove that if we draw $X$ from a $(\sigma,\epsilon)$-blurring instead of drawing it from $\mathcal{N}(0,\sigma I)$ and $f$ is stable at $X$ with high probability then the probability distribution of $X+f(X)$ will be a $(\sigma',\epsilon')$-blurring for $\sigma'\approx\sigma$ and $\epsilon'\approx\epsilon$. However, that is not quite what we will need. The issue is that we are going to repeatedly apply a transformation along these lines to a variable. If all we know is that its probability distribution is a $(\sigma^{(t)},\epsilon^{(t)})$-blurring in each step, then we potentially have a probability of $\epsilon^{(t)}$ each time step that it behaves badly in that step. That is consistent with there being a probability of $\sum \epsilon^{(t)}$ that it behaves badly eventually, which is too high.

In order to avoid this, we will think of these blurrings as approximations of a $(\sigma,0)$ blurring. Then, we will need to show that if $X$ is good in the sense of being present in the idealized form of the blurring then $X+f(X)$ will also be good. In order to do that, we will need the following definition.

\begin{definition}
Let $P$ be a $(\sigma,\epsilon)$-blurring of $\widehat{P}$, and $X\sim P$. A $\sigma$-revision of $X$ to $\widehat{P}$ is a random pair $(X',M)$ such that the probability distribution of $M$ is $\widehat{P}$, the probability distribution of $X'$ given that $M=\mu$ is $\mathcal{N}(\mu,\sigma I)$, and $P[X'\ne X]=||P-\mathcal{N}(0,\sigma I) *  \widehat{P}||_1/2$. Note that a $\sigma$-revision of $X$ to $\widehat{P}$ will always exist.
\end{definition}

\subsubsection{Means, SLAs, and Gaussian distributions}
Our plan now is to consider a version of NoisyStochasticGradientDescent in which the edge weights get revised after each step and then to show that under suitable assumptions when this algorithm is executed none of the revisions actually change the values of any of the edge weights. Then, we will show that whether the samples are generated randomly or by a parity function has minimal effect on the probability distribution of the edge weights after each step, allowing us to revise the edge weights in both cases to the same probability distribution. That will allow us to prove that the probability distribution of the final edge weights is nearly independent of which probability distribution the samples are drawn from.

The next step towards doing that is to show that if we run NoisySampleGradientDescentStep on a neural network with edge weights drawn from a linear combination of Gaussian distributions, the probability distribution of the resulting graph is essentially independent of what parity function we used to generate the sample. In order to do that, we are going to need some more results on the difficulty of distinguishing an unknown parity function from a random function. First of all, recall that corollary \ref{parityAverage} says that

\begin{corollary}
Let $n>0$ and $f:\B^{n+1}\rightarrow \mathbb{R}$. Also, let $X$ be a random element of $\B^n$ and $Y$ be a random element of $\B$. Then
\[\sum_{s\subseteq[n]} |E[f((X,Y))]-E[f((X,p_s(X)))]|\le 2^{n/2}\sqrt{E[f^2((X,Y))]}\]
\end{corollary}

We can apply this to probability distributions to get the following.

\begin{theorem}
Let $m>0$, and for each $z\in\B^{n+1}$, let $P_z$ be a probability distribution on $\mathbb{R}^m$ with probability density function $f_z$. Now, randomly select $Z\in\B^{n+1}$ and $X\in\B^n$ uniformly and independently. Next, draw $W$ from $P_Z$ and $W'_s$ from $P_{(X,p_s(X))}$ for each $s\subseteq[n]$. Let $P^\star$ be the probability distribution of $W$ and $P^{\star}_s$ be the probability distribution of $W'_s$ for each $s$. Then
\[2^{-n} \sum_{s\subseteq[n]} ||P^\star-P^\star_s||_1\le 2^{-n/2}\int_{\mathbb{R}^m} \max_{z\in \B^{n+1}} f_z(w) dw\]
\end{theorem}

\begin{proof}
Let $f^\star=2^{-n-1}\sum_{z\in\B^{n+1}} f_z$ be the probability density function of $P^\star$, and for each $s\subseteq[n]$, let $f^\star_s=2^{-n}\sum_{x\in\B^n} f_{(x,p_s(x))}$ be the probability density function of $P^\star_s$.

For any $w\in \mathbb{R}^m$, we have that
\begin{align*}
&\sum_{s\subseteq[n]} |f^\star(w)-f^\star_s(w)|\\
&=\sum_{s\subseteq[n]} |E[f_{Z}(w)]-E[f_{(X,p_s(X))}(w)|\\
&\le 2^{n/2}\sqrt{E[f^2_{Z}(w)]}\\
&\le 2^{n/2}\max_{z\in\B^{n+1}} f_z(w)\\
\end{align*}
That means that

\begin{align*}
& \sum_{s\subseteq[n]} ||P^\star-P^\star_s||_1\\
&= \sum_{s\subseteq[n]} \int_{\mathbb{R}^m} |f^\star(w)-f^\star_s(w)| dw\\
&\le \int_{\mathbb{R}^m} \sum_{s\subseteq[n]} |f^\star(w)-f^\star_s(w)| dw\\
&\le 2^{n/2}\int_{\mathbb{R}^m} \max_{z\in \B^{n+1}} f_z(w) dw
\end{align*}
\end{proof}

In particular, if these probability distributions are the result of applying a well-behaved distortion function to a Gaussian distribution, we have the following.

\begin{theorem} \label{gausParity}
Let $\sigma, B_0, B_1>0$, and $n$ and $m$ be positive integers with $m<1/B_1$. Also, for every $z\in\B^{n+1}$, let $f^{(z)}:\mathbb{R}^m\rightarrow\mathbb{R}^m$ be a function such that $|f^{(z)}_i(w)|\le B_0$ for all $i$ and $w$ and $|\frac{\partial f^{(z)}_i}{\partial w_j}(w)|\le B_1$ for all $i$, $j$, and $w$. Now, randomly select $Z\in\B^{n+1}$ and $X\in\B^n$ uniformly and independently. Next, draw $W_0$ from $\mathcal{N}(0,\sigma I)$, set $W=W_0+f^{(Z)}(W_0)$ and $W'_s=W_0+f^{(X,p_s(X))}(W_0)$ for each $s\subseteq[n]$. Let $P^\star$ be the probability distribution of $W$ and $P^\star_s$ be the probability distribution of $W'_s$ for each $s$. Then
\[2^{-n}\sum_{s\subseteq[n]}||P^\star-P^\star_s||_1\le 2^{-n/2}\cdot e^{2m B_0/\sqrt{2\pi\sigma}}/(1-mB_1)\]
\end{theorem}

\begin{proof}
First, note that the bound on $|\frac{\partial f^{(z)}_i}{\partial w_j}(w)|$ ensures that if $w+f^{(z)}(w)=w'+f^{(z)}(w')$ then $w=w'$. So, for any $z$ and $w$, the probability density function of $W_0+f^{(z)}(W_0)$ at $w$ is less than or equal to 
\[(2\pi\sigma)^{-m/2} e^{-\sum_{i=1}^m \max^2(|w_i|-B_0,0)/2\sigma}/ |I+[\nabla f^{(z)}]^T(w)|\]
which is less than or equal to 
\[(2\pi\sigma)^{-m/2} e^{-\sum_{i=1}^m \max^2(|w_i|-B_0,0)/2\sigma}/(1-mB_1)\]

By the previous theorem, that implies that
\begin{align*}
&2^{-n}\sum_{s\subseteq[n]}||P^\star-P^\star_s||_1\\
&\le 2^{-n/2}\int_{\mathbb{R}^m} (2\pi\sigma)^{-m/2} e^{-\sum_{i=1}^m \max^2(|w_i|-B_0,0)/2\sigma}/(1-mB_1) dw\\
&=2^{-n/2}\left[\int_{\mathbb{R}}  (2\pi\sigma)^{-1/2} e^{- \max^2(|w'|-B_0,0)/2\sigma} dw'\right]^m/(1-mB_1)\\
&=2^{-n/2} [1+2B_0/\sqrt{2\pi\sigma}]^m/(1-mB_1)\\
&\le 2^{-n/2}\cdot e^{2m B_0/\sqrt{2\pi\sigma}}/(1-mB_1)
\end{align*}
\end{proof}

The problem with this result is that it requires $f$ to have values and derivatives that are bounded everywhere, and the functions that we will encounter in practice will not necessarily have that property. We can reasonably require that our functions have bounded values and derivatives in the regions we are likely to evaluate them on, but not in the entire space. Our solution to this will be to replace the functions with new functions that have the same value as them in small regions that we are likely to evaluate them on, and that obey the desired bounds. The fact that we can do so is established by the following theorem.

\begin{theorem}
Let $B_0, B_1, B_2, r, \sigma>0$, $\mu\in\mathbb{R}^m$, and $f:\mathbb{R}^m\rightarrow \mathbb{R}^m$ such that there exists $x$ with $||x-\mu||_1< r$ such that $f$ is $(r,B_1, B_2)$-stable at $x$ and $|f_i(x)|\le B_0$ for all $i$. Then there exists a function $f^\star:\mathbb{R}^m\rightarrow\mathbb{R}^m$ such that $f^\star(x)=f(x)$ for all $x$ with $||x-\mu||_1< r$, and $|f_i(x)|\le B_0+2rB_1+2r^2B_2$ and $|\frac{\partial f_i}{\partial x_j}(x)|\le 2B_1+2rB_2$ for all $x\in\mathbb{R}^m$ and $i,j\in[m]$.
\end{theorem}

\begin{proof}
First, observe that the $(r,B_1, B_2)$-stability of $f$ at $x$ implies that for every $x'$ with $||x-x'||\le 2r$, we have that $|\frac{\partial f_i}{\partial x_j}(x')|\le B_1+rB_2$ and $|f_i(x')|\le B_0+2r(B_1+r B_0)$. In particular, this holds for all $x'$ with $||x'-\mu||_1\le 2r-||x-\mu||_1<r$. 

 That means that there exists $r'>r$ such that the values and derivatives of $f$ satisfy the desired bounds for all $x'$ with $||x'-\mu||_1\le r'$. Now, define the function $\overline{f}:\mathbb{R}^m\rightarrow\mathbb{R}^m$ such that $\overline{f}(x')=f(\mu+(x'-\mu)\cdot r'/||x'-\mu||_1)$. This function satisfies the bounds for all $x'$ with $||x'||_1>r'$, except that it may not be differentiable when $x'_j=\mu_j$ for some $j$. Consider defining $f^\star(x')$ to be equal to $f(x')$ when $||x'-\mu||_1\le r'$ and $\overline{f'}(x')$ otherwise. This would almost work, except that it may not be differentiable when $||x'-\mu||_1=r'$, or $||x'-\mu||_1>r'$ and $x'_j=\mu_j$ for some $j$.
 
 In order to fix this, we define a smooth function $h$ of bounded derivative such that $h(x')=0$ whenever $||x'-\mu||_1\le r$, and $h(x')\ge 1$ whenever $||x'-\mu||_1\ge r'$. Then, for all sufficiently small positive constants $\delta$, $f^\star * \mathcal{N}(0,\delta\cdot h^2(x') I)$ has the desired properties.
\end{proof}

Combining this with the previous theorem yields the following.
\begin{corollary}
Let $\sigma, B_0, B_1,B_2,r>0$, $\mu\in\mathbb{R}^m$, and $n$ and $m$ be positive integers with $m<1/(2B_1+2rB_2)$. Then, for every $z\in\B^{n+1}$, let $f^{(z)}:\mathbb{R}^m\rightarrow\mathbb{R}^m$ be a function such that there exists $x$ with $||x-\mu||_1< r$ such that $f$ is $(r,B_1, B_2)$-stable at $x$ and $|f_i(x)|\le B_0$ for all $i$. Next, draw $W_0$ from $\mathcal{N}(\mu,\sigma I)$. Now, randomly select $Z\in\B^{n+1}$ and $X\in\B^n$ uniformly and independently. Then, set $W=W_0+f^{(Z)}(W_0)$ and $W'_s=W_0+f^{(X,p_s(X))}(W_0)$ for each $s\subseteq[n]$. Let $P^\star$ be the probability distribution of $W$ and $P^\star_s$ be the probability distribution of $W'_s$ for each $s$. Then
\[2^{-n}\sum_{s\subseteq[n]}||P^\star-P^\star_s||_1\le 2^{-n/2}\cdot e^{2m ( B_0+2rB_1+2r^2B_2)/\sqrt{2\pi\sigma}}/(1-2mB_1-2rmB_2)+2e^{-(r/2\sqrt{\sigma}-m/\sqrt{2\pi})^2/m}\]
\end{corollary}

\begin{proof}
For each $z$, we can define $f^{(z)\star}$ as an approximation of $f^{(z)}$ as explained in the previous theorem. $||W_0-\mu||_1\le r$ with a probability of at least $1-e^{-(r/2\sqrt{\sigma}-m/\sqrt{2\pi})^2/m}$, in which case $f^{(z)\star}(W_0)=f^{(z)}(W_0)$ for all $z$. For a random $s$, the probability distributions of $W_0+f^{(Z)\star}(W_0)$ and $W_0+f^{(X,p_s(X))\star}(W_0)$ have an $L_1$ difference of at most $2^{-n/2}\cdot e^{2m ( B_0+2rB_1+2r^2B_2)/\sqrt{2\pi\sigma}}/(1-2mB_1-2rmB_2)$ on average by $\ref{gausParity}$. Combining these yields the desired result.
\end{proof}

That finally gives us the components needed to prove the following.

\begin{theorem}
Let $m,n>0$ and define $f^{[z]}:\mathbb{R}^m\rightarrow \mathbb{R}^m$ to be a smooth function for all $z\in \B^{n+1}$. Also, let $\sigma, B_0, B_1, B_2>0$ such that $B_1<1/2m$, $m\sqrt{2\sigma/\pi}<r\le(1-2mB_1)/(2mB_2)$, $T$ be a positive integer, and $\mu_0\in \mathbb{R}^m$. Then, let $\star$ be the uniform distribution on $\B^{n+1}$, and for each $s\subseteq [n]$, let $\rho_s$ be the probability distribution of $(X,p_s(X))$ when $X$ is chosen randomly from $\B^n$. Next, let $P_{\mathcal{Z}}$ be a probability distribution on $\B^{n+1}$ that is chosen by means of the following procedure. First, with probability $1/2$, set $P_{\mathcal{Z}}=\star$. Otherwise, select a random $S\subseteq[n]$ and set $P_{\mathcal{Z}}=\rho_S$.

Now, draw $W^{(0)}$ from $\mathcal{N}(\mu_0,\sigma I)$, independently draw $Z_i\sim P_{\mathcal{Z}}$ and $\Delta^{(i)}\sim \mathcal{N}(0,[2mB_1-m^2B_1^2]\sigma I)$ for all $0<i\le T$. Then, set $W^{(i)}=W^{(i-1)}+f^{[Z_i]}(W^{(i-1)})+\Delta^{(i)}$ for each $0<i\le T$, and let $p$ be the probability that there exists $0\le i\le T$ such that $F^{[Z_i]}$ is $(r, B_1, B_2)$-unstable at $W^{(i)}$ or $||F^{[Z_i]}(W^{(i)})||_\infty>B_0$. Finally, let $Q$ and $Q'_s$ be the probability distribution of $W^{(T)}$ given that $P_{\mathcal{Z}}=\star$ and the probability distribution of $W^{(T)}$ given that $P_{\mathcal{Z}}=\rho_s$. Then
\[2^{-n}\sum_{s\subseteq[n]}||Q-Q'_s||_1\le 4p+T(4\epsilon+\epsilon'+4\epsilon'')\]
where
 \begin{align*}
      &\epsilon=\frac{4(m+2)m^2B_2\sqrt{2\sigma/\pi}/(1-mB_1)+3m^5B_2^2\sigma/(1-mB_1)^2}{8}\\
 &+(1-(1+mB_1)B_2mr)e^{-(r/2\sqrt{\sigma}-m/\sqrt{2\pi})^2/m}
  \end{align*}

\[\epsilon'=2^{-n/2}\cdot e^{2m ( B_0+2rB_1+2r^2B_2)/\sqrt{2\pi\sigma}}/(1-2mB_1-2rmB_2)+2e^{-(r/2\sqrt{\sigma}-m/\sqrt{2\pi})^2/m}\]
\[\epsilon''=e^{-(r/2\sqrt{\sigma}-m/\sqrt{2\pi})^2/m}\]
\end{theorem}

\begin{proof}
In order to prove this, we plan to define new variables $\widetilde{W}^{(i)\prime}$ such that $\widetilde{W}^{(i)\prime}=W^{(i)}$ with high probability for each $i$ and the probability distribution of $\widetilde{W}^{(i)\prime}$ is independent of $P_{\mathcal{Z}}$. More precisely, we define the variables $\widetilde{W}^{(i)}$, $\widetilde{W}^{(i)\prime}$, and $\widetilde{M}^{(i)}$ for each $i$ as follows. First, set $\widetilde{M}^{(0)}=\mu_0$ and $\widetilde{W}^{(0)\prime}=\widetilde{W}^{(0)}=W^{(0)}$. 

Next, for a function $f$ and a point $w$, we say that $f$ is quasistable at $w$ if there exists $w'$ such that $||w'-w||_1\le r$, $f^{[Z_i]}$ is $(r, B_1,B_2)$-stable at $w'$, and $||f^{[Z_i]}(w')||_\infty\le B_0$, and that it is quasiunstable at $w$ otherwise.

for each $0<i\le T$, if $f^{[Z_i]}$ is quasistable at $\widetilde{M}^{(i-1)}$, set 
\[\widetilde{W}^{(i)}=\widetilde{W}^{(i-1)\prime}+f^{[Z_i]}(\widetilde{W}^{(i-1)\prime})+\Delta^{(i)}\]
Otherwise, set 
\[\widetilde{W}^{(i)}=\widetilde{W}^{(i-1)\prime}+\Delta^{(i)}\]

Next, for each $\rho$, let $P^{(i)}_{\rho}$ be the probability distribution of $\widetilde{W}^{(i)}$ given that $P_{\mathcal{Z}}=\rho$. Then, define $\widehat{P}^{(i)}$ as a probability distribution such that $P^{(i)}_{\star}$ is a $(\sigma,\epsilon_0)$-blurring of $\widehat{P}^{(i)}$ with $\epsilon_0$ as small as possible. Finally, for each $\rho$, if $P_{\mathcal{Z}}=\rho$, let $(\widetilde{W}^{(i)\prime},\widetilde{M}^{(i)})$ be a $\sigma$-revision of $\widetilde{W}^{(i)}$ to $\widehat{P}^{(i)}$.

In order to analyse the behavior of these variables, we will need to make a series of observations. First, note that for every $i$, $\rho$, and $\mu$ the probability distribution of $\widetilde{W}^{(i-1)\prime}$ given that $P_Z=\rho$ and $M^{(i-1)}=\mu$
is $\mathcal{N}(\mu,\sigma I)$. Also, either $f^{[Z_i]}$ is quasistable at $\mu$ or $0$ is quasistable at $\mu$. Either way, the probability distribution of $\widetilde{W}^{(i)}$ under these circumstances must be a $(\sigma,\epsilon)$-blurring by Lemma \ref{stableBlur} and Lemma \ref{addBlur}. That in turn means that $P^{(i)}_\rho$ is a $(\sigma,\epsilon)$ blurring for all $i$ and $\rho$, and thus that $P^{(i)}_\star$ must be a $(\sigma,\epsilon)$ blurring of $\widehat{P}^{(i)}$. Furthermore, by the previous corollary, 

\[2^{-n}\sum_{s\subseteq[n]}||P^{(i)}_\star-P^{(i)}_{\rho_s}||_1\le \epsilon'\]

The combination of these implies that 
\[2^{-n}\sum_{s\subseteq[n]}||\mathcal{N}(0,\sigma I) * \widehat{P}^{(i)}-P^{(i)}_{\rho_s}||_1\le 2\epsilon+\epsilon'\]
which in turn means that $P[\widetilde{W}^{(i)\prime}\ne \widetilde{W}^{(i)}]\le \epsilon+\epsilon'/4$. That in turn means that with probability at least $1-T(\epsilon+\epsilon'/4)$ it is the case that $\widetilde{W}^{(i)\prime}= \widetilde{W}^{(i)}$ for all $i$.

If $\widetilde{W}^{(i)\prime}= \widetilde{W}^{(i)}$ for all $i$ and $\widetilde{W}^{(T)\prime}\ne W^{(T)}$ then there must exist some $i$ such that $\widetilde{W}^{(i-1)\prime}= W^{(i-1)}$ but $\widetilde{W}^{(i)}\ne W^{(i)}$. That in turn means that
\begin{align*}
\widetilde{W}^{(i)} &\ne W^{(i)}\\
&=W^{(i-1)}+f^{[Z_i]}(W^{(i-1)})+\Delta^{(i)}\\
&=\widetilde{W}^{(i-1)\prime}+f^{[Z_i]}(\widetilde{W}^{(i-1)\prime})+\Delta^{(i)}
\end{align*}

If $F^{[Z_i]}$ were quasistable at $M^{(i-1)}$, that is exactly the formula that would be used to calculate $\widetilde{W}^{(i)}$, so $F^{[Z_i]}$ must be quasiunstable at $M^{(i-1)}$. That in turn requires that either $F^{[Z_i]}$ is $(r, B_1,B_2)$-unstable at $\widetilde{W}^{(i-1)\prime}= W^{(i-1)}$,  $||f^{[Z_i]}(W^{(i-1)})||_\infty> B_0$, or $||\widetilde{W}^{(i-1)\prime}-M^{(i-1)}||_1>r$. With probability at least $1-p$, neither of the first two scenarios occur for any $i$, while for any given $i$ the later occurs with a probability of at most $\epsilon''$. Thus, 
\[P[\widetilde{W}^{(T)\prime}\ne W^{(T)}]\le p+T(\epsilon+\epsilon'/4+\epsilon'')\]

The probability distribution of $\widetilde{W}^{(T)\prime}$ is independent of $P_{\mathcal{Z}}$, so it must be the case that
\[2^{-n}\sum_{s\subseteq[n]}||Q-Q'||_1\le 2P[\widetilde{W}^{(T)\prime}\ne W^{(T)}|P_{\mathcal{Z}}=\star]+2P[\widetilde{W}^{(T)\prime}\ne W^{(T)}|P_{\mathcal{Z}}\ne \star]\le 4p+T(4\epsilon+\epsilon'+4\epsilon'')\]
\end{proof}

In particular, if we let $(h,G)$ be a neural net, $G_W$ be $G$ with its edge weights changed to the elements of $W$, $L$ be a loss function,  
\[\overline{f}^{(x,y)}(W)=L(eval_{(h,G_W)}(x)-y),\]
and $f^{(x,y)}=-\gamma\nabla \overline{f}^{(x,y)}$ for each $x,y$ then this translates to the following.

\begin{corollary}
 Let $(h,G)$ be a neural net with $n$ inputs and $m$ edges, $G_W$ be $G$ with its edge weights changed to the elements of $W$, and $L$ be a loss function.
Also, let $\gamma, \sigma, B_0, B_1, B_2>0$ such that $B_1<1/2m$, $m\sqrt{2\sigma/\pi}<r\le(1-2mB_1)/(2mB_2)$, and $T$ be a positive integer. Then, let $\star$ be the uniform distribution on $\B^{n+1}$, and for each $s\subseteq [n]$, let $\rho_s$ be the probability distribution of $(X,p_s(X))$ when $X$ is chosen randomly from $\B^n$. Next, let $P_{\mathcal{Z}}$ be a probability distribution on $\B^{n+1}$ that is chosen by means of the following procedure. First, with probability $1/2$, set $P_{\mathcal{Z}}=\star$. Otherwise, select a random $S\subseteq[n]$ and set $P_{\mathcal{Z}}=\rho_S$.

Now, let $G'$ be $G$ with each of its edge weights perturbed by an independently generated variable drawn from $\mathcal{N}(0,\sigma I)$ and run 
\\$NoisyStochasticGradientDescentAlgorithm(h, G', P_{\mathcal{Z}}, L, \gamma, \infty, \mathcal{N}(0,[2mB_1-m^2B_1^2]\sigma I), T)$. Then, let $p$ be the probability that there exists $0\le i<T$ such that at least one of the following holds: \begin{enumerate}
\item One of the first derivatives of $L(eval_{(h,G_i)}(X_i)-Y_i)$ with respect to the edge weights has magnitude greater than $B_0/\gamma$. 
\item There exists a perturbation $G'_i$ of $G_i$ with no edge weight changed by more than $r$ such that one of the second derivatives of $L(eval_{(h,G'_i)}(X_i)-Y_i)$ with respect to the edge weights has magnitude greater than $B_1/\gamma$. 
\item There exists a perturbation $G'_i$ of $G_i$ with no edge weight changed by more than $2r$ such that one of the third derivatives of $L(eval_{(h,G'_i)}(X_i)-Y_i)$ with respect to the edge weights has magnitude greater than $B_2/\gamma$. 
\end{enumerate}
Finally, let $Q$ be the probability distribution of the final edge weights given that $P_{\mathcal{Z}}=\star$ and $Q'_s$ be the probability distribution of the final edge weights given that $P_{\mathcal{Z}}=\rho_s$. Then
\[2^{-n}\sum_{s\subseteq[n]}||Q-Q'_s||_1\le 4p+T(4\epsilon+\epsilon'+4\epsilon'')\]
where
 \begin{align*}
     &\epsilon=\frac{4(m+2)m^2B_2\sqrt{2\sigma/\pi}/(1-mB_1)+3m^5B_2^2\sigma/(1-mB_1)^2}{8}\\
     &+(1-(1+mB_1)B_2mr)e^{-(r/2\sqrt{\sigma}-m/\sqrt{2\pi})^2/m}
      \end{align*} 
\[\epsilon'=2^{-n/2}\cdot e^{2m ( B_0+2rB_1+2r^2B_2)/\sqrt{2\pi\sigma}}/(1-2mB_1-2rmB_2)+2e^{-(r/2\sqrt{\sigma}-m/\sqrt{2\pi})^2/m}\]

\[\epsilon''=e^{-(r/2\sqrt{\sigma}-m/\sqrt{2\pi})^2/m}\]
\end{corollary}

\begin{corollary}
 Let $(h,G)$ be a neural net with $n$ inputs and $m$ edges, $G_W$ be $G$ with its edge weights changed to the elements of $W$, $L$ be a loss function, and $B>0$. Next, define $\gamma$ such that $0<\gamma\le \pi n/80m^2B$, and let $T$ be a positive integer. Then, let $\star$ be the uniform distribution on $\B^{n+1}$, and for each $s\subseteq [n]$, let $\rho_s$ be the probability distribution of $(X,p_s(X))$ when $X$ is chosen randomly from $\B^n$. Next, let $P_{\mathcal{Z}}$ be a probability distribution on $\B^{n+1}$ that is chosen by means of the following procedure. First, with probability $1/2$, set $P_{\mathcal{Z}}=\star$. Otherwise, select a random $S\subseteq[n]$ and set $P_{\mathcal{Z}}=\rho_S$.

Next, set $\sigma=\left(\frac{40m\gamma B}{n}\right)^2/2\pi$. Now, let $G'$ be $G$ with each of its edge weights perturbed by an independently generated variable drawn from $\mathcal{N}(0,\sigma I)$ and run 
\\$NoisyStochasticGradientDescentAlgorithm(h, G', P_{\mathcal{Z}}, L, \gamma, \infty, \mathcal{N}(0,[2mB\gamma-m^2B^2\gamma^2]\sigma I), T)$. Let $p$ be the probability that there exists $0\le i<T$ such that there exists a perturbation $G'_i$ of $G_i$ with no edge weight changed by more than $160m^2 \gamma B/\pi n$ such that one of the first three derivatives of $L(eval_{(h,G_i)}(X_i)-Y_i)$ with respect to the edge weights has magnitude greater than $B$. Finally, let $Q$ be the probability distribution of the final edge weights given that $P_{\mathcal{Z}}=\star$ and $Q'_s$ be the probability distribution of the final edge weights given that $P_{\mathcal{Z}}=\rho_s$. Then
\[2^{-n}\sum_{s\subseteq[n]}||Q-Q'_s||_1\le 4p+T(720m^4B^2\gamma^2/\pi n+14[e/4]^{n/4})\]
\end{corollary}

\begin{proof}
First, set $r=80m^2\gamma B/\pi n$. Also, set $B_1=B_2=B_3=\gamma B$. By the previous corollary, we have that 
\[2^{-n}\sum_{s\subseteq[n]}||Q-Q'_s||_1\le 4p+T(4\epsilon+\epsilon'+4\epsilon'')\]
where
 \begin{align*}
      &\epsilon=\frac{4(m+2)m^2B_2\sqrt{2\sigma/\pi}/(1-mB_1)+3m^5B_2^2\sigma/(1-mB_1)^2}{8}\\
      &+(1-(1+mB_1)B_2mr)e^{-(r/2\sqrt{\sigma}-m/\sqrt{2\pi})^2/m}
      \end{align*}

\[\epsilon'=2^{-n/2}\cdot e^{2m ( B_0+2rB_1+2r^2B_2)/\sqrt{2\pi\sigma}}/(1-2mB_1-2rmB_2)+2e^{-(r/2\sqrt{\sigma}-m/\sqrt{2\pi})^2/m}\]
\[\epsilon''=e^{-(r/2\sqrt{\sigma}-m/\sqrt{2\pi})^2/m}\]

If $720m^4B^2\gamma^2/\pi n\ge 2$, then the conclusion of this corollary is uninterestingly true. Otherwise, $\epsilon\le 180m^4\gamma^2B^2/\pi n+\epsilon''$. Either way, $\epsilon'\le 4[e/4]^{n/4}+2\epsilon''$, and $\epsilon''\le e^{-m/2\pi}$. $m\ge n$ and $e^{1/2\pi}\ge [4/e]^{1/4}$, so $\epsilon''\le [e/4]^{n/4}$. The desired conclusion follows.
\end{proof}

That allows us to prove the following elaboration of theorem \ref{thm3}.

\begin{theorem}
Let $(h,G)$ be a neural net with $n$ inputs and $m$ edges, $G_W$ be $G$ with its edge weights changed to the elements of $W$, $L$ be a loss function, and $B>0$. Next, define $\gamma$ such that $0<\gamma\le \pi n/80m^2B$, and let $T$ be a positive integer. Then, for each $s\subseteq [n]$, let $\rho_s$ be the probability distribution of $(X,p_s(X))$ when $X$ is chosen randomly from $\B^n$. Now, select $S\subseteq[n]$ at random. Next, set $\sigma=\left(\frac{40m\gamma B}{n}\right)^2/2\pi$. Now, let $G'$ be $G$ with each of its edge weights perturbed by an independently generated variable drawn from $\mathcal{N}(0,\sigma I)$ and run 
\\$NoisyStochasticGradientDescentAlgorithm(h, G', P_{\mathcal{Z}}, L, \gamma, \infty, \mathcal{N}(0,[2mB\gamma-m^2B^2\gamma^2]\sigma I), T)$. 
Let $p$ be the probability that there exists $0\le i<T$ such that there exists a perturbation $G'_i$ of $G_i$ with no edge weight changed by more than $160m^2 \gamma B/\pi n$ such that one of the first three derivatives of $L(eval_{(h,G_i)}(X_i)-Y_i)$ with respect to the edge weights has magnitude greater than $B$.
For a random $X\in\B^n$, the probability that the resulting net computes $p_S(X)$ correctly is at most $1/2+2p+T(360m^4B^2\gamma^2/\pi n+7[e/4]^{n/4})$.
\end{theorem}

\begin{proof}
Let $Q'_s$ be the probability distribution of the resulting neural net given that $S=s$, and let $Q$ be the probability distribution of the net output by NoisyStochasticGradientDescentAlgorithm $(h, G', \star, L, \gamma, \infty, \mathcal{N}(0,[2mB\gamma-m^2B^2\gamma^2]\sigma I), T)$, where $\star$ is the uniform distribution on $\B^{n+1}$. Also, for each $(x,y)\in\B^{n+1}$, let $R_(x,y)$ be the set of all neural nets that output $y$ when given $x$ as input. The probability that the neural net in question computes $p_S(X)$ correctly is at most
\begin{align*}
&2^{-2n}\sum_{s\subseteq[n],x\in\B^n} P_{G\sim Q'_s}[(f,G)\in R_{x,p_s(y)}]\\
&\le 2^{-2n}\sum_{s\subseteq[n],x\in\B^n} P_{G\sim Q}[(f,G)\in R_{x,p_S(x)}]+||Q-Q'_s||_1/2\\ 
&\le 1/2+2p+T(360m^4B^2\gamma^2/\pi n+7[e/4]^{n/4})
\end{align*}
\end{proof}

\section{Proofs of positive results: universality of deep learning}\label{universality}

In previous sections, we were attempting to show that under some set of conditions a neural net trained by SGD is unable to learn a function that is reasonably learnable. However, there are some fairly reasonable conditions under which we actually can use a neural net trained by SGD to learn any function that is reasonably learnable. More precisely, we claim that given any probability distribution of functions from $\{0,1\}^n \to \{0,1\}$ such that there exists an algorithm that learns a random function drawn from this distribution with accuracy $1/2+\epsilon$ using a polynomial amount of time, memory, and samples, there exists a series of polynomial-sized neural networks that can be constructed in polynomial time and that can learn a random function drawn from this distribution with an accuracy of at least $1/2+\epsilon$ after being trained by SGD on a polynomial number of samples despite possibly poly-noise.

\subsection{Emulation of arbitrary algorithms}
Any algorithm that learns a function from samples must repeatedly get a new sample and then change some of the values in its memory in a way that is determined by the current values in its memory and the value of the sample. Eventually, it must also attempt to compute the function's output based on its input and the values in memory. If the learning algorithm is efficient, then there must be a polynomial-sized circuit that computes the values in the algorithm's memory in the next timestep from the sample it was given and its memory values in the current timestep. Likewise, there must be a polynomial-sized circuit that computes its guesses of the function's output from the function's input and the values in its memory.

Any polynomial-sized circuit can be translated into a neural net of polynomial size. Normally, stochastic gradient descent would tend to alter the weights of edges in that net, which might cause it to stop performing the calculations that we want. However, we can prevent its edge weights from changing by using an activation function that is constant in some areas, and ensuring that the nodes in the translated circuit always get inputs in that range. That way, the derivatives of their activation levels with respect to the weights of any of the edges leading to them are $0$, so backpropagation will never change the edge weights in the net. That leaves the issue of giving the net some memory that it can read and write. A neural net's memory takes the form of its edge weights. Normally, we would not be able to precisely control how stochastic gradient descent would alter these weights. However, it is possible to design the net in such a way that if certain vertices output certain values, then every path to the output through a designated edge will pass through a vertex that has a total input in one of the flat parts of the activation function. So, if those vertices are set that way the derivative of the loss function with respect to the edge weight in question will be $0$, and the weight will not change. That would allow us to control whether or not the edge weight changes, which gives us a way of setting the values in memory. As such, we can create a neural net that carries out this algorithm when it is trained by means of stochastic gradient descent with appropriate samples and learning rate. This net will contain the following components:\begin{enumerate}
    \item{\em The output vertex.} This is the output vertex of the net, and the net will be designed in such a way that it always has a value of $\pm 1$.

    \item {\em The input bits.} These will include the regular input vertices for the function in question. However, there will also be a couple of extra input bits that are to be set randomly in each timestep. They will provide a source of randomness that is necessary for the net to run randomized algorithms\footnote{Two random bits will always be sufficient because the algorithm can spend as many timesteps as it needs copying random bits into memory and ignoring the rest of its input.}, in addition to some other guesswork that will turn out to be necessary (see more on this below).

    \item {\em The memory component.} For each bit of memory that the original algorithm uses, the net will have a vertex with an edge from the constant vertex that will be set to either a positive or negative value depending on whether that bit is currently set to $0$ or $1$. Each such vertex will also have an edge leading to another vertex which is connected to the output vertex by two paths. The middle vertex in each of these paths will also have an edge from a control vertex. If the control vertex has a value of $2$, then that vertex's activation will be $0$, which will result in all  subsequent vertices on that path outputting $0$, and none of the edge weights on that path changing as a result of backpropagation along that path. On the other hand, if the control vertex has a value of $0$, then that vertex will have a nonzero activation, and so will all subsequent vertices on that path. The learning rate will be chosen so that in this case, if the net gives the wrong output, the weight of every edge on this path will be multiplied by $-1$. This will allow the computation component to set values in memory using the control vertices. (See definition \ref{memDef} and lemma \ref{memLem} for details on the memory component.)
    
    \item {\em The computation component.} This component will have edges leading to it from the inputs and from the memory component. It will use the inputs and the values in memory to compute what the net should output and what to set the memory bits to at the end of the current timestep if the net's output is wrong. There will be edges leading from the appropriate vertices in this component to the control vertices in the memory component in order to set the bits to the values it has computed. If the net's output is right, the derivative of the loss function with respect to any edge weight will be $0$, so the entire net will not change. This component will be constructed in such a way that the derivative of the loss function with respect to the weights of its edges will always be $0$. As a result, none of the edge weights in the computation component will ever change, as explained in lemma \ref{em1}. This component will also decide whether or not the net has learned enough about the function in question based on the values in memory. If it thinks that it still needs to learn, then it will have the net output a random value and attempt to set the values in memory to whatever they should be set to if that guess is wrong. If it thinks that it has learned enough, then it will try to get the output right and leave the values in memory unchanged.
    
    \begin{figure}
        \centering
        \includegraphics[width=16cm]{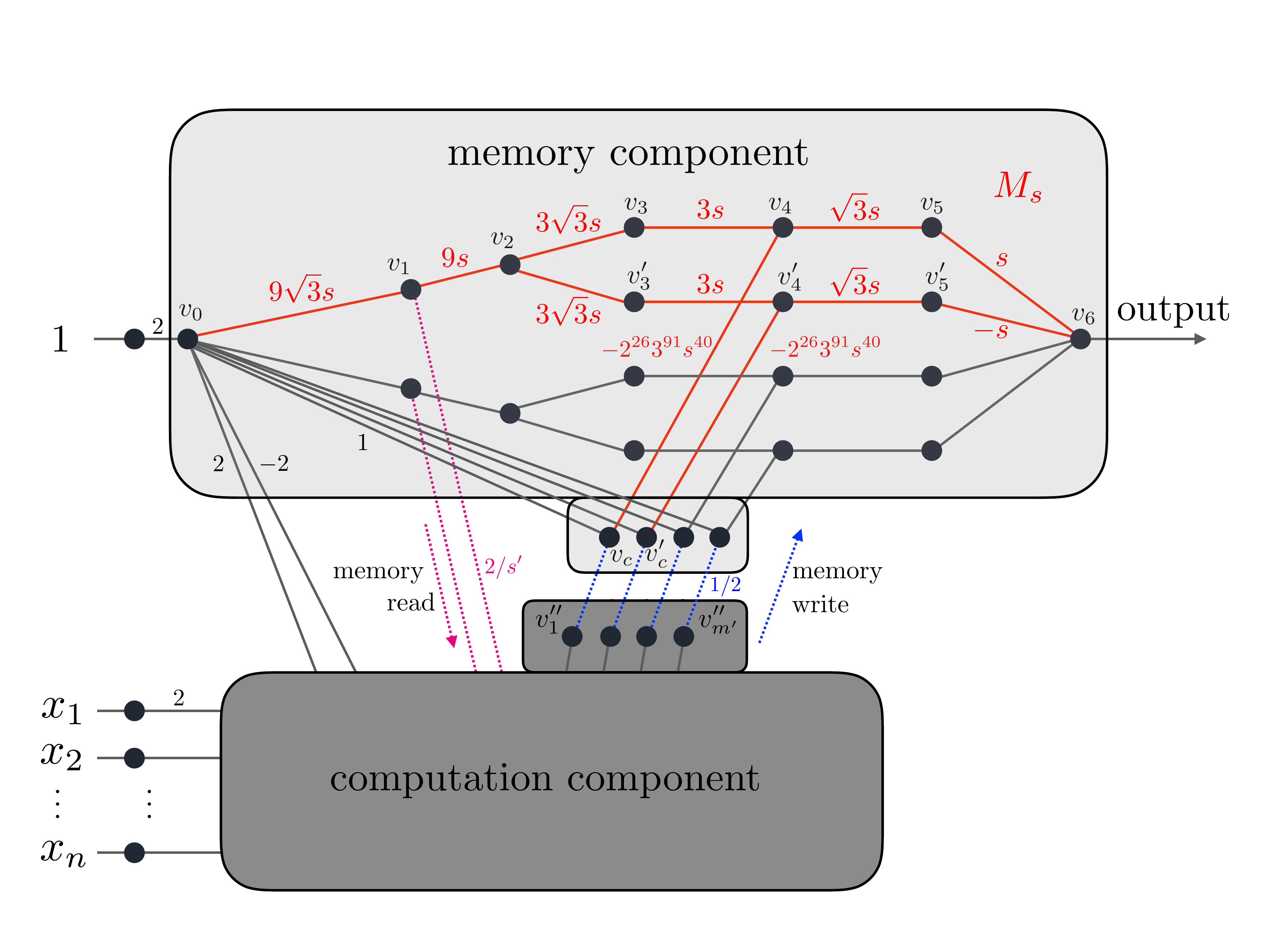}
        \caption{The emulation net. The parameters are $s=\sqrt[364]{2^{-243}3^{-1641/2}/m'}$, where $m'=\max(m,\lceil 2^{-243}3^{-1641/2}(18\sqrt{3})^{364}\rceil)$, $s'=(18\sqrt{3}s)^3$ and $m$ is the total number of bits required to perform the computation from the computation component. In this illustration, we considered only two copies of the $M_s$ from Definition \ref{memDef}; one copy is highlighted in red. The magenta dashed edges are the memory read edges and the blue dashed edges are the memory write edges. The latter allow to change the controller vertices $v_c,v_c'$ that act on $M_s$ to edit the memory. Random bit inputs are omitted in this figure.}
        \label{emulationDiagram}
    \end{figure}
\end{enumerate}

See Figure \ref{emulationDiagram} for a representation of the overall net. 
One complication that this approach encounters is that if the net outputs the correct value, then the derivative of the loss function with respect to any edge weight is $0$, so the net cannot learn from that sample.\footnote{This holds for any loss function that has a minimum when the output is correct, not just the $L_2$ loss function that we are using. We could avoid this by having the net output $\pm 1/2$ instead of $\pm 1$. However, if we did that then the change in each edge weight if the net got the right output would be $-1/3$ of the change in that edge weight if it got the wrong output, which would be likely to result in an edge weight that we did not want in at least one of those cases. There are ways to deal with that, but they do not seem clearly preferable to the current approach.} Our approach to dealing with that is to have a learning phase where we guess the output randomly and then have the net output the opposite of our guess. That way, if the guess is right the net learns from that sample, and if it is wrong it stays unchanged. Each guess is right with probability $1/2$ regardless of the sample, so the probability distribution of the samples it is actually learning from is the same as the probability distribution of the samples overall, and it only needs $(2+o(1))$ times as many samples as the original algorithm in order to learn the function. Once it thinks it has learned enough, such as after learning from a designated number of samples, it can switch to attempting to compute the function it has learned on each new input.

\begin{example}
We now give an illustration of how previous components would  run and interact for learning parities. One can learn an unknown parity function by collecting samples until one has a set that spans the space of possible inputs, at which point one can compute the function by expressing any new input as a linear combination of those inputs and returning the corresponding linear combination of their outputs. As such, if we wanted to design a neural net to learn a parity function this way, the memory component would have $n(n+1)$ bits designated for remembering samples, and $\log_2(n+1)$ bits to keep a count of the samples it had already memorized. Whenever it received a new input $x$, the computation component would get the value of $x$ from the input nodes and the samples it had previously memorized, $(x_1,y_1),...,(x_r,y_r)$, from the memory component. Then it would check whether or not $x$ could be expressed as a linear combination of $x_1,...,x_r$. If it could be, then the computation component would compute the corresponding linear combination of $y_1,...,y_r$ and have the net return it. Otherwise, the computation component would take a random value that it got from one of the extra input nodes, $y'$. Then, it would attempt to have the memory component add $(x,y')$ to its list of memorized samples and have the net return $NOT(y')$. That way, if the correct output was $y'$, then the net would return the wrong value and the edge weights would update in a way that added the sample to the net's memory. If the correct output was $NOT(y')$, then the net would return the right value, and none of the edge weights would change. As a result, it would need about $2n$ samples before it succeeded at memorizing a list that spanned the space of all possible inputs, at which point it would return the correct outputs for any subsequent inputs.
\end{example}


Before we can prove anything about how our net learns, we will need to establish some properties of our activation function. Throughout this section, we will use an activation function $f:\mathbb{R}\rightarrow\mathbb{R}$ such that $f(x)=2$ for all $x>3/2$, $f(x)=-2$ for all $x<-3/2$, and $f(x)=x^3$ for all $-1<x<1$. There is a way to define $f$ on $[-3/2,-1]\cup[1,3/2]$ such that $f$ is smooth and nondecreasing. The details of how this is done will not affect any of our arguments, so we pick some assignment of values to $f$ on these intervals with these properties. This activation function has the important property that its derivative is $0$ everywhere outside of  $[-3/2,3/2]$. As a result, if we use SGD to train a neural net using this activation function, then in any given time step, the weights of the edges leading to any vertex that had a total input that is not in $[-3/2,3/2]$ will not change. This allows us to create sections of the net that perform a desired computation without ever changing. In particular, it will allow us to construct the net's computation component in such a way that it will perform the necessary computations without ever getting altered by SGD. More formally, we have the following.

\begin{lemma}[Backpropagation-proofed circuit emulation]\label{em1}
Let $h:\{0,1\}^m\rightarrow\{0,1\}^{m'}$ be a function that can be computed by a circuit made of AND, OR, and NOT gates with a total of $b$ gates. Also, consider a neural net with $m$ input\footnote{Note that these will not be the input of the general neural net that is being built.} vertices $v'_1,...,v'_m$, and a collection of chosen real numbers $y^{(0)}_1<y^{(1)}_1,y^{(0)}_2<y^{(1)}_2,...,y^{(0)}_m<y^{(1)}_m$. It is possible to add a set of at most $b$ new vertices to the net, including output vertices $v''_1,...,v''_{m'}$, along with edges leading to them such that for any possible addition of edges leading from the new vertices to old vertices, if the net is trained by SGD and the output of $v'_i$ is either $y^{(0)}_i$ or $y^{(1)}_i$ for every $i$ in every timestep, then the following hold:
\begin{enumerate}
    \item None of the weights of the edges leading to the new vertices ever change, and no paths through the new vertices contribute to the derivative of the loss function with respect to edges leading to the $v'_i$.
    
    \item In any given time step, if the output of $v'_i$ encodes $x_i$ with $y^{(0)}_i$ and $y^{(1)}_i$ representing $0$ and $1$ respectively for each $i$\footnote{It would be convenient if $v'_1,...,v'_m$ all used the same encoding. However, the computation component will need to get inputs from the net's input vertices and from the memory component. The input vertices encode $0$ and $1$ as $\pm 1$, while the memory component encodes them as $\pm s'$ for some small $s'$. Therefore, it is necessary to be able to handle inputs that use different encodings.}, then the output of $v''_j$ encodes $h_j(x_1,...,x_m)$ for each $j$ with $-2$ and $2$ encoding $0$ and $1$ respectively.
\end{enumerate}
\end{lemma}

\begin{proof}
In order to do this, add one new vertex for each gate in a circuit that computes $h$. When the new vertices are used to compute $h$, we want each vertex to output $2$ if the corresponding gate outputs a $1$ and $-2$ if the corresponding gate outputs a $0$. In order to make one new vertex compute the NOT of another new vertex, it suffices to have an edge of weight $-1$ to the vertex computing the NOT and no other edges to that vertex. We can compute an AND of two new vertices by having a vertex with two edges of weight $1$ from these vertices and an edge of weight $-2$ from the constant vertex. Similarly, we can compute an OR of two new vertices by having a vertex with two edges of weight $1$ from these vertices and an edge of weight $2$ from the constant vertex. For vertices corresponding to gates that act directly on the inputs, we have the complication that their vertices do not necessarily encode $0$ and $1$ as $\pm 2$, but we can compensate for that by changing the weights of the edges from these vertices, and the edges to these gates from the constant vertices appropriately. 

This ensures that if the outputs of the $v'_i$ encode binary values $x_1,...,x_m$ appropriately, then each of the new vertices will output the value corresponding to the output of the appropriate gate. So, these vertices compute $h(x_1,...,x_m)$ correctly. Furthermore, since the input to each of these vertices is outside of $[-3/2,3/2]$, the derivatives of their activation functions with respect to their inputs are all $0$. As such, none of the weights of the edges leading to them ever change, and paths through them do not contribute to changes in the weights of edges leading to the $v'_i$.
\end{proof}

Note that any efficient learning algorithm will have a polynomial number of bits of memory. In each time step, it might compute an output from its memory and sample input, and it will compute which memory values it should change based on its memory, sample input, and sample output. All of these computations must be performable in polynomial time, so there is a polynomial sized circuit that performs them. Therefore, by the lemma it is possible to add a polynomial sized component to any neural net that performs these calculations, and as long as the inputs to this component always take on values corresponding to $0$ or $1$, backpropagation will never alter the weights of the edges in this component. That leaves the issue of how the neural net can encode and update memory bits. Our plan for this is to add in a vertex for each memory bit that has an edge with a weight encoding the bit leading to it from a constant bit and no other edges leading to it. We will also add in paths from these vertices to the output that are designed to allow us to control how backpropagation alters the weights of the edges leading to the memory vertices. More precisely, we define the following.

\begin{definition}\label{memDef}
For any positive real number $s$, let $M_s$ be the weighted directed graph with 12 vertices, $v_0$, $v_1$, $v_2$, $v_3$, $v_4$, $v_5$, $v_c$, $v'_3$, $v'_4$, $v'_5$, $v'_c$, and $v_6$ and the following edges:
\begin{enumerate}
    \item An edge of weight $3^{3-t/2} s$ from $v_{t-1}$ to $v_t$ for each $0<t\le 6$
    
    \item An edge of weight $3\sqrt{3} s$ from $v_2$ to $v'_3$
    
    \item An edge of weight $3^{3-t/2} s$ from $v'_{t-1}$ to $v'_t$ for each $3<t< 6$
    
    \item An edge of weight $-s$ from $v'_5$ to $v_6$
    
    \item An edge of weight $-2^{26}\cdot 3^{91}s^{40}$ from $v_c$ to $v_4$.
    
    \item An edge of weight $-2^{26}\cdot 3^{91}s^{40}$ from $v'_c$ to $v'_4$.
\end{enumerate}
\end{definition}

We refer to Figure \ref{emulationDiagram} to visualize $M_s$.
The idea is that this structure can be used to remember one bit, which is encoded in the current weight of the edge from $v_0$ to $v_1$. A weight of $9\sqrt{3} s$ encodes a $0$ and a weight of $-9\sqrt{3} s$ encodes a $1$. In order to set the value of this bit, we will use $v_c$ and $v'_c$, which will be controlled by the computation component. If we want to keep the bit the same, then we will have them both output $2$, in which case $v_4$ and $v'_4$ will both output $0$, with the result that the derivative of the loss function with respect to any of the edge weights in this structure will be $0$. However, if we want to change the value of this bit, we will have one of $v_c$ and $v'_c$ output $0$. That will result in a nonzero output from $v_4$ or $v'_4$, which will lead to the net's output having a nonzero derivative with respect to some of the edge weights in this structure. Then, if the net gives the wrong output, the weights of some of the edges in the structure will be multiplied by $-1$, including the weight of the edge from $v_0$ to $v_1$. Unfortunately, if the net gives the right output then the derivative of the loss function with respect to any edge weight will be $0$, which means that any attempt to change a value in memory on that timestep will fail.

More formally, we have the following.

\begin{lemma}[Editing memory when the net gives the wrong output]\label{memLem}
Let $0<s<1/18\sqrt{3}$, $\gamma=2^{-244}\cdot 3^{-1643/2}s^{-362}$, and $L(x)=x^2$ for all $x$. Also, let $(f,G)$ be a neural net such that $G$ contains $M_s$ as a subgraph with $v_6$ as $G$'s output vertex, and there are no edges from vertices outside this subgraph to vertices in the subgraph other than $v_0$, $v_c$, and $v'_c$. Now, assume that this neural net is trained using SGD with learning rate $\gamma$ and loss function $L$ for $t$ time steps, and the following hold:
\begin{enumerate}
    \item The sample output is always $\pm 1$.
    
    \item The net gives an output of $\pm 1$ in every time step.
    
    \item $v_0$ outputs $2$ in every time step.
    
    \item $v_c$ and $v'_c$ each output $0$ or $2$ in every time step.
    
    \item $v'_c$ outputs $2$ in every time step when the net outputs $1$ and $v_c$ outputs $2$ in every time step when the net outputs $-1$.
    
    \item The derivatives of the loss function with respect to the weights of all edges leaving this subgraph are always $0$.
\end{enumerate}
Then during the training process, the weight of the edge from $v_0$ to $v_1$ is multiplied by $-1$ during every time step when the net gives the wrong output and $v_c$ and $v'_c$ do not both output $2$, and its weight stays the same during all other time steps.
\end{lemma}

\begin{proof}
More precisely, we claim that the weight of the edge from $v_c$ to $v_4$ and the weight of the edge from $v'_c$ to $v'_4$ never change, and that all of the other edges in $M_s$ only ever change by switching signs. Also, we claim that at the end of any time step, either all of the edges on the path from $v_0$ to $v_2$ have their original weights, or all of them have weights equal to the negatives of their original weights. Furthermore, we claim that the same holds for the edges on each path from $v_2$ to $v_6$.

In order to prove this, we induct on the number of time steps. It obviously holds after $0$ time steps. Now, assume that it holds after $t'-1$ time steps, and consider time step $t'$. If the net gave the correct output, then the derivative of the loss function with respect to the output is $0$, so none of the weights change. 

Now, consider the case where the net outputs $1$ and the correct output is $-1$. By assumption, $v'_c$ outputs $2$ in this time step, so $v'_4$ gets an input of $2^{27}\cdot 3^{91}s^{40}$ from $v'_3$ and an input of $-2^{27}\cdot 3^{91}s^{40}$ from $v'_c$. So, both its output and the derivative of its output with respect to its input are $0$. That means that the same holds for $v_5$, which means that none of the edge weights on this path from $v_2$ to $v_6$ change this time step, and nothing backpropagates through this path. If $v_c$ also outputs $2$, then $v_4$ and $v_5$ output $0$ for the same reason, and none of the edge weights in this copy of $M_s$ change. On the other hand, if $v_c$ outputs $0$, then the output vertex gets an input of $2^{243}\cdot 3^{1641/2} s^{364}$ from $v_5$. The derivative of this input with respect to the weight of the edge from $v_{i-1}$ to $v_{i}$ is $2^{243}\cdot 3^{1641/2} s^{364}\cdot[3^{6-i}/ (3^{3-i/2} s)]$ if these weights are positive, and the negative of that if they are negative. Furthermore, the derivative of the loss function with respect to the input to the output vertex is $12$. So, the algorithm reduces the weights of all the edges on the path from $v_0$ to $v_6$ that goes through $v_4$ exactly enough to change them to the negatives of their former values. Also, since $v_c$ output $0$, the weight of the edge from $v_c$ to $v_4$ had no effect on anything this time step, so it stays unchanged.

The case where the net outputs $-1$ and the correct output is $1$ is analogous, with the modification that the output vertex gets an input of $-2^{243}\cdot 3^{1641/2} s^{364}$ from $v'_5$ if $v'_c$ outputs $0$ and the edges on the path from $v_0$ to $v_6$ that goes through $v'_4$ are the ones that change signs. So, by induction, the claimed properties hold at the end of every time step. Furthermore, this argument shows that the sign of the edge from $v_0$ to $v_1$ changes in exactly the time steps where the net outputs the wrong value and $v_c$ and $v'_c$ do not both output $2$.
\end{proof}

So, $M_s$ satisifes some but not all of the properties we would like a memory component to have. We can read the bit it is storing, and we can control which time steps it might change in by controlling the inputs to $v_c$ and $v'_c$. However, for it to work we need the output of the overall net to be $\pm 1$ in every time step, and each such memory component will input $\pm 2^{243}\cdot 3^{1641/2}s^{364}$ to the output vertex every time we try to flip it. More problematically, the values these components are storing can only change when the net gets the output wrong. We can deal with the first issue by choosing parameters such that $2^{243}\cdot 3^{1641/2}s^{364}$ is the inverse of an integer that is at least as large as the number of bits that we want to remember, and then adding some extraneous memory components that we can flip in order to ensure that exactly $1/2^{243}\cdot 3^{1641/2}s^{364}$ memory components get flipped in each time step. We cannot change the fact that the net will not learn from samples where it got the output right, but we can use this to emulate any efficient learning algorithm that only updates when it gets something wrong. More formally, we have the following.

\begin{lemma}
For each $n$, let $m_n$ be polynomial in $n$, and $h_n:\{0,1\}^{n+m_n}\rightarrow\{0,1\}$ and $g_n:\{0,1\}^{n+m_n}\rightarrow\{0,1\}^{m_n}$ be functions that can be computed in polynomial time. Then there exists a neural net $(G_n,f)$ of polynomial size and $\gamma>0$ such that the following holds. Let $T>0$ and $(x_t,y_t)\in\{0,1\}^n\cdot\{0,1\}$ for each $0<t\le T$. Then, let $b_0=(0,...,0)$, and for each $0<t\le T$, let $y^\star_t=h_n(x_t,b_{t-1})$ and let $b_t$ equal $b_{t-1}$ if $y^\star_t=y_t$ and $g_n(x_t,b_{t-1})$ otherwise. Then if we use stochastic gradient descent to train $(G_n,f)$ on the samples $(2x_t-1,2y_t-1)$ with a learning rate of $\gamma$, the net outputs $1$ in every time step where $y^\star_t=1$ and $-1$ in every time step where $y^\star_t=0$.
\end{lemma}

\begin{proof}
First, let $m'=\max(m,\lceil 2^{-243}3^{-1641/2}(18\sqrt{3})^{364}\rceil)$, and $s=\sqrt[364]{2^{-243}3^{-1641/2}/m'}$. Then, set $\gamma=2^{-244}\cdot 3^{-1643/2}s^{-362}$.

We construct $G_n$ as follows. First, we take $m+m'$ copies of $M_s$, merge all of the copies of $v_6$ to make an output vertex, and merge all of the copies of $v_0$. Then we add in $n$ input vertices and a constant vertex and add an edge of weight $2$ from the constant vertex to $v_0$. Next, define $r:\{0,1\}^{n+m}\rightarrow \{0,1\}^{1+2m+2m'}$ such that given $x\in \{0,1\}^n$ and $b\in\{0,1\}^m$, $r(x,b)$ lists $h_n(x,b)$ and one half the values of the $v_c$ and $v'_c$ necessary to change the values stored by the first $m$ memory units in the net from $b$ to $g_n(x,b)$ and then flip the next $m'-|\{i: b_i\ne (g_n(x,b))_i\}|$ provided the net outputs $2 h_n(x,b)-1$. Then, add a section to the net that computes $r$ on the input bits and the bits stored in the first $m$ memory units, and connect each copy of $v_c$ or $v'_c$ to the appropriate output by an edge of weight $1/2$ and the constant bit by an edge of weight $1$.

In order to show that this works, first observe that since $h_n$ and $g_n$ can be computed efficiently, so can $r$. So, there exists a polynomial sized subnet that computes it correctly by lemma \ref{em1}. That lemma also shows that this section of the net will never change as long as all of the inputs and all of the memory bits encode $0$ or $1$ in every time step. Similarly, in every time step $v_0$ will have an input of $2$ and all of the copies of $v_c$ and $v'_c$ will have inputs of $0$ or $2$. So, the derivatives of their outputs with respect to their inputs will be $0$, which means that the weights of the edges leading to them will never change. That means that the only edges that could change in weight are those in the memory components. In each time step, $m'$ memory components each contribute $(2 h_n(x_t,b_{t-1})-1)/m'$ to the output vertex, so it takes on a value of $(2 h_n(x_t,b_{t-1})-1)$, assuming that the memory components were storing $b_{t-1}$ like they were supposed to. As such, the net outputs $y^\star_t$, the memory bits stay the same if $y^\star_t=y_t$, and the first $m$ memory bits get changed to $g_n(x_t,b_{t-1})$ otherwise with some irrelevant changes to the rest. Therefore, by induction on the time step, this net performs correctly on all time steps.
\end{proof}

\begin{remark}
With the construction in this proof, $m'$ will always be at least $10^{79}$, which ensures that this net will be impractically large. This is a result of the fact that the only edges going to the output vertex are those contained in the memory component, and the paths in the memory component take a small activation and repeatedly cube it. If we had chosen an activation function that raises its input to the $\frac{11}{9}$ when its absolute value was less than $1$ instead of cubing it, the minimum possible value of $m'$ would have been on the order of $1000$.
\end{remark}

In other words, we can train a neural net with SGD in order to duplicate any efficient algorithm that takes $n$ bits as input, gives $1$ bit as output, and only updates its memory when its output fails to match some designated ``correct" output. The only part of that that is a problem is the restriction that it can not update its memory in steps when it gets the output right. As a result, the probability distribution of the samples that the net actually learns from could be different from the true probability distribution of the samples. We do not know how an algorithm that we are emulating will behave if we draw its samples from a different probability distribution, so this could cause problems. Our solution to that will be to have a training phase where the net gives random outputs so that it will learn from each sample with probability $1/2$, and then switch to attempting to compute the actual correct output rather than learning. That allows us to prove the following (re-statement of Theorem \ref{thm_univ}).


\begin{theorem}
For each $n>0$, let $P_\X$ be a probability measure on $\{0,1\}^n$, and $P_{\F}$ be a probability measure on the set of functions from $\{0,1\}^n$ to $\{0,1\}$. Also, let $B_{1/2}$ be the uniform distribution on $\{0,1\}$. Next, define $\alpha_n$ such that there is some algorithm that takes a polynomial number of samples $(x_i,F(x_i))$ where the $x_i$ are independently drawn from $P_{\X}$ and $F\sim P_{\F}$, runs in polynomial time, and learns $(P_\F,P_\X)$ with accuracy $\alpha$. Then there exists $\gamma_n>0$, a polynomial-sized neural net $(G_n,f)$, and a polynomial $T_n$ such that using stochastic gradient descent with learning rate $\gamma_n$ and loss function $L(x)=x^2$ to train $(G_n,f)$ on $T_n$ samples $((2x_i-1,2r_i-1,2r'_i-1), F(x_i))$ where $(x_i,r_i,r'_i)\sim P_\X\times B_{1/2}^2$ learns $(P_\F,P_\X)$ with accuracy $\alpha-o(1)$.
\end{theorem}

\begin{proof}
We can assume that the algorithm counts the samples it has received, learns from the designated number, and then stops learning if it receives additional samples. The fact that the algorithm learns in polynomial time also means that it can only update a polynomial number of locations in memory, so it only needs a polynomial number of bits of memory, $m_n$. Also, its learning process can be divided into steps which each query at most one new sample $(x_i,F(x_i))$ and one new random bit. So, there must be an efficiently computable function $A$ such that if $b$ is the value of the algorithm's memory at the start of a step, and it receives $(x_i,y_i)$ as its sample (if any) and $r_i$ as its random bit (if any), then it ends the step with its memory set to $A(b,x_i,y_i,r_i)$. 

Now, define $A':\{0,1\}^{m_n+n+3}\rightarrow \{0,1\}^{m_n}$ such that
\[A'(b,x,y,r,r')=
\begin{cases}
b&\text{ if } y=r'\\
A(b,x,y,r)&\text{ if } y\ne r'
\end{cases}
\]
Next, let $b_0$ be the initial state of the algorithm's memory, and consider setting $b_i=A'(b_{i-1},x_i,F(x_i),r_i,r'_i)$ for each $i>0$. We know that $r'_i$ is equally likely to be $0$ or $1$ and independent of all other components, so $b_i$ is equal to $A(b_{i-1},x_i,F(x_i),r_i)$ with probability $1/2$ and $b_{i-1}$ otherwise. Furthermore, the probability distribution of $(b_{i-1},x_i,F(x_i),r_i)$ is independent of whether or not $y_i=r'_i$. Also, if we set $b'=b_0$ and then repeatedly replace $b'$ with $A(b',x,F(x),r)$, then there is some polynomial number of times we need to do that before $b'$ stops changing because the algorithm has enough samples and is no longer learning. So, with probability $1-o(1)$, the value of $b_i$ will stabilize by the time it has received $n$ times that many samples. Furthermore, the probability distribution of the value $b_i$ stabilizes at is exactly the same as the probability distribution of the value the algorithm's memory stabilizes at because the probability distribution of tuples $(b_{i-1},x_i,F(x_i))$ that actually result in changes to $b_i$ is exactly the same as the overall probability distribution of $(b_{i-1},x_i,F(x_i))$. So, given the final value of $b_i$, one can efficiently compute $F$ with an expected accuracy of at least $\alpha$.

Now, let $\overline{A}(b,x)$ be the value the algorithm outputs when trying to compute $F(x)$ if its memory has a value of $b$ after training. Then, define $A''$ such that
\[A''(b,x,r,r')=
\begin{cases}
\overline{A}(b,x) &\text{ if b corresponds to a memory state resulting from training on enough samples} \\
r' &\text{ otherwise }
\end{cases}
\]
By the previous lemma, there exists a polynomial sized neural net $(G_n,f)$ and $\gamma_n>0$ such that if we use SGD to train $(G_n,f)$ on $((2x_i-1,2r_i-1,2r'_i-1), F(x_i))$ with a learning rate of $\gamma_n$ then the net outputs $2A''(b_{i-1},x_i,r_i,r'_i)-1$ for all $i$. By the previous analysis, that means that after a polynomial number of steps, the net will compute $F$ with an expected accuracy of $\alpha-o(1)$.
\end{proof}

\begin{remark}
This net uses two random bits because it needs one in order to randomly choose outputs during the learning phase and another to supply randomness in order to emulate randomized algorithms. If we let $m$ be the minimum number of gates in a circuit that computes the algorithm's output and the contents of its memory after the current timestep from its input, its current memory values, and feedback on what the correct output was, then the neural net in question will have $\theta(m)$ vertices and $\gamma_n=\theta(m^{362/364})$. If the algorithm that we are emulating is deterministic, then $T_n$ will be approximately twice the number of samples the algorithm needs to learn the function; if it is randomized it might need a number of additional samples equal to approximately twice the number of random bits the algorithm needs.
\end{remark}

So, for any distribution of functions from $\{0,1\}^n$ to $\{0,1\}$ that can be learned in polynomial time, there is a neural net that learns it in polynomial time when it is trained by SGD.

\begin{remark}\label{kolmogorov}
    This theorem shows that each efficiently learnable $(P_\F,P_\X)$ has some neural net that learns it efficiently. However, instead of emulating an algorithm chosen for a specific distribution, we can use a ``Kolmogorov complexity'' trick and emulate a metaalgorithm such as the following:
    
    GeneralLearningMetaalgorithm(c):
    \begin{enumerate}
        \item List every algorithm that can be written in at most $\log(\log(n))$ bits.
        
        \item Get $n^c$ samples from the target distribution, and train each of these algorithms on them in parallel. If any of these algorithms takes more than $n^c$ time steps on any sample, then interrupt it and skip training it on that sample.
        
        \item Get $n^c$ more samples, have all of the aforementioned algorithms attempt to compute the function on each of them, and record which of them was most accurate. Again, if any of them take more than $n^c$ steps on one of these samples, interrupt it and consider it as having computed the function incorrectly on that sample.
        
        \item Return the function that resulted from training the most accurate algorithm.
    \end{enumerate}
    
    Given any $(P_\F,P_\X)$ that is efficiently learnable, there exist $\epsilon,c>0$ such that there is some algorithm that learns $(P_\F,P_\X)$ with accuracy $1/2+\epsilon-o(1)$, needs at most $n^c$ samples in order to do so, and takes a maximum of $n^c$ time steps on each sample. For all sufficiently large $n$, this algorithm will be less than $\log(\log(n))$ bits long, so generalLearningMetaalgorithm(c) will consider it. There are only $O(\log(n))$ algorithms that are at most $\log(\log(n))$ bits long, so in the testing phase all of them will have observed accuracies within $O(n^{-c/2}\log(n))$ of their actual accuracies with high probability. That means that the function that generalLearningMetaalgorithm(c) judges as most accurate will be at most $O(n^{-c/2}\log(n))$ less accurate than the true most accurate function considered. So, generalLearningMetaalgorithm(c) learns $(P_\F,P_\X)$ with accuracy $1/2+\epsilon-o(1)$. A bit more precisely, this shows that for any efficiently learnable $(P_\F,P_\X)$, there exists $C_0$ such that for all $c>C_0$, generalLearningMetaalgorithm(c) learns $(P_\F,P_\X)$. 
    
    Now, if we let $(f,G_c)$ be a neural net emulating generalLearningMetaalgorithm(c), then $(f,G_c)$ has polynomial size and can be constructed in polynomial time for any fixed $c$. Any efficiently learnable $(P_\F,P_\X)$ can be learned by training $(f,G_c)$ with stochastic gradient descent with the right $c$ and the right learning rate, assuming that random bits are appended to the input. Furthermore, the only thing we need to know about $(P_\F,P_\X)$ in order to choose the net and learning rate is some upper bound on the number of samples and amount of time needed to learn it.
\end{remark}

\begin{remark}
The previous remark shows that for any $c>0$, there is a polynomial sized neural net that learns any $(P_\F,P_\X)$ that can be learned by an algorithm that uses $n^c$ samples and $n^c$ time per sample. However, that is still more restrictive than we really need to be. It is actually possible to build a net that learns any $(P_\F,P_\X)$ that can be efficiently learned using $n^c$ memory, and then computed in $n^c$ time once the learning process is done. In order to show this, first observe that any learning algorithm that spends more than $n^c$ time on each sample can be rewritten to simply get a new sample and ignore it after every $n^c$ steps. That converts it to an algorithm that spends $n^c$ time after receiving each sample while multiplying the number of samples it needs by an amount that is at most polynomial in $n$. 

The fact that we do not know how many samples the algorithm needs can be dealt with by modifying the metaalgorithm to find the algorithm that performs best when trained on $1$ sample, then the algorithm that performs best when trained on $2$, then the algorithm that performs best when trained on $4$, and so on. That way, after receiving any number of samples, it will have learned to compute the function with an accuracy that is within $o(1)$ of the best accuracy attainable after learning from $1/4$ that number of samples. The fact that we do not know how many samples we need also renders us unable to have a learning phase, and then switch to attempting to compute the function accurately after we have seen enough samples. Instead, we need to have it try to learn from each sample with a gradually decreasing probability and try to compute the function otherwise. For instance, consider designing the net so that it keeps a count of exactly how many times it has been wrong. Whenever that number reaches a perfect square, it attempts to learn from the next sample; otherwise, it tries to compute the function on that input. If it takes the metaalgorithm $n^{c'}$ samples to learn the function with accuracy $1-\epsilon$, then it will take this net roughly $n^{2c'}$ samples to learn it with the same accuracy, and by that point the steps where it attempts to learn the function rather than computing it will only add another $o(1)$ to the error rate. So, if there is any efficient algorithm that learns $(P_\F,P_\X)$ with $n^c$ memory and computes it in $n^c$ time once it has learned it, then this net will learn it efficiently.
\end{remark}

\subsection{Noisy emulation of arbitrary algorithms}

So far, our discussion of emulating arbitrary learning algorithms using SGD has assumed that we are using SGD without noise. It is of particular interest to ask whether there are efficiently learnable functions that noisy SGD can never learn with inverse-polynomial noise, as GD or SQ algorithms break in such cases (for example for parities). It turns out that the emulation argument can be adapted to sufficiently small amounts of noise. The computation component is already fairly noise tolerant because the inputs to all of its vertices will normally always have absolute values of at least $2$. If these are changed by less than $1/2$, these vertices will still have activations of $\pm 2$ with the same signs as before, and the derivatives of their activations with respect to their inputs will remain $0$.

However, the memory component has more problems handling noise. In the noise-free case, whenever we do not want the value it stores to change, we arrange for some key vertices inside the component to receive input $0$ so that their outputs and the derivatives of their outputs with respect to their inputs will both be $0$. However, once we start adding noise we will no longer be able to ensure that the inputs to these vertices are exactly $0$. This could result in a feedback loop where the edge weights shift faster and faster as they get further from their desired values. In order to avoid this, we will use an activation function designed to have output $0$ whenever its input is sufficiently close to $0$. More precisely, in this section we will use an activation function $f^\star:\mathbb{R}\rightarrow\mathbb{R}$ chosen so that $f^\star(x)=0$ whenever $|x|\le 2^{-121}3^{-9}$, $f^\star(x)=x^3$ whenever $2^{-120}3^{-9}\le |x|\le 1$, and $f^\star(x)=2\text{ sign}(x)$ whenever $|x|\ge 3/2$. There must be a way to define $f^\star$ on the remaining intervals such that it is smooth and nondecreasing. The details of how this is done will not affect out argument, so we pick some such assignment.

The memory component also has trouble handling bit flips when there is noise. Any time we flip a bit stored in memory, any errors in the edge weights of the copy of $M_s$ storing that bit are likely to get worse. As a result, making the memory component noise tolerant requires a fairly substantial redesign. First of all, in order to prevent perturbations in its edge weights from being amplified until they become major problems, we will only update each value stored in memory once. That still leaves the issue that due to errors in the edge weights, we cannot ensure that the output of the net is exactly $\pm 1$. As a result, even if the net gets the output right, the edge weights will still change somewhat. That introduces the possibility that multiple unsuccessful attempts at flipping a bit in memory will eventually cause major distortions to the corresponding edge weights. In order to address that, we will have our net always give an output of $1/2$ during the learning phase so that whenever we try to change a value in memory, it will change significantly regardless of what the correct output is. Of course, that leaves each memory component with $3$ possible states, the state it is in originally, the state it changes to if the correct output is $1$, and the state it changes to if the correct output is $-1$. More precisely, each memory value will be stored in a copy of the following.

\begin{definition}
Let $M'$ be the weighted directed graph with $9$ vertices, $v_0$, $v_1$, $v_2$, $v_3$, $v_4$, $v_5$, $v_c$, $v'_c$, and $v_r$ and the following edges:
\begin{enumerate}
    \item An edge of weight $3^{-t/2}/4$ from $v_t$ to $v_{t+1}$ for each $t$
    
    \item An edge of weight $128$ from $v_1$ to $v_r$
    
    \item An edge of weight $-2^{-81}\cdot 3^{-9}$ from $v_c$ to $v_4$
    
    \item An edge of weight $-2^{-41}\cdot 3^{-9}$ from $v'_c$ to $v_4$
\end{enumerate}
\end{definition}

    \begin{figure}
        \centering
        \includegraphics[width=11cm]{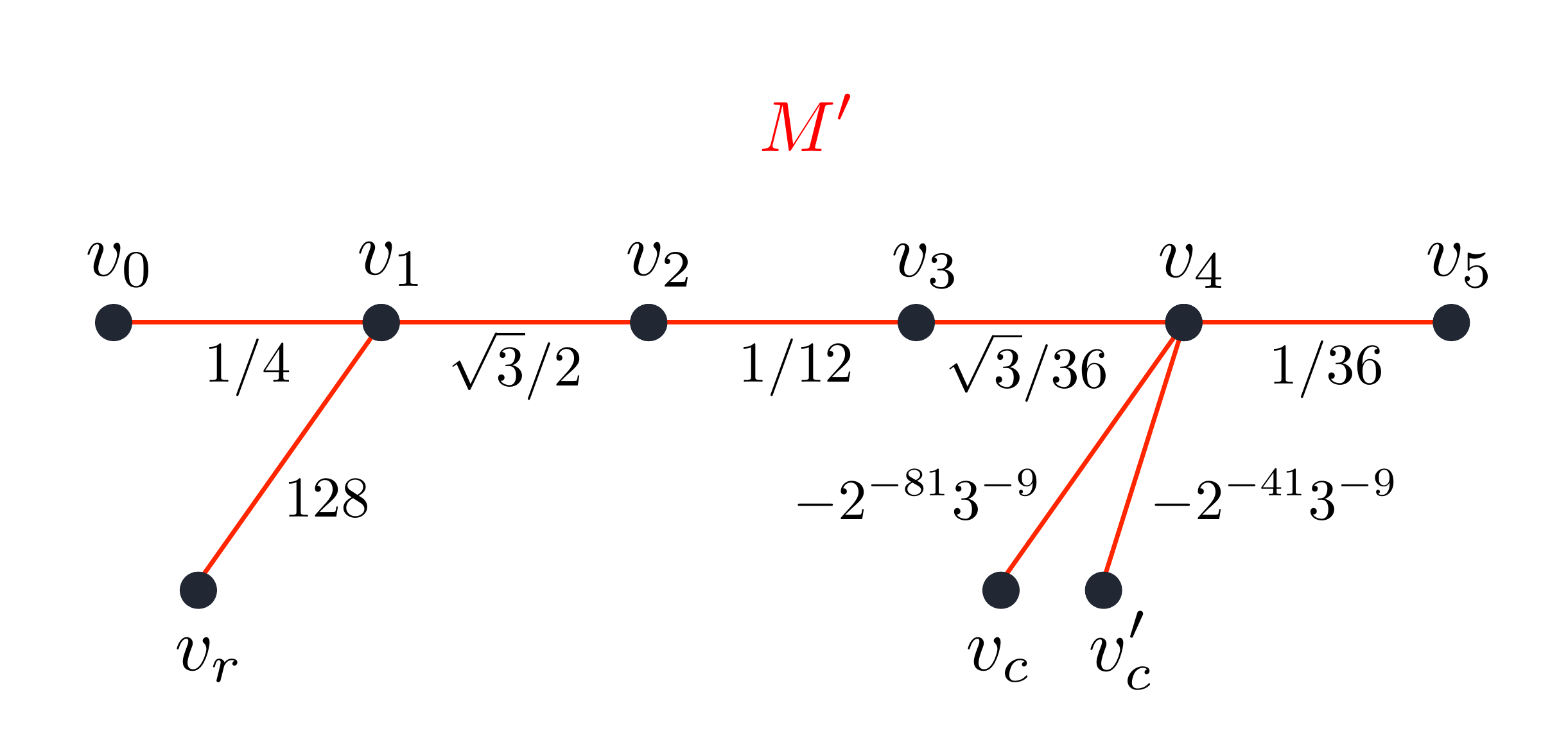}
        \caption{The noise-tolerant memory component $M'$.}
       \label{mprime}
    \end{figure}

See Figure \ref{mprime} for a representation of $M'$. 
The idea is that by controlling the values of $v_c$ and $v'_c$ we can either force $v_4$ to have an input of approximately $0$ in order to prevent any of the weights from changing or allow it to have a significant value in which case the weights will change. With the correct learning rate, if the correct output is $1$ then the weights of the edges on the path from $v_0$ to $v_5$ will double, while if the correct output is $-1$ then these weights will multiply by $-2$. That means that $v_2$ will have an output of approximately $2^{-24}3^{-3/2}$ if this has never been changed, and an output of approximately $2^{-12}3^{-3/2}$ if it has. Meanwhile, $v_r$ will have an output of $-2$ if it was changed when the correct output was $-1$ and a value of $2$ otherwise. More formally, we have the following.

\begin{lemma}[Editing memory using noisy SGD]\label{memLem2}
Let $\gamma=2^{716/3}\cdot 3^{24}$, and $L(x)=x^2$ for all $x$. Next, let $t_0, T\in\mathbb{Z}^+$ and $0<\epsilon,\epsilon'$ such that $\epsilon \le 2^{-134}3^{-11}$, $\epsilon'\le 2^{-123}3^{-11}$. Also, let $(f^\star,G)$ be a neural net such that $G$ contains $M'$ as a subgraph with $v_5$ as $G$'s output vertex, $v_0$ as the constant vertex, and no edges from vertices outside this subgraph to vertices in the subgraph other than $v_c$ and $v'_c$. Now, assume that this neural net is trained using noisy SGD with learning rate $\gamma$ and loss function $L$ for $T-1$ time steps, and then evaluated on an input, and the following hold:
\begin{enumerate}
    \item The sample label is always $\pm 1$.
    
    \item The net gives an output that is in $[1/2-\epsilon',1/2+\epsilon']$ on step $t$ for every $t<T$.
    
    \item For every $t<t_0$, $v_c$ gives an output of $2$ and $v'_c$ gives an output of $0$ on step $t$.
    
    \item If $t_0\le T$ then $v_c$ and $v'_c$ both give outputs of $0$ on step $t_0$.
    
    \item For every $t>t_0$, $v'_c$ gives an output of $2$ and $v_c$ gives an output of $0$ on step $t$.
    
    \item For each edge in the graph, the sum of the absolute values of the noise terms applied to that edge over the course of the training process is at most $\epsilon$.
    
    \item The derivatives of the loss function with respect to the weights of all edges leaving this subgraph are always $0$.
    
\end{enumerate}
Then during the training process, $v_2$ gives an output in $[2^{-25}3^{-3/2},2^{-23}3^{-3/2}]$ on step $t$ for all $t\le t_0$ and an output in $[2^{-13}3^{-3/2},2^{-11}3^{-3/2}]$ on step $t$ for all $t> t_0$. Also, on step $t$, $v_r$ gives an output of $-2$ if $t>t_0$ and the sample label was $-1$ on step $t_0$ and an output of $2$ otherwise. Thirdly, on step $t_0$ the edge from $v_4$ to $v_5$ provides an input to the output vertex in $[2^{-242}3^{-29}-2^{-201}3^{-27}\epsilon,2^{-242}3^{-29}+2^{-201}3^{-27}\epsilon]$, and for all $t\ne t_0$, the edge from $v_4$ to $v_5$ provides an input of $0$ to the output vertex on step $t$.
\end{lemma}

\begin{proof}
First of all, we define the target weight of an edge to be what we would like its weight to be. More precisely, the target weights of $(v_c,v_4)$, $(v'_c,v_4)$, and $(v_1,v_r)$ are defined to be equal to their initial weights at all time steps. The target weights of the edges on the path from $v_0$ to $v_5$ are defined to be equal to their initial weights until step $t_0$. After step $t_0$, these edges have target weights that are equal to double their initial weights if the sample label at step $t_0$ was $1$ and $-2$ times their initial weights if the sample label at step $t_0$ was $-1$.

Next, we define the primary distortion of a given edge at a given time to be the sum of all noise terms added to its weight by noisy SGD up to that point. Then, we define the secondary distortion of an edge to be the difference between its weight, and the sum of its target weight and its primary distortion. By our assumptions, the primary distortion of any edge always has an absolute value of at most $\epsilon$. We plan to prove that the secondary distortion stays reasonably small by inducting on the time step, at which point we will have established that the actual weights of the edges stay reasonably close to their target weights. 

Now, for all vertices $v$ and $v'$, and every time step $t$, let $w_{(v,v')[t]}$ be the weight of the edge from $v$ to $v'$ at the start of step $t$, $y_{v[t]}$ be the output of $v$ on step $t$, $d_{v[t]}$ be the derivative of the loss function with respect to the output of $v$ on step $t$, and $d'_{v[t]}$ be the derivative of the loss function with respect to the input of $v$ on step $t$. Next, consider some $t<t_0$ and assume that the secondary distortion of every edge in $M'$ is $0$ at the start of step $t$. In this case, $v_1$ has an activation in $[(1/4-\epsilon)^3,(1/4+\epsilon)^3]$, so $v_r$ has an activation of $2$ and the derivative of the loss function with respect to $w_{(v_1,v_r)}$ is $0$. Also, the activation of $v_2$ is between $2^{-25}3^{-3/2}$ and $2^{-23}3^{-3/2}$. On another note, the total input to $v_4$ on step $t$ is
\begin{align*}
&w_{(v_0,v_1)[t]}^{27}w_{(v_1,v_2)[t]}^{9}w_{(v_2,v_3)[t]}^{3}w_{(v_3,v_4)[t]}+w_{(v_c,v_4)[t]}y_{v_c[t]}+w_{(v'_c,v_4)[t]}y_{v'_c[t]}\\
&\le \left(\frac{1}{4}+\epsilon\right)^{27}\left(\frac{\sqrt{3}}{12}+\epsilon\right)^{9}\left(\frac{1}{12}+\epsilon\right)^{3}\left(\frac{\sqrt{3}}{36}+\epsilon\right)+\left(-2^{-81}3^{-9}+\epsilon\right)\cdot 2\\
&\le 2^{-80}3^{-9}e^{(144+48\sqrt{3})\epsilon}-2^{-80}3^{-9}+2\epsilon\\
&\le 3\epsilon
\end{align*}

On the flip side, the total input to $v_4$ on step $t$ is at least
\begin{align*}
&\left(\frac{1}{4}-\epsilon\right)^{27}\left(\frac{\sqrt{3}}{12}-\epsilon\right)^{9}\left(\frac{1}{12}-\epsilon\right)^{3}\left(\frac{\sqrt{3}}{36}-\epsilon\right)+(-2^{-81}3^{-9}-\epsilon)\cdot 2\\
&\ge 2^{-80}3^{-9}e^{-2(144+48\sqrt{3})\epsilon}-2^{-80}3^{-9}-2\epsilon\\
&\ge -3\epsilon
\end{align*}

So, $|y_{v_4[t]}|=0$, and the edge from $v_4$ to $v_5$ provides an input of $0$ to the output vertex on step $t$. The derivative of this contribution with respect to the weights of any of the edges in $M'$ is also $0$. So, if all of the secondary distortions are $0$ at the beginning of step $t$, then all of the secondary distortions will still be $0$ at the end of step $t$. The secondary distortions start at $0$, so by induction on $t$, the secondary distortions are all $0$ at the end of step $t$ for every $t<\min(t_0,T)$. This also implies that the edge from $v_4$ to $v_5$ provides an input of $0$ to the output, $y_{v_r[t]}=2$, and $v_{v_2[t]}\in [2^{-25}3^{-3/2},2^{-23}3^{-3/2}]$ for every $t<t_0$.

\vspace{8mm}

Now, consider the case where $t=t_0\le T$. In this case, $v_r$ has an activation of $2$ and the derivative of the loss function with respect to $w_{(v_1,v_r)}$ is $0$ for the same reasons as in the last case. Also, the activation of $v_2$ is still between $2^{-25}3^{-3/2}$ and $2^{-23}3^{-3/2}$.

On this step, the total input to $v_4$ is 
\begin{align*}
&w_{(v_0,v_1)[t]}^{27}w_{(v_1,v_2)[t]}^{9}w_{(v_2,v_3)[t]}^{3}w_{(v_3,v_4)[t]}+w_{(v_c,v_4)[t]}y_{v_c[t]}+w_{(v'_c,v_4)[t]}y_{v'_c[t]}\\
&\le \left(\frac{1}{4}+\epsilon\right)^{27}\left(\frac{\sqrt{3}}{12}+\epsilon\right)^{9}\left(\frac{1}{12}+\epsilon\right)^{3}\left(\frac{\sqrt{3}}{36}+\epsilon\right)+0\\
&\le 2^{-80}3^{-9}e^{(144+48\sqrt{3})\epsilon}\\
&\le 2^{-80}3^{-9}+2^{-74}3^{-7}\epsilon
\end{align*}
On the flip side, the total input to $v_4$ is at least
\begin{align*}
& \left(\frac{1}{4}-\epsilon\right)^{27}\left(\frac{\sqrt{3}}{12}-\epsilon\right)^{9}\left(\frac{1}{12}-\epsilon\right)^{3}\left(\frac{\sqrt{3}}{36}-\epsilon\right)+0\\
&\ge 2^{-80}3^{-9}e^{-2(144+48\sqrt{3})\epsilon}\\
&\ge 2^{-80}3^{-9}-2^{-74}3^{-7}\epsilon
\end{align*}

So, $y_{v_4[t]}\in [(2^{-80}3^{-9}-2^{-74}3^{-7}\epsilon)^3,(2^{-80}3^{-9}+2^{-74}3^{-7}\epsilon)^3]$, and the edge from $v_4$ to $v_5$ provides an input in $[2^{-242}3^{-29}-2^{-235}3^{-26}\epsilon,2^{-242}3^{-29}+2^{-235}3^{-26}\epsilon]$ to the output vertex on step $t_0$. If $t_0<T$ then the net gives an output in $[1/2-\epsilon',1/2+\epsilon']$, so $d_{v_5[t]}$ is in $[-1-2\epsilon',-1+2\epsilon']$ if the sample label is $1$ and in $[3-2\epsilon',3+2\epsilon']$ if the sample label is $-1$. That in turn means that $d'_{v_5[t]}$ is in $[\frac{3\sqrt[3]{2}}{2}(-1-4\epsilon'),\frac{3\sqrt[3]{2}}{2}(-1+4\epsilon')]$ if the sample label is $1$ and in $[\frac{3\sqrt[3]{2}}{2}(3-8\epsilon'),\frac{3\sqrt[3]{2}}{2}(3+8\epsilon')]$ if the sample label is $-1$. Either way, the derivatives of the loss function with respect to $w_{(v_c,v_4)}$ and $w_{(v'_c,v_4)}$ are both $0$. 

Also, for each $0\le i<5$, the derivative of the loss function with respect to $w_{(v_i,v_{i+1})}$ is 
\[w_{(v_0,v_1)[t]}^{81}w_{(v_1,v_2)[t]}^{27}w_{(v_2,v_3)[t]}^{9}w_{(v_3,v_4)[t]}^3w_{(v_4,v_5)}\cdot \frac{3^{4-i}}{w_{(v_i,v_{i+1})}}\cdot d'_{v[5]}\]
which is between $2^{-240}3^{-29}\cdot (1-7200\epsilon)\cdot 3^{4-i/2}\cdot d'_{v[5]}$ and $2^{-240}3^{-29}\cdot(1+7200\epsilon)\cdot 3^{4-i/2}\cdot d'_{v[5]}$

So, if the sample label is $1$, then on this step gradient descent increases the weight of each edge on the path from $v_0$ to $v_5$ by an amount that is within $3600\epsilon+2\epsilon'$ of its original value. If the sample label is $-1$, then on this step gradient descent decreases the weight of each edge on this path by an amount that is within $10800\epsilon+6\epsilon'$ of thrice its original value. Either way, it leaves the weight of the edge from $v_1$ to $v_r$ unchanged. So, all of the secondary distortions will be at most $10800\epsilon+6\epsilon'$ at the end of step $t_0$ if $t_0<T$.

\vspace{8mm}
Finally, consider the case where $t>t_0$ and assume that the secondary distortion of every edge in $M'$ is at most $10800\epsilon+6\epsilon'$ at the start of step $t$. Also, let $\epsilon''=10801\epsilon+6\epsilon'$, and $y_0$ be the sample label from step $t_0$. In this case, $v_1$ has an activation between $(1/2-\epsilon'')^3y_0$ and $(1/2+\epsilon'')^3y_0$, so $v_r$ has an activation of $2 y_0$ and the derivative of the loss function with respect to $w_{(v_1,v_r)}$ is $0$. Also, the activation of $v_2$ is between $2^{-13}3^{-3/2}$ and $2^{-11}3^{-3/2}$. On another note, the total input to $v_4$ on step $t$ is
\begin{align*}
&w_{(v_0,v_1)[t]}^{27}w_{(v_1,v_2)[t]}^{9}w_{(v_2,v_3)[t]}^{3}w_{(v_3,v_4)[t]}+w_{(v_c,v_4)[t]}y_{v_c[t]}+w_{(v'_c,v_4)[t]}y_{v'_c[t]}\\
&\le \left(\frac{y_0}{2}+\epsilon''y_0\right)^{27}\left(\frac{\sqrt{3}y_0}{6}+\epsilon''y_0\right)^{9}\left(\frac{y_0}{6}+\epsilon''y_0\right)^{3}\left(\frac{\sqrt{3}y_0}{18}+\epsilon''y_0\right)+(-2^{-41}3^{-9}+\epsilon'')\cdot 2\\
&\le 2^{-40}3^{-9}e^{(72+24\sqrt{3})\epsilon''}-2^{-40}3^{-9}+2\epsilon''\\
&\le 3\epsilon''
\end{align*}

On the flip side, the total input to $v_4$ on step $t$ is at least
\begin{align*}
&\left(\frac{y_0}{2}-\epsilon''y_0\right)^{27}\left(\frac{\sqrt{3} y_0}{6}-\epsilon''y_0\right)^{9}\left(\frac{y_0}{6}-\epsilon''y_0\right)^{3}\left(\frac{\sqrt{3} y_0}{18}-\epsilon''y_0\right)+(-2^{-41}3^{-9}-\epsilon'')\cdot 2\\
&\ge 2^{-40}3^{-9}e^{-2(72+24\sqrt{3})\epsilon''}-2^{-40}3^{-9}-2\epsilon''\\
&\ge -3\epsilon''
\end{align*}

So, $y_{v_4[t]}=0$, and the edge from $v_4$ to $v_5$ provides an input of $0$ to the output vertex on step $t$. The derivatives of this contribution with respect to the weights of any of the edges in $M'$ are also $0$. So, if all of the secondary distortions are at most $10800\epsilon+6\epsilon'$ at the beginning of step $t$, then all of the secondary distortions will still be at most $10800\epsilon+6\epsilon'$ at the end of step $t$. 
We have already established that the secondary distortions will be in that range at the end of step $t_0$, so by induction on $t$, the secondary distortions are all at most $10800\epsilon+6\epsilon'$ at the end of step $t$ for every $t_0<t<T'$. This also implies that the edge from $v_4$ to $v_5$ provides an input of $0$ to the output, $y_{v_r[t]}=2y_0$ and $v_{v_2[t]}\in [2^{-13}3^{-3/2},2^{-11}3^{-3/2}]$ for every $t>t_0$.

\end{proof}

Now that we have established that we can use $M'$ to store information in a noise tolerant manner, our next order of business is to show that we can make the computation component noise-tolerant. This is relatively simple because all of its vertices always have inputs of absolute value at least $2$, so changing these inputs by less than $1/2$ has no effect. We have the following.

\begin{lemma}[Backpropagation-proofed noise-tolerant circuit emulation]\label{noisy_em1}
Let $h:\{0,1\}^m\rightarrow\{0,1\}^{m'}$ be a function that can be computed by a circuit made of AND, OR, and NOT gates with a total of $b$ gates. Also, consider a neural net with $m$ input\footnote{Note that these will not be the $n$ data input of the general neural net that is being built; these input vertices take both the data inputs and some inputs from the memory component.} vertices $v'_1,...,v'_m$, and choose real numbers $y^{(0)}<y^{(1)}$. It is possible to add a set of at most $b$ new vertices to the net, including output vertices $v''_1,...,v''_{m'}$, along with edges leading to them such that for any possible addition of edges leading from the new vertices to old vertices, if the net is trained by noisy SGD, the output of $v'_i$ is either less than $y^{(0)}$ or more than $y^{(1)}$ for every $i$ in every timestep, and for every edge leading to one of the new vertices, the sum of the absolute values of the noise terms applied to that edge over the course of the training process is less than $1/12$, then the following hold:
\begin{enumerate}
    \item The derivative of the loss function with respect to the weight of each edge leading to a new vertex is $0$ in every timestep, and no paths through the new vertices contribute to the derivative of the loss function with respect to edges leading to the $v'_i$.
    
    \item In any given time step, if the output of $v'_i$ encodes $x_i$ with values less than $y^{(0)}$ and values greater than $y^{(1)}$ representing $0$ and $1$ respectively for each $i$\footnote{This time we can use the same values of $y^{(0)}$ and $y^{(1)}$ for all $v'_i$ because we just need them to be between whatever the vertex encodes $0$ as and whatever it encodes $1$ as for all vertices.}, then the output of $v''_j$ encodes $h_j(x_1,...,x_m)$ for each $j$ with $-2$ and $2$ encoding $0$ and $1$ respectively.
\end{enumerate}
\end{lemma}

\begin{proof}
In order to do this, we will add one new vertex for each gate and each input in a circuit that computes $h$. When the new vertices are used to compute $h$, we want each vertex to output $2$ if the corresponding gate or input outputs a $1$ and $-2$ if the corresponding gate or input outputs a $0$. In order to do that, we need the vertex to receive an input of at least $3/2$ if the corresponding gate outputs a $1$ and an input of at most $-3/2$ if the corresponding gate outputs a $0$. No vertex can ever give an output with an absolute value greater than $2$, and by assumption none of the edges leading to the new vertices will have their weights changed by $1/12$ or more by the noise. As such, any noise terms added to the weights of edges leading to a new vertex will alter its input by at most $1/6$ of its in-degree. So, as long as its input without these noise terms has the desired sign and an absolute value of at least $3/2$ plus $1/6$ of its in-degree, it will give the desired output.

In order to make one new vertex compute the NOT of another new vertex, it suffices to have an edge of weight $-1$ to the vertex computing the NOT and no other edges to that vertex. We can compute an AND of two new vertices by having a vertex with two edges of weight $1$ from these vertices and an edge of weight $-2$ from the constant vertex. Similarly, we can compute an OR of two new vertices by having a vertex with two edges of weight $1$ from these vertices and an edge of weight $2$ from the constant vertex. For each $i$, in order to make a new vertex corresponding to the $i$th input, we add a vertex and give it an edge of weight $4/(y^{(1)}-y^{(0)})$ from the associated $v'_i$ and an edge of weight $-(2y^{(1)}+2y^{(0)})/(y^{(1)}-y^{(0)})$ from the constant vertex. These provide an overall input of at least $2$ to the new vertex if $v'_i$ has an output greater than $y^{(1)}$ and an input of at most $-2$ if $v'_i$ has an output less than $y^{(0)}$.

This ensures that if the outputs of the $v'_i$ encode binary values $x_1,...,x_m$ appropriately, then each of the new vertices will output the value corresponding to the output of the appropriate gate or input. So, these vertices compute $h(x_1,...,x_m)$ correctly. Furthermore, since the input to each of these vertices is outside of $(-3/2,3/2)$, the derivatives of their activation functions with respect to their inputs are all $0$. As such, the derivative of the loss function with respect to any of the edges leading to them is always $0$, and paths through them do not contribute to changes in the weights of edges leading to the $v'_i$.
\end{proof}

Now that we know that we can make the memory component and computation component work, it is time to put the pieces together. We plan to have the net simply memorize each sample it receives until it has enough information to compute the function. More precisely, if there is an algorithm that needs $T$ samples to learn functions from a given distribution, our net will have $2nT$ copies of $M'$ corresponding to every combination of a timestep $1\le t\le T$, an input bit, and a value for said bit. Then, in step $t$ it will set the copies of $M'$ corresponding to the inputs it received in that time step. That will allow the computation component to determine what the current time step is, and what the inputs and labels were in all previous times steps by checking the values of the copies of $v_2$ and $v_r$. That will allow it to either determine which copies of $M'$ to set next, or attempt to compute the function on the current input and return it. This design works in the following sense.

\begin{lemma}
For each $n>0$, let $t_n$ be a positive integer such that $t_n=\omega(1)$ and $t_n=O(n^c)$ for some constant $c$. Also, let $h_n:\{0,1\}^{(n+1)t_n+n}\rightarrow\{0,1\}$ be a function that can be computed in time polynomial in $n$. Then there exists a polynomial sized neural net $(G_n, f)$ such that the following holds. Let $\gamma=2^{716/3}\cdot 3^{24}$, $\delta\in [-1/n^2t_n,1/n^2t_n]^{t_n\times |E(G_n)|}$, $x^{(i)}\in\{0,1\}^{n}$ for all $0\le i\le t_n$, and $y^{(i)}\in\{0,1\}$ for all $0\le i<t_n$. Then if we use perturbed stochastic gradient descent with noise $\delta$, loss function $L(x)=x^2$, and learning rate $\gamma$ to train $(G_n,f)$ on $(2x^{(i)}-1,2y^{(i)}-1)$ for $0\le i<t_n$ and then run the resulting net on $2x_{t,n}-1$, we will get an output within $1/2$ of $2h\left(x^{(0)},y^{(0)},x^{(1)},y^{(1)},...,x^{(t_n)}\right)-1$ with probability $1-o(1)$. 
\end{lemma}

\begin{proof}
We construct $G_n$ as follows. We start with a graph consisting of $n$ input vertices. Then, we take $2n t_n$ copies of $M'$, merge all of the copies of $v_0$ to make a constant vertex, and merge all of the copies of $v_5$ to make an output vertex. We assign each of these copies a distinct label of the form $M'_{t', i, z}$, where $0\le t'<t_n$, $0< i\le n$, and $z\in\{0,1\}$. We also add edges of weight $1$ from the constant vertex to all of the control vertices. Next, for each $0\le t'< t_n$, we add an output control vertex $v_{oc[t']}$. For each such $t'$, we add an edge of weight $1$ from the constant vertex to $v_{oc[t']}$ and an edge of weight $\sqrt[3]{4}/4-2^{-243}3^{-29}n$ from $v_{oc[t']}$ to the output vertex. Then, we add a final output control vertex $v_{oc[t_n]}$. We do not add an edge from the constant vertex to $v_{oc[t_n]}$, and the edge from $v_{oc[t_n]}$ to the output vertex has weight $49/100$. 

Finally, we use the construction from the previous lemma to build a computation component. This component will get input from all of the input vertices and every copy of $v_r$ and $v_2$ in any of the copies of $T'$, interpreting anything less than $2^{-21}3^{-3/2}$ as a $0$ and anything more than $2^{-15}3^{-3/2}$ as a $1$. This should allow it to read the input bits, and determine which of the copies of $M'$ have been set and what the sample outputs were when they were set. For each control vertex from a copy of $T'$ and each of the first $n$ output control vertices, the computation component will contain a vertex with an edge of weight $1/2$ leading to that vertex. It will contain two vertices with edges of weight $1/2$ leading to $v_{oc[t_n]}$. This should allow it to set each control vertex or output control vertex to $0$ or $2$, and to set $v_{oc[t_n]}$ to $-2$, $0$, or $2$. 

The computation component will be designed so that in each time step it will do the following, assuming that its edge weights have not changed too much and the outputs of the copies of $v_r$ and $v_2$ are in the ranges given by lemma \ref{memLem2}. First, it will determine the smallest $0\le t\le t_n$ such that $M'_{(t',i,z)}$ has not been set for any $t'\ge t$, $0<i\le n$, and $z\in\{0,1\}$. That should equal the current timestep. If $t< t_n$, then it will do the following. For each $0<i\le n$, it will use the control vertices to set $M'_{(t,i,[x'_i+1]/2)}$, where $x'_i$ is the value it read from the $i$th input vertex. It will keep the rest of the copies of $M'$ the same. It will also attempt to make $v_{oc[t]}$ output $2$ and the other output control vertices output $0$. If $t=t_n$, then for each $0\le t'<t$ and $1\le i\le n$, the computation component will set $x^{\star(t')}_i$ to $1$ if $M'_{(t',i,1)}$ has been set, and $0$ otherwise. It will set $y^{\star(t')}$ to $1$ if either $M'_{(t',1,0)}$ or $M'_{(t',1,1)}$ has been set in a timestep when the sample label was $1$ and $0$ otherwise. It will also let $x^{\star(t_n)}$ be the values of $x^{(t_n)}$ inferred from the input. Then it will attempt to make $v_{oc[t_n]}$ output $4h(x^{\star(0)},y^{\star(0)},...,x^{\star(t_n)})-2$ and the other output control vertices output $0$. It will not set any of the copies of $M'$ in this case.

In order to prove that this works, we start by setting
$\epsilon=\min(2^{-134}3^{-11},2^{77}3^{15}/n)$ and $\epsilon'=2^{-123}3^{-11}$. The absolute value of the noise term applied to every edge in every time step is at most $1/n^2 t_n$, so the sums of the absolute values of the noise terms applied to every edge over the course of the algorithm are at most $\epsilon$ if $n>2^{67}3^{6}$. For the rest of the proof, assume that this holds.

Now, we claim that for every $0\le t'<t_n$, all of the following hold: \begin{enumerate}

\item Every copy of $v_r$ or $v_2$ in the memory component outputs a value that is not in $[2^{-21}3^{-3/2},2^{-15}3^{-3/2}]$ on timestep $t'$.

\item For every copy of $M'$, there exists $t_0$ such that its copies of $v_c$ and $v'_c$ take on values satisfying lemma \ref{memLem2} for timesteps $0$ through $t'$.

\item The net gives an output in $[1/2-\epsilon',1/2+\epsilon']$ on timestep $t'$.

\item The weight of every edge leading to an output control vertex ends step $t'$ with a weight that is within $\epsilon$ of its original weight.

\item For every $t''>t'$, the weight of the edge from $v_{oc[t'']}$ to the output vertex has a weight within $\epsilon$ of its original weight at the end of step $t'$.
\end{enumerate}

In order to prove this, we use strong induction on $t'$. So, let $0\le t'<t_n$, and assume that this holds for all $t''<t'$. By assumption, the conditions of lemma \ref{memLem2} were satisfied for every copy of $M'$ in the first $t'$ timesteps. So, the outputs of the copies of $v_r$ and $v_2$ encode information about their copies of $M'$ in the manner given by this lemma. In particular, that means that their outputs are not in $[2^{-21}3^{-3/2},2^{-15}3^{-3/2}]$ on timestep $t'$. By the previous lemma, the fact that this holds for timesteps $0$ through $t'$ means that the computation component will still be working properly on step $t'$, it will be able to interpret the inputs it receives correctly, and its output vertices will take on the desired values. The assumptions also imply that every copy of $v_c$ or $v'_c$ took on values of $0$ or $2$ in step $t''$ for every $t''<t'$. That means that the derivatives of the loss function with respect to the weights of the edges leading to these vertices was always $0$, so their weights at the start of step $t'$ were within $\epsilon$ of their initial weights. That means that the inputs to these copies will be in $[-4\epsilon,4\epsilon]$ for ones that are supposed to output $1$ and in $[2-4\epsilon,2+4\epsilon]$ for ones that are supposed to output $2$. Between this and the fact that the computation component is working correctly, we have that for each $(t'',i,z)$, the copies of $v_c$ and $v'_c$ in $M'_{(t'',i,z)}$ will have taken on values satisfying the conditions of lemma \ref{memLem2} in timesteps $0$ through $t'$ with $t_0$ set to $t''$ if $x_i^{(t'')}=z$ and $t_n+1$ otherwise.

Similarly, the fact that the weights of the edges leading to the output control vertices stay within $\epsilon$ of their original values for the first $t'-1$ steps implies that $v_{oc[t'']}$ outputs $2$ and all other output control vertices output $0$ on step $t''$ for all $t''\le t'$. That in turn implies that the derivatives of the loss function with respect to these weights were $0$ for the first $t'+1$ steps, and thus that their weights are still within $\epsilon$ of their original values at the end of step $t'$. Now, observe that there are exactly $n$ copies of $M'$ that get set in step $t'$, and each of them provide an input to the output vertex in $[2^{-242}3^{-29}-2^{-201}3^{-27}\epsilon,2^{-242}3^{-29}+2^{-201}3^{-27}\epsilon]$. Also, $v_{oc[t']}$ provides an input to the output in $[\sqrt[3]{4}/2-2^{-242}3^{-29}n-2\epsilon,\sqrt[3]{4}/2-2^{-242}3^{-29}n+2\epsilon]$ on step $t'$, and all other vertices with edges to the output vertex output $0$ in this time step. So, the total input to the output vertex is within $2^{-201}3^{-27}\epsilon n+2\epsilon\le \epsilon'/3$ of $\sqrt[3]{4}/2$. So, the net gives an output in $[1/2-\epsilon',1/2+\epsilon']$ on step $t'$, as desired. This also implies that the derivative of the loss function with respect to the weights of the edges from all output vertices except $v_{oc[t']}$ to the output vertex are $0$ on step $t'$. So, for every $t''>t'$, the weight of the edge from $v_{oc[t'']}$ to the output vertex is still within $\epsilon$ of its original value at the end of step $t'$. This completes the induction argument.

This means that on step $t_n$, all of the copies of $v_r$ and $v_c$ will still have outputs that encode whether or not they have been set and what the sample output was on the steps when they were set in the manner specified in lemma \ref{memLem2}, and that the computation component will still be working. So, the computation component will set $x^{\star (t')}=x^{(t')}$ and $y^{\star(t')}=y^{(t')}$ for each $t'<t_n$. It will also set $x^{\star(t_n)}=x^{(t_n)}$, and then it will compute $h\left(x^{(0)},y^{(0)},x^{(1)},y^{(1)},...,x^{(t_n)}\right)$ correctly. Call this expression $y'$. All edges leading to the output control and control vertices will still have weights within $\epsilon$ of their original values, so it will be able to make $v_{oc[t_n]}$ output $4y'-2$, all other output control vertices output $0$, and none of the copies of $M'$ provide a nonzero input to the output vertex. The output of $v_{oc[t_n]}$ is $0$ in all timesteps prior to $t_n$, so the weight of the edge leading from it to the output vertex at the start of step $t_n$ is within $\epsilon$ of its original value. So, the output vertex will receive a total input that is within $2\epsilon$ of $\frac{49}{50}(2y'-1)$, and give an output that is within $6\epsilon$ of $\frac{49^3}{50^3}(2y'-1)$. That is within $1/2$ of $2y'-1$, as desired.
\end{proof}

This allows us to prove that we can emulate an arbitrary algorithm by using the fact that the output of any efficient algorithm can be expressed as an efficiently computable function of its inputs and some random bits. More formally, we have the following (re-statement of Theorem \ref{thm_univ2}).

\begin{theorem}
For each $n>0$, let $P_\X$ be a probability measure on $\{0,1\}^n$, and $P_{\F}$ be a probability measure on the set of functions from $\{0,1\}^n$ to $\{0,1\}$. Also, let $B_{1/2}$ be the uniform distribution on $\{0,1\}$, $t_n$ be polynomial in $n$, and $\delta\in [-1/n^2t_n,1/n^2t_n]^{t_n\times |E(G_n)|}$, $x^{(i)}\in\{0,1\}^{n}$. Next, define $\alpha_n$ such that there is some algorithm that takes $t_n$ samples $(x_i,F(x_i))$ where the $x_i$ are independently drawn from $P_{\X}$ and $F\sim P_{\F}$, runs in polynomial time, and learns $(P_\F,P_\X)$ with accuracy $\alpha$. Then there exists $\gamma>0$, and a polynomial-sized neural net $(G_n,f)$ such that using perturbed stochastic gradient descent with noise $\delta$, learning rate $\gamma$, and loss function $L(x)=x^2$ to train $(G_n,f)$ on $t_n$ samples $((2x_i-1,2r_i-1), 2F(x_i)-1)$ where $(x_i,r_i)\sim P_\X\times B_{1/2}$ learns $(P_\F,P_\X)$ with accuracy $\alpha-o(1)$.
\end{theorem}

\begin{proof}
Let $A$ be an efficient algorithm that learns $(P_\F,P_\X)$ with accuracy $\alpha$, and $t_n$ be a polynomial in $n$ such that $A$ uses fewer than $t_n$ samples and random bits with probability $1-o(1)$. Next, define $h_n\{0,1\}^{(n+1)t_n+t_n+n\rightarrow \{0,1\}}$ such that the algorithm outputs $h_n(z_1,...,z_{t_n},b_1,...,b_{t_n},x')$ if it receives samples $z_1,...,z_{t_n}$, random bits $b_1,...,b_{t_n}$ and final input $x'$. There exists a polynomial $t^\star_n$ such that $A$ computes $h_n(z_1,...,z_{t_n},b_1,...,b_{t_n},x')$ in $t^\star_n$ or fewer steps with probability $1-o(1)$ given samples $z_1,...,z_{t_n}$ generated by a function drawn from $(P_\F,P_\X)$, random bits $b_1,...,b_{t_n}$, and $x'\sim P_\X$. So, let $h'_n(z_1,...,z_{t_n},b_1,...,b_{t_n},x')$ be $h_n(z_1,...,z_{t_n},b_1,...,b_{t_n},x')$ if $A$ computes it in $t^\star_n$ or fewer steps and $0$ otherwise. $h'_n$ can always be computed in polynomial time, so by the previous lemma there exists a polynomial sized neural net $(G_n,f)$ that gives an output within $1/2$ of $2h'_n((x_1,y_1),...,(x_{t_n},y_{t_n}),b_1,...,b_{t_n},x')-1$ with probability $1-o(1)$ when it is trained using noisy SGD with noise $\Delta$, learning rate $2^{716/3}3^{24}$, and loss function $L$ on $((2x_i-1,2b_i-1),2F(x_i)-1)$ and then run on $2x'-1$. When the $(x_i,y_i)$ are generated by a function drawn from $(P_\F,P_\X)$, and $x'\sim P_\X$, using $A$ to learn the function and then compute it on $x'$ yields $h'_n(z_1,...,z_{t_n},b_1,...,b_{t_n},x')$ with probability $1-o(1)$. Therefore, training this net with noisy SGD in the manner described learns $(P_\F,P_\X)$ with accuracy $\alpha-o(1)$.
\end{proof}

\begin{remark}Like in the noise free case it would be possible to emulate a metaalgorithm that learns any function that can be learned from $n^c$ samples in $n^c$ time instead of an algorithm for a specific distribution. However, unlike in the noise free case there is no easy way to adapt the metaalgorithm to cases where we do not have an upper bound on the number of samples needed.
\end{remark}

\begin{remark}
Throughout the learning process used by the last theorem and lemma, every control vertex, output control vertex, and vertex in the computation component always takes on a value where the activation function has derivative $0$. As such, the weights of any edges leading to these vertices stay within $\epsilon$ of their original values. Also, the conditions of lemma \ref{memLem2} are satisfied, so none of the edge weights in the memory component go above $\epsilon'$ more than double their original values. That leaves the edges from the output control vertices to the output vertex. Each output vertex only takes on a nonzero value once, and on that step it has a value of $2$. The derivative of the loss function with respect to the input to the output vertex is at most $12$, so each such edge weight changes by at most $24\gamma+\epsilon$ over the course of the algorithm. So, none of the  edge weights go above a constant (i.e., $2^{242}3^{25}$) during the training process.
\end{remark}

\subsection{Additional comments on the emulation}
The previous result uses choices of a neural net and SGD parameters that are in many ways unreasonable. This choice of activation function is not used in practice, many of the vertices do not have edges from the constant vertex, and the learning rate is deliberately chosen to be so high that it keeps overshooting the minima. If one wanted to do something normal with a neural net trained by SGD one is unlikely to do it that way, and using it to emulate an algorithm is much less efficient than just running the algorithm directly, so this is unlikely to come up.

In order to emulate a learning algorithm with a more reasonable neural net and choice of parameters, we will need to use the following ideas in addition to the ideas from the previous result. First of all, we can control which edges tend to have their weights change significantly by giving edges that we want to change a very low starting weight and then putting high weight edges after them to increase the derivative of the output with respect to them. Secondly, rather than viewing the algorithm we are trying to emulate as a fixed circuit, we will view it as a series of circuits that each compute a new output and new memory values from the previous memory values and the current inputs. Thirdly, a lower learning rate and tighter restrictions on how quickly the network can change prevent us from setting memory values in one step. Instead, we initialize the memory values to a local maximum so that once we perturb them, even slightly, they will continue to move in that direction until they take on the final value. Fourth, in most steps the network will not try to learn anything, so that with high probability all memory values that were set in one step will have enough time to stabilize before the algorithm tries to adjust anything else. Finally, once we have gotten to the point that the algorithm is ready to approximate the function, its estimates will be connected to the output vertex, and the output will gradually become more influenced by it over time as a basic consequence of SGD.

\bibliographystyle{amsalpha}
\bibliography{deep}

\end{document}